\documentclass{article}

\PassOptionsToPackage{dvipsnames,table}{xcolor}

\usepackage[preprint]{neurips_2026}

\usepackage[utf8]{inputenc}
\usepackage[T1]{fontenc}
\usepackage{hyperref}
\hypersetup{pdftitle={Calibration without Ground Truth}, pdfauthor={Yuqing Kong, Mingyu Song, Yizhou Wang, Yifan Wu}}
\usepackage{url}
\usepackage{booktabs}
\usepackage{amsfonts}
\usepackage{nicefrac}
\usepackage{microtype}
\usepackage{xcolor}

\usepackage{amsmath}
\usepackage{mathtools}
\usepackage{amssymb}
\usepackage{amsthm}
\usepackage{cleveref}
\usepackage{thmtools}
\usepackage{thm-restate}
\usepackage{bbm}
\usepackage{bm}
\usepackage{verbatim}
\usepackage{array}
\usepackage{caption}
\usepackage{subcaption}
\usepackage{float}
\usepackage{changepage}
\usepackage{enumitem}
\usepackage{cases}
\usepackage{xfrac}
\usepackage{tikz}
\usepackage{multirow}
\usepackage{tcolorbox}
\usepackage{fancyvrb}
\usepackage{listings}
\usetikzlibrary{decorations.pathreplacing}

\usepackage{pifont}
\usepackage{makecell}

\newcommand{\pms}[1]{{\scriptsize$\pm$#1}}

\usepackage{graphicx}
\usetikzlibrary{calc}
\usetikzlibrary{arrows.meta, positioning}

\definecolor{primaryblue}{RGB}{0, 102, 204}
\definecolor{highlightred}{RGB}{204, 0, 0}
\definecolor{projectiongreen}{RGB}{0, 153, 76}
\definecolor{backgroundgray}{RGB}{180, 180, 180}
\definecolor{setfillcolor}{RGB}{204, 229, 255}

\definecolor{gridgray}{RGB}{220, 220, 220}
\definecolor{axislabelblack}{RGB}{50, 50, 50}
\definecolor{curvegreen}{RGB}{46, 139, 87}
\definecolor{blueline}{RGB}{50, 100, 200}
\definecolor{redline}{RGB}{200, 50, 50}
\definecolor{graypoint}{RGB}{100, 100, 100}

\theoremstyle{plain}
\newtheorem{theorem}{Theorem}[section]
\newtheorem{proposition}[theorem]{Proposition}
\newtheorem{lemma}[theorem]{Lemma}
\newtheorem{corollary}[theorem]{Corollary}
\theoremstyle{definition}
\newtheorem{definition}[theorem]{Definition}
\newtheorem{assumption}[theorem]{Assumption}
\theoremstyle{remark}
\newtheorem{remark}[theorem]{Remark}
\newtheorem{example}[theorem]{Example}
\newtheorem{observation}[theorem]{Observation}

\title{Calibration without Ground Truth}

\author{%
  Yuqing Kong\thanks{Authors are listed in alphabetical order.}\hspace{0.5em}\thanks{Corresponding author.} \\
  Center on Frontiers of Computing Studies \\
  School of Computer Science \\
  Peking University \\
  Beijing, China \\
  \texttt{yuqing.kong@pku.edu.cn} \\
  \And
  Mingyu Song \\
  Center on Frontiers of Computing Studies \\
  School of Computer Science \\
  Peking University \\
  Beijing, China \\
  \texttt{mingyu.song@pku.edu.cn} \\
  \AND
  Yizhou Wang \\
  Center on Frontiers of Computing Studies \\
  School of Computer Science \\
  State Key Laboratory of \\
  General Artificial Intelligence \\
  Peking University \\
  Beijing, China \\
  \texttt{yizhou.wang@pku.edu.cn} \\
  \And
  Yifan Wu \\
  Economics and Computation Group\\
  Microsoft Research, New England\\
  Cambridge, Massachusetts, United States\\
  \texttt{yifan.wu2357@gmail.com}
}

\newcommand{\R}{\mathbb{R}}

\renewcommand{\hat}{\widehat}
\renewcommand{\tilde}{\widetilde}

\newcommand{\prob}{\mathcal{D}}

\DeclareMathOperator*{\argmax}{arg\,max}

\DeclareMathOperator{\E}{\mathbb{E}}

\newcommand{\cA}{{\mathcal A}}

\newcommand{\state}{y}
\newcommand{\staterv}{Y}
\newcommand{\statesp}{\mathcal{Y}}

\newcommand{\score}{\ell}

\newcommand{\predRV}{\bm{Q}}
\newcommand{\pred}{\bm{q}}
\newcommand{\truepred}{\bm{p}}

\newcommand{\repRV}{R}
\newcommand{\rep}{r}

\newcommand{\predsp}{\mathcal{Q}}
\newcommand{\repsp}{\mathcal{R}}

\newcommand{\calmap}{h}
\newcommand{\calmapsp}{\mathcal{H}}

\newcommand{\prop}{\Gamma}

\newcommand{\simplex}{\Delta(\statesp)}

\newcommand{\genFunc}{G}

\newcommand{\kernel}{K}

\newcommand{\statervvec}{\mathbf{Y}}

\begin{document}

\maketitle

\begin{abstract}
  We propose a label-free post-processing framework that improves a strong but miscalibrated primary model using a weaker yet better-calibrated reference.
Our key insight is that strict worst-case improvement is possible if and only if the two models are not \emph{mutually calibrated}, meaning there does not exist a joint distribution over their reports and outcomes such that both predictions are simultaneously calibrated. We formalize this condition and connect it to no-arbitrage results from economics.
Under this condition, we develop an efficient post-processing algorithm for the strong model's outputs based on Bregman projection, with a strict worst-case performance improvement guarantee.
Experiments on three 8B LLM families across three multiple-choice benchmarks show that our method reduces the primary model's ECE by over $40\%$ on common benchmarks.

\end{abstract}

\section{Introduction}

Modern neural networks and LLMs achieve high accuracy but are often systematically overconfident \citep{guo2017calibration, ovadia2019can}. In principle, models should be calibrated, which requires reported confidence to align with the conditional likelihood of correctness \citep{dawid}.

Several methods have been developed to mitigate LLM overconfidence, including supervised fine-tuning \citep{wang-etal-2025-towards-objective, xiao2025restoring} and temperature scaling with human annotations \citep{guo2017calibration}. However, these approaches rely heavily on human-labeled data, which is challenging in highly specialized domains in practice. Things are even worse given the fact that temperature scaling performs poorly with limited labels \citep{mozafari2018attended}. \citet{villalobos2024position} predict that LLMs will exhaust the stock of public human text between 2026 and 2032, making calibrating LLMs \emph{without ground truth} more practical and urgent.

In this regime, one promising direction is weak-to-strong generalization, where an existing weak model is used to improve a strong one \citep{burns2024weak}. This weak-to-strong generalization suggests that the weak reference can be valuable even when the model is imperfect.
A natural instance of this pattern arises in calibration: \emph{fine-tuned} LLMs after post-training or RLHF are highly capable but overconfident, while their \emph{base} pre-trained counterparts are less accurate yet better calibrated \citep{achiam2023gpt,luo2026your,tan2026basecal}.
This pattern is common in many scenarios, and a well-calibrated weak model is relatively easy to obtain: a general-purpose text model may achieve higher accuracy than a structured baseline while producing less calibrated probabilities \citep{cruz2024evaluating}, and in real-time social-media annotation where human labels are hard to obtain, a calibrated rule-based system such as VADER \citep{hutto2014vader} can serve as a reference.

This raises a natural question: can a calibrated weaker model be used to improve a stronger one without accessing ground-truth labels, while preserving the high capability of the stronger model? Recent work \citep{luo2026your} pursues similar goals but does not establish a worst-case proper-loss improvement guarantee, while \citet{tan2026basecal} requires additional queries to weaker models or extra access to the hidden states of LLMs.

We name the stronger, miscalibrated model the \textbf{primary} predictor, and the weaker, better-calibrated model the \textbf{reference}. Our goal is to design a \textbf{post-processing algorithm} that maps the primary model's outputs to a strictly improved predictor with a provable worst-case improvement guarantee, while requiring no ground-truth labels. At deployment, the learned transformation uses only the primary model's output and does not require querying the reference predictor.

We characterize the key condition when strict improvement is possible as \textbf{mutual calibration}: two predictors are mutually calibratable if there exists a joint distribution over both predictions and outcomes under which both predictors are simultaneously calibrated. For full-distribution reports, our central result (\Cref{thm:main_result}) shows that a strictly improving transformation exists \emph{if and only if} the two predictors are \emph{not} mutually calibrated. This condition also connects to economics: mutual calibration parallels the common-prior assumption, and its failure implies guaranteed arbitrage opportunities \citep{milgrom1982information, arieli2021feasible, morris2020no}. For a chosen differentiable strictly proper loss, our algorithm exploits exactly this contradiction to guarantee a worst-case loss reduction. We provide the formal connection in Appendix~\ref{sec:arbitrage}.

The following example illustrates how contradictory predictions create an arbitrage opportunity:

\begin{example}[Arbitrage in Contradictory Predictions]
\label{example: intro contradictory}
Consider two predictors forecasting the probability that an investment will double in value.
\begin{itemize}
    \item The primary predictor outputs $10\%$ for half of the samples (Group~A of timestamps) and $70\%$ for the other half (Group~B of timestamps).
    \item The reference predictor always outputs a base rate of $50\%$ and is known to be calibrated.
\end{itemize}
The primary predictor implies an average base rate of $0.5 \times 70\% + 0.5 \times 10\% = 40\%$, contradicting the reference's $50\%$, so the two predictors are not mutually calibratable.

Consider an arbitrageur offering the following contracts ($0<\epsilon<0.1$).
\begin{itemize}
    \item \textbf{To the reference predictor:} the arbitrageur pays $(0.5+\epsilon)$ if the investment doubles, and receives $0.5$ if not.
    \item \textbf{To the primary predictor:} When predicting $0.7$, the arbitrageur receives $0.3$ if the investment doubles, and pays $(0.7+\epsilon)$ otherwise; when predicting $0.1$, receives $0.9$ if the investment doubles, and pays $(0.1+\epsilon)$ otherwise.
\end{itemize}
Both predictors accept these contracts because their subjective expected profits are strictly positive. The arbitrageur's payoff is $0.4-\epsilon$ in Group~A and $-0.2-\epsilon$ in Group~B, independently of the realized outcome. Since the groups have equal probability, the expected payoff is $0.1-\epsilon>0$ under every possible outcome, indicating an arbitrage opportunity through the contradictory predictions.

\end{example}

Our algorithm has two steps: (1) define the \emph{reference-compatible set}, which is a compact convex set including all transformations of the primary predictor that can be induced by ground truth extensions under which the reference is calibrated; and (2) project the primary predictor onto this set under the Bregman divergence induced by the chosen proper loss, choosing the transformation that minimally adjusts the primary model while restoring mutual calibration (\Cref{fig:projection}). When two predictors are not mutually calibrated, the projection guarantees a proper-loss reduction by reducing inconsistency.

\begin{figure}[ht]
\centering
\begin{tikzpicture}[scale=3,
    point/.style={circle, fill, minimum size=5pt, inner sep=0pt},
    declare function={
        xa=0.1;
        ya=0.7;
        xb=0.2;
        yb=0.8;
    }
]
    \draw[lightgray, line width=0.5pt] (0, 0) grid (1, 1);
    \draw[lightgray!50, thin, step=0.1] (0, 0) grid (1, 1);

    \foreach \x in {0.2, 0.4, 0.6, 0.8, 1.0}
        \draw (\x, 0.01) -- (\x, -0.01) node[below, font=\small, yshift=-1pt] {$\x$};

    \foreach \y in {0.2, 0.4, 0.8, 1.0}
        \draw (0.01, \y) -- (-0.01, \y) node[left, font=\small, xshift=-1pt] {$\y$};

    \draw[->, >=stealth, black, line width=0.8pt] (-0.1, 0) -- (1.1, 0) node[below right] {Group A ($\pred_1$)};
    \draw[->, >=stealth, black, line width=0.8pt] (0, -0.1) -- (0, 1.1) node[above] {Group B ($\pred_2$)};

    \coordinate (Origin) at (0, 0);
    \coordinate (A) at (1, 0);
    \coordinate (B) at (0, 1);
    \coordinate (P) at (xa, ya);
    \coordinate (Q) at (xb, yb);

    \draw[line width=1.5pt, blue] (A) -- (B)
    node[pos=0.6, above right, sloped, font=\scriptsize] {\textbf{Compatible}: $\pred_1+\pred_2=1$};

    \draw[darkgray, line width=1.5pt, -{Stealth}] (P) -- (Q);

    \draw[darkgray, line width=0.5pt] ($(Q)!0.05!(P)$) -- ($(Q)!0.05!(P)!0.05!90:(P)$) -- ($(Q)!0.05!90:(P)$);

    \node[point, fill=green!60!black, label={[green!60!black, font=\bfseries\footnotesize, text width=2.2cm, align=center, yshift=-0.15cm, xshift=-0.6cm]above right: projection (0.2, 0.8)}] at (Q) {};

    \node[point, fill=red, label={[red, font=\bfseries\scriptsize, xshift=0.35cm, text width=2cm, align=center]below:Strong Model (Miscalibrated) (0.1, 0.7)}] at (P) {};
\end{tikzpicture}
\caption{Projection onto the \textcolor{blue}{reference-compatible set}: transformations induced by ground truth extensions under which the reference is calibrated. The \textcolor{red}{red} point is the primary predictor, and the \textcolor{blue}{blue} line is the constraint imposed by the calibrated reference. Under the chosen Bregman geometry (here we use Brier loss so the projection is Euclidean), the \textcolor{Green}{green} point is the compatible map closest to the primary predictor. The specific numbers are derived from \Cref{example: intro contradictory}.}
\label{fig:projection}
\end{figure}

\subsection{Our Contributions}

\paragraph{Theory}
We formalize the problem of improving a (approximately) miscalibrated primary model using a calibrated reference and develop an efficient Bregman-projection-based algorithm.

\begin{itemize}
    \item \textbf{Condition for Strict Improvement.} In the full-distribution setting, we characterize when strict improvement is possible: precisely when the primary and reference models are \emph{not} mutually calibrated. We link mutual calibration to the No-Trade Theorem in economics, and establish a correspondence between maximal extractable arbitrage profit and calibration error metrics (Appendix~\ref{sec:arbitrage}).

    \item \textbf{Worst-Case Improvement Guarantee.} Our algorithm outputs a transformation that reduces the strong model's expected proper loss by at least its Bregman divergence to the \emph{reference-compatible set} (the ``distance'' between the red and green points in \Cref{fig:projection}) for every ground-truth extension under which the reference is calibrated (\Cref{thm:main_result}). Furthermore, we relax the assumption of a perfectly calibrated reference, quantifying how the guarantee degrades with reference calibration error (\Cref{thm:main_result_relaxed_general}).

    \item \textbf{Generality of the Framework.} Our framework applies broadly to:
    \begin{itemize}
        \item \emph{Diverse Metrics:} Optimization under proper losses satisfying the stated regularity conditions, including Brier Score, Log Loss, and Confidence Loss.
        \item \emph{General Elicitable Properties:} Extensions beyond full predictive distributions to elicitable reports such as means, medians, and top-one classification with confidence.
    \end{itemize}
\end{itemize}

\paragraph{Experiment}

We evaluate our method on three mainstream open-source 8B LLM families across three multiple-choice benchmarks. We set the better-calibrated Base models as \emph{reference} predictors to improve their \emph{primary} Instruct counterparts. Experimental results show significant reductions in both the Brier Score and Expected Calibration Error (ECE) relative to the Instruct models. Despite being label-free, our method substantially improves the calibration of the primary model. Unlike standard post-hoc calibration methods that modify only confidence scores, our approach re-estimates the full probability distribution; empirically, it preserves the high prediction accuracy of the Instruct models and sometimes yields marginal accuracy gains (\Cref{sec:results}).

\subsection{Related Work}

We discuss the most closely related work here; more related work is deferred to Appendix~\ref{apdx: more related work}.

\paragraph{Mutual Calibration and the No-Trade Theorem}
The notion of \emph{mutual calibration} closely parallels the Common Prior Assumption and rationality in the economic literature. Seminal work by \citet{aumann1976agreeing} shows that Bayesian agents who share a common prior cannot ``agree to disagree''. This insight was extended to market settings by the No-Trade Theorem \citep{milgrom1982information}, which states that risk-averse agents with common priors cannot profit from trade in an equilibrium. More recently, \citet{arieli2021feasible} and \citet{morris2020no} characterize the set of \emph{feasible} joint distributions of posterior beliefs as those that admit no arbitrage, formalizing feasibility through arbitrage-free constraints. This notion of feasibility aligns precisely with our definition of mutual calibration. We leverage this connection to show that failures of mutual calibration correspond to exploitable arbitrage opportunities, which in turn can be used in a learning context to systematically optimize model predictions and improve performance.

\paragraph{Calibration Baselines}
Classical post-hoc calibration uses a labeled calibration set, including temperature scaling \citep{guo2017calibration}, histogram binning and isotonic regression \citep{zadrozny2001obtaining,zadrozny2002transforming}, and richer probabilistic methods such as GP calibration and PAC-Bayes recalibration \citep{wenger2020nonparametric,fujisawa2024pac}. These methods serve as supervised baselines, while our target setting is calibration transfer when labels are absent.
Closest to our setting, \citet{luo2026your} and \citet{tan2026basecal} develop label-free calibration methods that exploit the empirical observation that base LLMs are often better calibrated than their post-trained counterparts: \citet{luo2026your} use agreement with the base model to guide unsupervised temperature scaling, while BaseCal~\citep{tan2026basecal} either scores post-trained model responses with the base model or learns a hidden-state projection and uses the base model's output layer to estimate confidence. Unlike our approach, these empirically effective methods lack worst-case loss-improvement guarantees.

\paragraph{Weak-to-Strong Generalization and Self-Improvement.}
A closely related line of work studies weak-to-strong generalization, using supervision from a weak model to improve a stronger model \citep{burns2024weak}. Recent theoretical work analyzes this phenomenon using techniques closely related to ours: \citet{charikar2024quantifying} characterize the gain in weak-to-strong generalization under squared loss, \citet{mulgund2025relating} extend the analysis beyond squared loss, and \citet{yao2025revisiting} compare reverse-KL and forward-KL losses. \citet{yao2025capabilities} relate student calibration to teacher calibration and disagreement, while \citet{xu2026bregman} characterize weak-to-strong risk gains through teacher--student Bregman misfit. These works share our high-level motivation of using a weaker source of information to improve a stronger model and also employ Bregman geometry. The main difference is that standard weak-to-strong generalization typically uses weak labels or pseudo-labels to train a strong student, whereas both predictors remain frozen in our setting: the weak model serves as a calibrated probabilistic reference, and we post-process the strong model using only the unlabeled joint distribution of their reports. On the self-improvement side of language models, \citet{mohri2025coherence} uses Bregman projection to improve coherence across answers to self-generated paraphrases, rather than mutual calibration between two predictors.

\section{Problem Formulation}
\label{sec:problem_statement}

\paragraph{Setup.}
We consider a prediction setting with an outcome variable $\staterv \in \statesp$ and two predictive models. We treat the output of the predictive models as random variables:
\begin{itemize}
    \item $\predRV_1$: The \textbf{primary predictor} (e.g., a strong, but miscalibrated model).
    \item $\predRV_0$: The \textbf{reference predictor} (e.g., a weak, but calibrated model).
\end{itemize}

We first consider the standard case where predictors output full probability distributions over outcomes, $\predRV \in \simplex$. We assume the existence of a calibrated reference predictor. In applications, prior model or domain knowledge may support this assumption. When a small labeled diagnostic set is available, it can provide empirical evidence about calibration performance comparisons (see Appendix~\ref{sec:few-label validation}).

\begin{assumption}[Calibrated Reference]
    We assume the reference predictor $\predRV_0$ is \emph{calibrated}, meaning $\Pr(\staterv = y \mid \predRV_0 = \pred) = \pred(y)$ for all $y \in \statesp$.
\end{assumption}

\paragraph{Assumption relaxation.}
Appendix~\ref{sec:imperfect calibration reference} relaxes perfect reference calibration and derives a lower bound on loss improvement with an explicit calibration-error penalty (\Cref{thm:main_result_relaxed_general}). Under the conditions of that theorem, the bound guarantees strict improvement when the geometric gain exceeds the miscalibration penalty. Empirically, we observe that even when the reference is only better calibrated than the primary, our method can exploit this relative advantage to reduce error.

We assume access to the joint distribution over $\predRV_0$, $\predRV_1$, denoted by $\prob_{\predRV_0, \predRV_1}$, estimated by the empirical joint distribution of the two predictors' outputs on an \textbf{unlabeled} dataset. We focus on finite supports $\predsp_0,\predsp_1$. We say a distribution $\widetilde\prob_{\predRV_0,\predRV_1,\staterv}$ \emph{extends} $\prob_{\predRV_0,\predRV_1}$ when the latter is exactly its $(\predRV_0,\predRV_1)$-marginal.

\textbf{Problem Statement.}
We aim to learn a transformation map $\calmap:\predsp_1\mapsto\Delta(\statesp)$ solely from $\prob_{\predRV_0, \predRV_1}$ such that $\calmap(\predRV_1)$ strictly outperforms $\predRV_1$ under a chosen proper loss, ``transferring'' calibration information from $\predRV_0$ to $\predRV_1$ without requiring $\predRV_0$ during subsequent deployment.

Formally, we evaluate performance using a \emph{(strictly) proper loss} $\score: \simplex \times \statesp \to \mathbb{R}$, such that the expected loss is (uniquely) minimized when the predicted distribution matches the true posterior. Specifically, for any true distribution $\truepred \in \Delta(\statesp)$ of the outcome $\state$, it holds that \(\E_{\state \sim \truepred}\big[\score(\truepred, \state)\big]\leq\E_{\state \sim \truepred}\big[\score(\pred, \state)\big]\) for any forecast $\pred$. Common examples include the Log-Loss and the Brier Score.

\begin{definition}[Strictly Improving Transformation]
    Given $\prob_{\predRV_0, \predRV_1}$, a deterministic transformation $\calmap:\predsp_1\to\Delta(\statesp)$ is \emph{strictly improving} if:
    \begin{enumerate}
        \item For every distribution $\prob_{\predRV_0, \predRV_1, \staterv}$ extending $\prob_{\predRV_0, \predRV_1}$ under which $\predRV_0$ is calibrated, applying $\calmap$ weakly improves the loss: $\E[\score(\calmap(\predRV_1), \staterv)] \leq  \E[\score(\predRV_1, \staterv)]$.
        \item There exists such an extension $\prob_{\predRV_0, \predRV_1, \staterv}$ leading to strict inequality.
    \end{enumerate}
\end{definition}

We aim to find a strictly improving transformation. This is not guaranteed in all cases. For instance, if both predictors \(\predRV_0\) and \(\predRV_1\) are already providing the ground truth, no transformation will improve the performance.
Thus, we will quantify the lower bound of improvement as a function of \(\prob_{\predRV_0, \predRV_1}\), and explore the conditions on \(\prob_{\predRV_0, \predRV_1}\) that make a strictly improving transformation possible.

In the full-distribution setting, we obtain an even stronger result: whenever a strictly improving transformation exists, a projection $\calmap^*$ achieves a uniformly positive expected-loss reduction over \emph{all} valid ground truth distributions \(\prob_{\predRV_0, \predRV_1, \staterv}\).

\subsection{Extension to General Elicitable Properties}
\label{sec:properties}

In many practical applications, models report summary statistics rather than full probability distributions. We generalize our framework to \emph{elicitable properties}: those for which truthful reporting is incentivized by a proper loss \citep{lambert2008eliciting}. Our central example is the (mode, max-probability) pair used throughout our LLM experiments.

\begin{example}[Mode and Confidence]
\label{example3.3:mode and confidence}
    Consider a Question-Answering (QA) system where the predictor outputs a discrete answer and a confidence score.
    \begin{itemize}
        \item \textbf{Property $\prop$:} The pair $(\text{mode}, \text{max-probability})$:
        \begin{equation*}
        \prop(\pred) = \left\{ (y, c) \mid y \in \argmax_{y' \in \statesp} \pred(y'), \ c = \max_{y' \in \statesp} \pred(y') \right\}.
        \end{equation*}
        \item \textbf{Report $\rep$:} A tuple $(\hat{y}, \hat{c})$ with predicted class $\hat{y}\in\statesp$ and confidence $\hat{c}\in[0,1]$.
        \item \textbf{Loss $\score$:} A strictly proper loss for this property is \citep{damani2025beyond}:
        \begin{equation}
        \label{eq:proper loss}
            \score((\hat{y}, \hat{c}), y) = -\mathbb{I}(y = \hat{y}) + (\hat{c}-\mathbb{I}(y = \hat{y}))^2.
        \end{equation}
    \end{itemize}
    We call it the \emph{confidence loss}. This loss function penalizes both incorrect classification (accuracy) and poor confidence assessment (calibration).
\end{example}

We say the reference predictor $\repRV_0$ is calibrated with respect to property $\prop$ if its report matches the property of the conditional distribution given that report:
\[
\rep \in \prop(\prob_{\staterv \mid \repRV_0=\rep}).
\]
for every possible report $\rep$.

For property reports, we allow a random transformation that assigns a finitely supported distribution $\pi_{\rep}$ on the allowable report space $\cA_1$ to each $\rep\in\repsp_1$, where $\repsp_1$ denotes the finite support of the primary predictor. Conditional on $\repRV_1=\rep$, the transformed report $\widehat R\sim\pi_{\rep}$ is independent of $(\repRV_0,\staterv)$. Deterministic rules correspond to point masses. We seek a strictly positive worst-case expected-loss reduction over all extensions of the observed joint report distribution $\prob_{\repRV_0,\repRV_1}$ under which the reference $\repRV_0$ is calibrated with respect to the property. For each such extension, the expectation is over $(\repRV_0,\repRV_1,\staterv)$ and the output randomization. Our construction of the optimal transformation combines a distribution-valued projection $\calmap:\repsp_1\to\Delta(\statesp)$ with a suitable report-selection rule supported on $\prop(\calmap(\rep))$. See Appendix~\ref{app:elicitable setup} for details.

\section{Solution and Theoretical Guarantees}
\label{sec:solution_theory}

This section presents our method for improving a primary predictor given a calibrated reference. We adopt a geometric perspective, viewing transformed predictors as points. Feasible transformed predictors form a compact, convex region. If a primary predictor is not mutually calibrated with the reference predictor, it lies outside this region. Our approach projects such a predictor back onto this reference-compatible set, yielding a worst-case proper-loss improvement guarantee.

\subsection{The Condition for Improvement: Mutual Calibration}

We first introduce the central concept determining the possibility of strict improvement: \textbf{Mutual Calibration}. This concept characterizes the consistency of the joint distribution of two predictors without requiring access to the ground truth outcome.

\begin{definition}[Mutual Calibration]
\label{def:mutual-calibration}
Given a joint distribution $\prob_{\predRV_0,\predRV_1}$, predictors $\predRV_0$ and $\predRV_1$ are \textbf{mutually calibrated} if there exists an extension $\widetilde\prob_{\predRV_0,\predRV_1,\staterv}$ with $(\predRV_0,\predRV_1)$-marginal $\prob_{\predRV_0,\predRV_1}$ under which both predictors are calibrated with respect to the outcome $\staterv$.
\end{definition}

Our main theorem (Theorem~\ref{thm:main_result}) shows that a strictly improving transformation exists if and only if mutual calibration fails.

\paragraph{Economic Interpretation.}
Mutually calibrated predictors can be interpreted as Bayesian agents whose reports are posteriors under a common prior, where the extension $\widetilde\prob_{\predRV_0,\predRV_1,\staterv}$ serves as the shared prior distribution. Therefore, characterizing whether two predictors are mutually calibrated given joint distribution $\prob_{\predRV_0,\predRV_1}$ is closely related to checking whether these two predictors share a common prior, which is a fundamental and widely-studied question in the theoretical economic literature. The No-Trade Theorem \citep{milgrom1982information} and related feasibility results \citep{arieli2021feasible,morris2020no} characterize when the joint distribution $\prob_{\predRV_0,\predRV_1}$ can be derived from a common prior in terms of \textbf{no-arbitrage guarantee}. A failure of mutual calibration implies the existence of an arbitrage opportunity, and in Appendix~\ref{sec:arbitrage} we formally quantify the relationship between the maximal arbitrage space and the predictors' calibration errors. This can be viewed as an economic interpretation to our mutual calibration concept.

\subsection{Algorithm: Projection onto Reference-Compatible Set}

When mutual calibration fails, we propose the following algorithm to generate a strictly improving transformation.

Given a joint distribution between the primary and reference predictors, the algorithm has two steps: first, define a \textbf{reference-compatible set} of predictors as the set of all feasible transformations. Second, \textbf{project} the primary predictor onto this set by finding the closest reference-compatible predictor in terms of Bregman divergence.

\subsubsection{Step 1: The Reference-Compatible Set}

Let $\calmapsp=\{\calmap\mid\calmap:\predsp_1\mapsto\Delta(\statesp)\}$ denote the space of all functions mapping the primary predictions to a full prediction over outcome space. Given $\prob_{\predRV_0,\predRV_1}$, define
\begin{equation}
\label{eq:compatible-posterior-set}
\calmapsp_{\mathrm{comp}}(\prob_{\predRV_0,\predRV_1})
=\left\{
h\in\calmapsp:
\begin{array}{l}
\exists\,\widetilde\prob_{\predRV_0,\predRV_1,\staterv}\text{ with marginal }\prob_{\predRV_0,\predRV_1}\text{ such that}\\[0.5ex]
\E_{\widetilde\prob}[\statervvec\mid\predRV_0]=\predRV_0,
\quad
\E_{\widetilde\prob}[\statervvec\mid\predRV_1]=h(\predRV_1)
\end{array}
\right\}.
\end{equation}

The definition has an explicit linear representation, implying the compactness and convexity of the set. Let $z(y,\pred_0,\pred_1)=\Pr_{\widetilde\prob}(\staterv=y,\predRV_0=\pred_0,\predRV_1=\pred_1)$, then $h\in\calmapsp_{\mathrm{comp}}(\prob_{\predRV_0,\predRV_1})$ if and only if there exists a nonnegative array $z$ satisfying
\begin{equation}
\label{eq:latent-mass-constraints}
    \begin{aligned}
        \sum_y z(y,\pred_0,\pred_1)&=\Pr_{\prob}(\predRV_0=\pred_0,\predRV_1=\pred_1) \qquad \forall \pred_0 \forall \pred_1, \\
        \sum_{\pred_1}z(y,\pred_0,\pred_1)&=\pred_0(y)\sum_{\pred_1}\Pr_{\prob}(\predRV_0=\pred_0,\predRV_1=\pred_1) \qquad \forall \pred_0 \forall y, \\
        \sum_{\pred_0}z(y,\pred_0,\pred_1)&=h(\pred_1)(y)\sum_{\pred_0}\Pr_{\prob}(\predRV_0=\pred_0,\predRV_1=\pred_1) \qquad \forall \pred_1 \forall y.
    \end{aligned}
\end{equation}

\begin{restatable}{observation}{Geometry}
\label{obs:compatible-geometry}
With finite outcome and report supports, $\calmapsp_{\mathrm{comp}}(\prob_{\predRV_0,\predRV_1})$ is nonempty, compact, and convex. Moreover,
\[
\operatorname{id}\in\calmapsp_{\mathrm{comp}}(\prob_{\predRV_0,\predRV_1})
\quad\Longleftrightarrow\quad
\predRV_0,\predRV_1\text{ are mutually calibrated}.
\]
\end{restatable}

Under the loss assumptions stated below, compactness ensures that the projection objective attains its minimum, while convexity permits standard convex optimization methods.

\subsubsection{Step 2: Bregman Projection}

To define the projection, we require a notion of ``distance''. Since we aim to reduce proper loss, the natural choice is the Bregman divergence generated by the proper loss $\score$ \citep{savage1971elicitation}. Formally, the divergence between a true distribution $\truepred$ and a prediction $\pred$ is the expected regret:
\begin{equation*}
D_{\score}(\truepred, \pred) = \E_{\widetilde Y \sim \truepred}[\score(\pred, \widetilde Y) - \score(\truepred, \widetilde Y)].
\end{equation*}

Here $\widetilde Y$ is an auxiliary draw from $\truepred$, not an observed ground-truth label.

We seek the reference-compatible transformation $\calmap \in \calmapsp_{\mathrm{comp}}(\prob_{\predRV_0,\predRV_1})$ that remains closest to the primary predictor $\predRV_1$. Crucially, this projection is equivalent to minimizing the expected divergence (or loss gain) relative to the primary predictor:

\begin{equation}
\label{eq:bregproj}
    \begin{aligned}
        \calmap^*
        &\in\operatorname*{argmin}_{h\in\calmapsp_{\mathrm{comp}}(\prob_{\predRV_0,\predRV_1})}
        \E_{\prob_{\predRV_1}}\!\left[D_{\score}(h(\predRV_1),\predRV_1)\right] \\
        &=\operatorname*{argmin}_{h\in\calmapsp_{\mathrm{comp}}(\prob_{\predRV_0,\predRV_1})}
        \E_{\prob_{\predRV_1}}\!\left[
        \E_{\widetilde Y\sim h(\predRV_1)}
        \big[\score(\predRV_1,\widetilde Y)-\score(h(\predRV_1),\widetilde Y)\big]
        \right].
    \end{aligned}
\end{equation}

Note that this formulation differs from a standard Bregman projection on the probability simplex because we operate on the space of transformations (functions). We show that this objective induces a valid Bregman divergence on the finite-dimensional space of maps without accessing ground truth labels. Compactness guarantees that a minimizer exists.

\subsection{Theoretical Guarantees}

In this section, we show that the projection-based transformation in \Cref{eq:bregproj} improves the primary predictor under \emph{every} ground-truth extension consistent with the observed reports and a calibrated reference. Specifically, the guaranteed reduction is at least the Bregman distance from the identity to the reference-compatible Set, and this lower bound is attained by a worst-case extension.

We impose mild conditions on proper loss. Every proper loss is associated with a convex generator $\genFunc(\pred)=-\E_{Y\sim \pred}[\score(\pred,\staterv)]$.  We assume finite supports and a strictly proper loss whose strictly convex generator is differentiable with finite gradients on a convex prediction domain containing the identity and the reference-compatible set. Brier loss satisfies these conditions directly; log loss is covered when the relevant forecasts remain in the relative interior of the simplex.

\begin{restatable}{theorem}{mainResult}
\label{thm:main_result}
Given a joint distribution over primary and reference predictors $\prob_{\predRV_0,\predRV_1}$. Let $\calmap^*$ be defined in \Cref{eq:bregproj}.  For every extension $\widetilde\prob_{\predRV_0,\predRV_1,\staterv}$ with $(\predRV_0,\predRV_1)$-marginal $\prob_{\predRV_0,\predRV_1}$ under which the reference predictor is calibrated,
\begin{equation}
\label{eq:uniform-main-guarantee}
\begin{split}
&\E_{\widetilde\prob}[\score(\predRV_1,\staterv)]
-\E_{\widetilde\prob}[\score(\calmap^*(\predRV_1),\staterv)] \geq \delta,\\
&\delta
:=\min_{h\in\calmapsp_{\mathrm{comp}}(\prob_{\predRV_0,\predRV_1})}
\E_{\prob_{\predRV_1}}[D_{\score}(h(\predRV_1),\predRV_1)].
\end{split}
\end{equation}
The value $\delta$ is strictly positive if and only if $\predRV_0$ and $\predRV_1$ are not mutually calibrated.  Consequently, a strictly improving deterministic transformation exists if and only if mutual calibration fails.  Moreover, some admissible extension attains equality in \Cref{eq:uniform-main-guarantee}, so $\delta$ is the exact worst-case gain.
\end{restatable}

The proof is short and directly follows from the generalized Pythagorean Theorem for Bregman Divergence defined on the functional space of transformations. A related projection argument appears in \citet{mohri2025coherence} for self-coherence constraints. We provide a complete proof in Appendix~\ref{apdx: proof-main} for reference. The total expected proper loss is lower bounded by the sum of two parts: the ``distance'' between the current prediction and the reference-compatible set and the expected proper loss of the projected predictor. The projected predictor will reduce the loss by at least this ``distance'', and thus the improvement is lower bounded by this ``distance''.

\paragraph{Assumption relaxation}
We relax the assumption that the reference is perfectly calibrated in \Cref{thm:main_result_relaxed_general}.
Our analysis shows that the loss improvement achieved by $\calmap^*$ admits a three-term decomposition into a geometric gain, a nonnegative compatibility bonus, and a signed miscalibration residual. Dropping the compatibility bonus and controlling the residual yields
\[
\E[\score(\predRV_1,\staterv)]
-\E[\score(\calmap^*(\predRV_1),\staterv)]
\geq A-B\sqrt{\epsilon_{\mathrm{ref}}},
\]
where $\epsilon_{\mathrm{ref}}=\E\|\E[\statervvec\mid\predRV_0]-\predRV_0\|_1$ characterizes the calibration error of the reference predictor, and $A=D_{\score}(\calmap^*,\operatorname{id})$ and $B$ is the root-mean-square gradient displacement. Thus, strict improvement is preserved whenever the geometric gain dominates the miscalibration penalty, i.e. $A>B\sqrt{\epsilon_{\mathrm{ref}}}$.

\subsection{Extension to General Elicitable Properties}

The framework extends to elicitable reports such as means, medians, and (answer, confidence) pairs. Under the regularity conditions in Appendix~\ref{app:elicitable}, an optimal distribution-valued projection $\calmap^*:\repsp_1\to\Delta(\statesp)$ together with a suitable, possibly randomized report-selection rule supported on $\prop(\calmap^*(\rep))$ attains the exact maximum worst-case expected-loss improvement $\delta_{\prop}$ under the original property loss. This value is positive if and only if mutual calibration fails. When the ambient Bayes risk is differentiable at every projected distribution, any deterministic Bayes-optimal report selection attains the same guarantee. Appendix~\ref{app:elicitable} gives the full construction and proof.

\begin{table*}[!t]
    \centering
    \caption{\textbf{Main results on MMLU-Redux and CommonSenseQA.} \textbf{Bold} and \underline{underline} denote the \textbf{best} and \underline{second-best} results among \emph{label-free} methods, while \dag indicates the \emph{global} best across all methods. Accuracy for Supervised Baselines is marked with \textit{--} as these post-hoc methods do not alter final predictions. Detailed results with more baselines and error bars can be found in Appendix~\ref{sec:additional baselines}.}
    \label{tab:main_results}
    \resizebox{\textwidth}{!}{
    \begin{tabular}{l|cccc|cccc}
        \toprule
        \multirow{2}{*}{\textbf{Methods}} & \multicolumn{4}{c|}{\textbf{MMLU-Redux}} & \multicolumn{4}{c}{\textbf{CommonSenseQA}} \\
         & \textbf{ACC} ($\uparrow$) & \textbf{BS} ($\downarrow$) & \textbf{ECE} ($\downarrow$) & \textbf{CL} ($\downarrow$) & \textbf{ACC} ($\uparrow$) & \textbf{BS} ($\downarrow$) & \textbf{ECE} ($\downarrow$) & \textbf{CL} ($\downarrow$) \\
        \midrule
        \multicolumn{9}{c}{\textbf{Qwen3-8B}} \\
        \midrule
        Base Model ($\predRV_0$) & 77.19\% & \underline{0.1227} & \textbf{0.0305}\dag & $-$0.6493 & \textbf{83.21\%}\dag & \textbf{0.1167}\dag & \textbf{0.0611} & \textbf{$-$0.7154}\dag \\
        Instruct Model ($\predRV_1$) & \underline{82.53\%} & 0.1362 & 0.1269 & $-$0.6891 & \underline{82.80\%} & 0.1447 & 0.1355 & $-$0.6833 \\
        \cellcolor{gray!10} TS \textit{(Supervised)} & \cellcolor{gray!10}\textit{--} & \cellcolor{gray!10}0.1137\dag & \cellcolor{gray!10}0.0426 & \cellcolor{gray!10}$-$0.7116\dag & \cellcolor{gray!10}\textit{--} & \cellcolor{gray!10}0.1178 & \cellcolor{gray!10}0.0266\dag & \cellcolor{gray!10}$-$0.7102 \\
        DACA \textit{(Label-Free)} & \textit{--} & \textbf{0.1225} & 0.0985 & \textbf{$-$0.7028} & \textit{--} & \underline{0.1226} & \underline{0.0631} & \underline{$-$0.7054} \\
        \textbf{Ours} \textit{(Label-Free)} & \textbf{82.56\%}\dag & 0.1236 & \underline{0.0348} & \underline{$-$0.7020} & \underline{82.80\%} & 0.1309 & 0.0638 & $-$0.6971 \\
        \midrule
        \multicolumn{9}{c}{\textbf{LLaMA-3.1-8B}} \\
        \midrule
        Base Model ($\predRV_0$) & 65.35\% & \textbf{0.1574}\dag & \textbf{0.0208}\dag & $-$0.4961 & \underline{71.33\%} & \underline{0.1708} & \textbf{0.0363}\dag & $-$0.5425 \\
        Instruct Model ($\predRV_1$) &  \textbf{70.91\%}\dag & 0.2550 & 0.2526 & $-$0.4542 & \textbf{78.30\%}\dag & 0.2099 & 0.2097 & $-$0.5731 \\
        \cellcolor{gray!10} TS \textit{(Supervised)} & \cellcolor{gray!10}\textit{--} & \cellcolor{gray!10}0.1873 & \cellcolor{gray!10}0.0364 & \cellcolor{gray!10}$-$0.5218\dag & \cellcolor{gray!10}\textit{--} & \cellcolor{gray!10}0.1671\dag & \cellcolor{gray!10}0.0471 & \cellcolor{gray!10}$-$0.6158\dag \\
        DACA \textit{(Label-Free)} & \textit{--} & \underline{0.1878} & \underline{0.0453} & \textbf{$-$0.5213} & \textit{--} & \textbf{0.1703} & \underline{0.0740} & \textbf{$-$0.6126} \\
        \textbf{Ours} \textit{(Label-Free)} & \textbf{70.91\%}\dag & 0.2115 & 0.1340 & \underline{$-$0.4976} & \textbf{78.30\%}\dag & 0.1776 & 0.1054 & \underline{$-$0.6053} \\
        \midrule
        \multicolumn{9}{c}{\textbf{Ministral-3-8B}} \\
        \midrule
        Base Model ($\predRV_0$) & 75.72\% & \textbf{0.1281}\dag & \underline{0.0534} & $-$0.6291 & 69.94\% & \underline{0.1948} & \underline{0.1349} & $-$0.5046 \\
        Instruct Model ($\predRV_1$) & \underline{80.30\%} & 0.1726 & 0.1656 & $-$0.6304 & \underline{75.35\%} & 0.2050 & 0.1691 & \underline{$-$0.5485} \\
        \cellcolor{gray!10} TS \textit{(Supervised)} & \cellcolor{gray!10}\textit{--} & \cellcolor{gray!10}0.1476 & \cellcolor{gray!10}0.0445 & \cellcolor{gray!10}$-$0.6553 & \cellcolor{gray!10}\textit{--} & \cellcolor{gray!10}0.1930 & \cellcolor{gray!10}0.1111 & \cellcolor{gray!10}$-$0.5605 \\
        DACA \textit{(Label-Free)} & \textit{--} & 0.1534 & 0.0914 & \underline{$-$0.6495} & \textit{--} & 0.2098 & 0.1674 & $-$0.5437 \\
        \textbf{Ours} \textit{(Label-Free)} & \textbf{80.48\%}\dag & \underline{0.1475} & \textbf{0.0218}\dag & \textbf{$-$0.6573}\dag & \textbf{75.41\%}\dag & \textbf{0.1779}\dag & \textbf{0.0571}\dag & \textbf{$-$0.5762}\dag \\
        \bottomrule
    \end{tabular}
    }
\end{table*}

\section{Experimental Setup}
\label{sec:experiments}

\textbf{Datasets.}
We evaluate our framework on three multiple-choice question answering benchmarks: \textbf{MMLU-Redux}~\citep{gema2025we}, \textbf{CommonsenseQA}~\citep{talmor2019commonsenseqa}, \textbf{MMLU}~\citep{hendrycks2021measuring}. For each benchmark, we use two disjoint sets: a calibration set for learning the recalibration mapping and an evaluation set for assessing performance. Dataset-specific splits are detailed in Appendix~\ref{sec:dataset details}.

\textbf{Models.}
We consider the following open-source LLMs: \texttt{QWen3-8B}~\citep{yang2025qwen3}, \texttt{LLaMA-3.1-8B}~\citep{grattafiori2024llama}, and \texttt{Ministral-3-8B}~\citep{liu2026ministral}.
Consistent with recent findings \citep{achiam2023gpt,leng2025taming}, we observe an alignment trade-off:
\begin{itemize}
    \item \textbf{Primary Predictor ($\predRV_1$):} We select the \textbf{Instruct} checkpoints, which offer superior accuracy but suffer from miscalibration.
    \item \textbf{Reference Predictor ($\predRV_0$):} We select the corresponding \textbf{Base} checkpoints, which retain the better-calibrated distributions of the pre-training corpus.
\end{itemize}
Our preliminary analysis (Table~\ref{tab:main_results}) confirms that \textbf{Base} models consistently yield lower ECE than \textbf{Instruct} models in our main settings, motivating the choice of the reference predictor in Section~\ref{sec:problem_statement}. When prior knowledge about model families is unavailable, Appendix~\ref{sec:few-label validation} further demonstrates that this calibration gap can be checked using only a small labeled diagnostic set, providing a practical criterion for selecting the reference model. Note here the small labeled diagnostic set is optional and only used for comparing calibration performance. These labels are \emph{not} used to fit the map.

\textbf{Implementation Details.}
While our framework is compatible with any proper loss under the stated regularity conditions, LLM applications typically involve eliciting an (answer, confidence) pair. Therefore, we instantiate our optimization objective using the \textbf{Confidence Loss (CL)} defined in Equation~(\ref{eq:proper loss}), which is designed to truthfully elicit such pairs. The reports are obtained from normalized answer-token log-probabilities of the open-source LLMs. We approximate $\prob_{\repRV_0, \repRV_1}$ by the empirical distribution over discretized paired reports on the calibration set, and restrict the transformation $\calmap$ to the observed support of the primary reports. For Confidence Loss, the finite-support projection yields a second-order cone formulation with linear reference-compatibility constraints. Given LLM inference results, we solve this lightweight problem with a standard convex solver.

At deployment, we first discretize the primary report's confidence using the same grid as in calibration. Then we apply the learned mapping to observed discretized reports and retain unseen discretized reports unchanged, without querying the reference model. Refer to Appendix~\ref{sec:implementation details} for details and Appendix~\ref{sec:discretization sensitivity} for the discretization analysis.

\textbf{Baselines.}
Unlike fine-tuning approaches which are computationally intensive, our framework operates entirely as a \emph{lightweight post-processing} over frozen LLMs.
We compare against both supervised and label-free post-hoc calibration methods. Table~\ref{tab:main_results} compares our method with supervised \textbf{Temperature Scaling (TS)}~\citep{guo2017calibration} and label-free \textbf{Disagreement-Aware Confidence Alignment (DACA)}~\citep{luo2026your}, which fits the temperature by minimizing the KL divergence between the reference and the primary model's outputs on data where both models agree on the final answer. Further details and comparisons appear in Appendix~\ref{sec:baselines details} and Appendix~\ref{sec:additional baselines}.

\textbf{Evaluation Metrics.}
We report evaluations from three perspectives: 1) \textbf{Prediction Accuracy (Acc)} for LLM capabilities 2) \textbf{Expected Calibration Error (ECE)}~\citep{guo2017calibration} with $M=10$ bins and \textbf{Brier Score (BS)}~\citep{brier1950verification} as standard calibration metrics 3) the instantiated \textbf{Confidence Loss (CL)} in Eq.~\ref{eq:proper loss} as the optimization objective. To verify robustness to the choice of calibration-error estimator, we additionally evaluate alternative binned ECE estimators and automatic-bandwidth SmoothECE~\citep{blasiok2024smooth} (Appendices~\ref{sec:evaluation metrics formula} and~\ref{sec:alternative ece}).

\section{Results}
\label{sec:results}

Table~\ref{tab:main_results} summarizes our main experimental results on \textbf{MMLU-Redux} and \textbf{CommonSenseQA}. Due to space constraints, we defer the results on \textbf{MMLU} to Appendix~\ref{sec:additional results for MMLU}.

\textbf{Significant Reduction in Calibration Error.} Our method achieves substantial reductions in calibration error compared to Instruct models across all benchmarks, with improvements exceeding $8\%$ in Brier Score (BS) and $40\%$ in Expected Calibration Error (ECE). Compared to the supervised baseline TS and the label-free baseline DACA, our method achieves lower ECE for \texttt{Qwen3-8B} on MMLU-Redux and for \texttt{Ministral-3-8B} on both main benchmarks. On CommonsenseQA with \texttt{Qwen3-8B}, TS has lower ECE, while DACA and our method obtain similar values ($0.0631$ and $0.0638$, respectively). Another exception occurs for \texttt{LLaMA-3.1-8B}, which we analyze below. These gains correlate directly with the optimization of Confidence Loss (CL), confirming that our framework effectively minimizes the proper loss to balance prediction correctness and confidence reliability.

\textbf{Preservation/Marginal Gains of Accuracy.}
Unlike standard post-hoc calibration methods that only modify confidence and thus fix accuracy (marked with \textit{--} in Table~\ref{tab:main_results}), our method re-estimates the full probability vector, allowing for corrections in the final answer. Results indicate that our approach effectively preserves the high predictive accuracy of the Instruct models. Furthermore, on MMLU-Redux with \texttt{Ministral-3-8B}, we observe an accuracy improvement from 80.30\% to 80.48\%, suggesting that leveraging the Base model may correct instances where the Instruct model is confidently incorrect.

\textbf{Robustness.}
Additional results in Appendix~\ref{sec:additional baselines} compare our method with other supervised and label-free baselines, showing similar trends. Specifically, we create a \emph{dummy reference} by randomizing the option order and always predicting a uniform distribution. This \emph{dummy reference} yields much smaller reductions in ECE, confirming that our framework indeed leverages meaningful information from the reference model. Appendices~\ref{sec:discretization sensitivity} and~\ref{sec:alternative ece} examine different discretization steps and calibration-error estimators. Across the alternative binned ECE variants and non-binned SmoothECE, our method consistently improves over the Instruct model.

\textbf{Failure Case Analysis.}
We call \texttt{LLaMA-3.1-8B} a ``failure case'' only in the relative sense that TS and DACA obtain lower ECE. Our method still reduces ECE from $0.2526$ to $0.1340$ on MMLU-Redux and from $0.2097$ to $0.1054$ on CommonsenseQA, while also reducing Confidence Loss relative to the Instruct model. In this case, the Base model places unusually large mass in low-confidence regions, indicating lower-quality reference information and limiting calibration transfer from Base model to Instruct model.

\section{Discussion}
\label{sec:discussion}

\paragraph{Limitations.}
Our results establish a label-free method for improving a primary model when a (approximately) calibrated reference is available. Consequently, the method relies on prior model or domain knowledge that supports the reference-quality assumption. When such knowledge is unavailable, a small labeled diagnostic set can be used to guide reference selection, but these labels do not enter the recalibration-map fitting process.

Another limitation comes from the number of constraints and variables when solving the recalibration map. Although the finite-support multiple-choice problems in our experiments reduce to lightweight convex programs after LLM inference, their size grows with the number of possible LLM reports. Therefore, directly applying our algorithm to open-ended scenarios may be computationally infeasible.

In practice, we recommend applying our algorithm only when reference selection is supported by prior knowledge or can be guided empirically. Further, implementing our algorithm involves discretizing input confidence scores. This discretization plays the usual bias--variance role in nonparametric estimation: coarser grids give more stable estimates, while finer grids retain more information about the reports. Our experiments show that the method is robust across a range of discretization choices, with $s=0.01$ providing a strong default in our main settings. More generally, one can choose the finest grid for which the empirical joint distribution remains sufficiently populated.

\paragraph{Future Work.}
The framework points to several extensions. First, the projection view naturally suggests online and adaptive variants, where the reference-compatible set is updated as new unlabeled examples arrive. Second, the same mutual-calibration principle can be applied beyond multiple-choice QA to richer elicitable reports, including open-ended generation with semantic confidence, human-AI decision support, and subjective annotation tasks. Third, because the method transfers a property rather than a label, it may be useful for aligning other distributional properties, such as uncertainty, abstention behavior, or fairness-related summaries, whenever a trusted reference for that property is available.

\begin{ack}
This work was supported by the National Natural Science Foundation of China (NSFC) under Grant Nos. 62372007 and 6247070125.

\paragraph{Competing Interests.}
The authors declare no competing interests.
\end{ack}

\bibliographystyle{plainnat}
\bibliography{ref}

\appendix

\section{More Related Work}
\label{apdx: more related work}
\paragraph{Bregman Projection}
Bregman projection has been widely studied in a number of learning theory problems, e.g, maximum entropy models \citep{CollinsSchapireSinger2002, DudikPhillipsSchapire2007, AltunSmola2006, BanerjeeDhillonGhoshMerugu2007, frongillo2014general, Harremoes2016, NockMenonOng2016, CortesKuznetsovMohriSyed2015}, boosting \citep{CollinsSchapireSinger2002, KivinenWarmuth1999}, clustering \citep{banerjee2005clustering}, Bayesian estimation \citep{frigyik2008functional}.

\paragraph{Model Distillation} Model distillation trains a model with the output of a reference model and aims at better performance than training on hard ground-truth labels alone (see, e.g., \citealt{hinton2015distilling, sanh2019distilbert, guminillm, agarwal2024policy}). Our work is similar to model distillation in the sense that we ``distill'' the calibration of the weak model. Distillation is most commonly used with a reference model at least as good as the trainable model, while our work is closer to the weak-to-strong generalization.

\paragraph{Calibration under Distribution Shift} The classical literature of calibration focuses on post-processing model outputs with ground truth \citep{platt1999probabilistic, zadrozny2002transforming, kull2019beyond}. The closest work to ours improves model calibration under distribution shifts. The goal of this literature is similar to our work, which is to measure or to improve calibration when there are no ground-truth labels available for the shifted distribution. This literature usually makes assumptions of the distribution shift. The proposed solution relies on the assumptions. For example, \citet{alexandari2020maximum, popordanoska2023estimating} assume that only the marginal distribution of the true outcomes is shifted while the conditional distribution of the features remains the same. \citet{park2020calibrated} assumes that the shifted distribution is sufficiently close to the training data. Our work is orthogonal to the work on distribution shift. We use a weak calibrated model as a reference and correct contradictions in predicted probabilities.

\paragraph{Calibration of Language Models}
A growing literature studies the calibration of confidence from language models.
\citet{achiam2023gpt} observe that the next-token probability of a language model can be miscalibrated, especially after post-training.
Some existing approaches study or improve semantic-level confidence calibration by prompting the model to output a confidence score, e.g., \citet{lin2022teaching, tian-etal-2023-just, leng2025taming, damani2025beyond}. While this line of work considers classical calibration with ground truth labels,
our work does not assume access to labeled data for fitting the map. Instead, we leverage the empirical observation that base pretrained checkpoints are often better calibrated than their instruct counterparts \citep{achiam2023gpt,luo2026your,tan2026basecal}, and we transfer this calibration as a distributional property from a weak reference to a stronger primary model via label-free post-processing.

\paragraph{Calibration with Inference-Time Compute}
\citet{dadkhahi2025distribution} study distribution-calibrated inference-time compute for thinking LLM-as-a-judge systems. Their setting calibrates aggregation over repeated judge samples using labeled or human-consensus evaluation targets. Our setting is different: we learn a post-processing map from the joint predictions of a primary and a reference model without target labels, and the learned map does not require repeated reference queries at deployment.

\section{Proofs of the Main Results}
\label{apdx: proof-main}

In this section, we provide the proofs for the main result. We first introduce the necessary preliminaries linking strictly proper losses to Bregman divergences.

\subsection{More Preliminaries: Proper Losses and Bregman Divergence}
\label{prelim: proper loss}

\begin{lemma}[Convex representation of proper losses; \citealt{savage1971elicitation,rockafellar1997convex}]
\label{lem:proper-representation}
For a finite outcome space, a loss $\score$ is proper if and only if there is a convex generator $\genFunc: \Delta(\statesp) \to \mathbb{R}$ and a (possibly subgradient) mapping $\psi(\pred)\in\partial\genFunc(\pred)$ such that, for all forecasts $\pred$ and outcomes $\state$,
\[
\score(\pred,y)=-\big(\genFunc(\pred)+\langle\psi(\pred),\mathbf e_y-\pred\rangle\big).
\]
where $\mathbf{e}_{\state}$ is the one-hot vector corresponding to outcome $\state$. If $\score$ is strictly proper, $\genFunc$ is strictly convex on the relevant affine hull.
\end{lemma}

\begin{remark}[Regret as Bregman divergence]
\label{remark:bregman}
The expected regret of forecasting $\pred$ when the true distribution is $\truepred$ is
\begin{align}
\E_{Y\sim \truepred}[\score(\pred,Y)-\score(\truepred,Y)]
&=\genFunc(\truepred)-\genFunc(\pred)-\langle\psi(\pred),\truepred-\pred\rangle =:D_{\score}(\truepred,\pred). \label{eq:pointwise-bregman}
\end{align}
\end{remark}

Throughout the differentiable results, the displayed gradients and divergences are assumed finite on the prediction domain.

As in \citet{mohri2025coherence}, we extend this divergence to the space of transformations $\mathcal{H}$, by taking expectations over the marginal distribution of the primary predictor's reports.

\begin{observation}
\label{obs:extendbreg}
On the positive-probability support of the primary predictor, a Bregman divergence on $\Delta(\statesp)$ induces a Bregman divergence on the space of transformations $\mathcal{H}$:
\[
D_{\score}(f,g)
=\E_{\prob_{\predRV_1}}\big[D_{\score}(f(\predRV_1),g(\predRV_1))\big].
\]
If $\genFunc$ is strictly convex and differentiable with finite gradients on the prediction domain, this is the Bregman divergence generated by a strictly convex and differentiable generator function $\mathcal G(f)=\sum_{\pred_1}\Pr_{\prob}(\predRV_1=\pred_1)\genFunc(f(\pred_1))$.
\end{observation}

\begin{proof}
Expanding \Cref{eq:pointwise-bregman} inside the expectation gives
\begin{align*}
D_{\score}(f,g)
=\E_{\prob_{\predRV_1}}\big[\genFunc(f(\predRV_1))-\genFunc(g(\predRV_1)) -\langle\psi(g(\predRV_1)),f(\predRV_1)-g(\predRV_1)\rangle\big],
\end{align*}
which is exactly the Bregman divergence of $\mathcal G$. Positive support weights preserve strict convexity.
\end{proof}

\begin{lemma}[Generalized Pythagorean inequality \citealt{rockafellar1997convex}]
\label{lem:pythagorean}
Let Bregman divergence $D_{\score}$ be induced by a strictly convex, differentiable generator $G$. Let $\mathcal C$ be a nonempty closed convex set, and suppose the Bregman projection of a point $v_0$ onto $\mathcal{C}$
\[
v^*\in\arg\min_{v\in\mathcal C}D_{\score}(v,v_0)
\]
exists and $G$ has finite gradients at $v_0$ and $v^*$. Then, for every $v\in\mathcal C$ and any point $v_0$,
\[
D_{\score}(v,v_0)\geq D_{\score}(v,v^*)+D_{\score}(v^*,v_0).
\]
In particular, existence is automatic when $\mathcal C$ is compact and $D_{\score}(\cdot,v_0)$ is continuous.
\end{lemma}

\begin{proof}
The first-order variational inequality for the convex projection problem gives
\[
\langle\nabla G(v^*)-\nabla G(v_0),v-v^*\rangle\geq0
\qquad\text{for all }v\in\mathcal C.
\]
The Bregman three-point identity gives
\[
D_G(v,v_0)-D_G(v,v^*)-D_G(v^*,v_0)
=\langle\nabla G(v^*)-\nabla G(v_0),v-v^*\rangle,
\]
which proves the claim.
\end{proof}

\subsection{Proof of \texorpdfstring{\Cref{obs:compatible-geometry}}{Observation \ref*{obs:compatible-geometry}}}

\Geometry*

\begin{proof}
Write
\[
p_{01}(q_0,q_1)
:=\prob_{\predRV_0,\predRV_1}(q_0,q_1),
\qquad
p_i(q_i):=\prob_{\predRV_i}(q_i),
\quad i\in\{0,1\}.
\]
Because $\predsp_i$ is the support of $\predRV_i$, $p_i(q_i)>0$ for every
$q_i\in\predsp_i$.
Let $\mathcal Z$ be the set of nonnegative arrays
$z=(z(y,q_0,q_1))_{y,q_0,q_1}$ satisfying
\begin{align*}
\sum_y z(y,q_0,q_1)
&=p_{01}(q_0,q_1)
&&\text{for every }q_0\in\predsp_0,\ q_1\in\predsp_1,\\
\sum_{q_1}z(y,q_0,q_1)
&=q_0(y)p_0(q_0)
&&\text{for every }q_0\in\predsp_0,\ y\in\statesp.
\end{align*}
These are precisely the fixed-report-marginal and reference-calibration
constraints in the first two lines of \Cref{eq:latent-mass-constraints}.
For each $z\in\mathcal Z$, define, on the positive-probability support
$\predsp_1$,
\begin{equation}
\label{eq:hz-from-latent-mass}
h_z(q_1)(y)
:=\frac{\sum_{q_0}z(y,q_0,q_1)}{p_1(q_1)}.
\end{equation}
The denominator is positive, and the first constraint gives
\[
\sum_y h_z(q_1)(y)
=\frac{\sum_{q_0,y}z(y,q_0,q_1)}{p_1(q_1)}
=\frac{\sum_{q_0}p_{01}(q_0,q_1)}{p_1(q_1)}=1.
\]
Thus $h_z(q_1)\in\Delta(\statesp)$.

Every $z\in\mathcal Z$ defines an extension by setting
$\widetilde\prob_z(q_0,q_1,y):=z(y,q_0,q_1)$. Its report marginal is
$\prob_{\predRV_0,\predRV_1}$, and the second constraint together with
\Cref{eq:hz-from-latent-mass} gives
\[
\Pr_{\widetilde\prob_z}(\staterv=y\mid\predRV_0=q_0)=q_0(y),
\qquad
\Pr_{\widetilde\prob_z}(\staterv=y\mid\predRV_1=q_1)=h_z(q_1)(y).
\]
Hence the extension calibrates the reference and induces $h_z$ as its
primary-bin posterior map. Conversely, the joint mass function of every
extension appearing in
\Cref{eq:compatible-posterior-set} belongs to $\mathcal Z$ and induces its map
through \Cref{eq:hz-from-latent-mass}. Hence
\[
\calmapsp_{\mathrm{comp}}(\prob_{\predRV_0,\predRV_1})
=\{h_z:z\in\mathcal Z\}.
\]

The set $\mathcal Z$ is nonempty: one feasible array is
\[
z^0(y,q_0,q_1):=p_{01}(q_0,q_1)q_0(y),
\]
which corresponds to drawing $\staterv\sim q_0$ conditional on
$(\predRV_0,\predRV_1)=(q_0,q_1)$. The set is convex and closed, and
nonnegativity together with the first constraint implies
$0\leq z(y,q_0,q_1)\leq p_{01}(q_0,q_1)$. Thus $\mathcal Z$ is bounded and,
in this finite-dimensional space, compact. The map
$z\mapsto h_z$ is linear and continuous because the denominators $p_1(q_1)$
are fixed and positive. Its image
$\calmapsp_{\mathrm{comp}}(\prob_{\predRV_0,\predRV_1})$ is consequently
nonempty, compact, and convex.

It remains to identify when the identity belongs to this set. If
$\operatorname{id}\in\calmapsp_{\mathrm{comp}}$, its witnessing extension
satisfies
\[
\E[\statervvec\mid\predRV_0]=\predRV_0,
\qquad
\E[\statervvec\mid\predRV_1]
=\operatorname{id}(\predRV_1)=\predRV_1.
\]
Thus both predictors are calibrated under the same extension, so they are
mutually calibrated. Conversely, any extension witnessing mutual calibration
has primary-bin posterior map
$q_1\mapsto\E[\statervvec\mid\predRV_1=q_1]=q_1$, which is the identity and
therefore belongs to the Reference-Compatible Set.
\end{proof}

\subsection{Proof of Theorem~\ref{thm:main_result}}

\mainResult*

\begin{proof}
Fix an arbitrary admissible extension $\widetilde\prob_{\predRV_0,\predRV_1,\staterv}$ and define its posterior map on the original primary-report bins by
\[
\eta_{\widetilde\prob}(\pred_1)
=\E_{\widetilde\prob}[\statervvec\mid\predRV_1=\pred_1].
\]
The extension $\widetilde\prob$ itself witnesses
$\eta_{\widetilde\prob}\in\calmapsp_{\mathrm{comp}}(\prob_{\predRV_0,\predRV_1})$ under \Cref{eq:compatible-posterior-set}.

For every primary report
$\pred_1$ with positive probability, the conditional distribution of
$\staterv$ given $\predRV_1=\pred_1$ under the extension $\widetilde\prob$ is $\eta_{\widetilde\prob}(\pred_1)$. Hence,
\begin{equation}
\label{eq:conditional-risk-difference}
    \begin{aligned}
    \E_{\widetilde\prob}\!\left[
    \score(\pred_1,\staterv)
    -\score(\calmap^*(\pred_1),\staterv)
    \,\middle|\,\predRV_1=\pred_1
    \right]
    &=
    \E_{\staterv\sim\eta_{\widetilde\prob}(\pred_1)}\!\left[
    \score(\pred_1,\staterv)
    -\score(\calmap^*(\pred_1),\staterv)
    \right] \\
    &=
    \E_{\staterv\sim\eta_{\widetilde\prob}(\pred_1)}\!\left[
    \score(\pred_1,\staterv)-\score(\eta_{\widetilde\prob}(\pred_1),\staterv)
    \right] \\
    &\quad -
    \E_{\staterv\sim\eta_{\widetilde\prob}(\pred_1)}\!\left[
    \score(\calmap^*(\pred_1),\staterv)
    -\score(\eta_{\widetilde\prob}(\pred_1),\staterv)
    \right] \\
    &=
    D_{\score}\big(\eta_{\widetilde\prob}(\pred_1),\pred_1\big)
    -D_{\score}\big(\eta_{\widetilde\prob}(\pred_1),\calmap^*(\pred_1)\big),
    \end{aligned}
\end{equation}
where the last equality follows from
\Cref{eq:pointwise-bregman}. Applying the law of total expectation therefore gives
\begin{equation}
    \begin{aligned}
        \E_{\widetilde\prob}[\score(\predRV_1,\staterv)]
        -\E_{\widetilde\prob}
         [\score(\calmap^*(\predRV_1),\staterv)]
        &=
        \E_{\prob_{\predRV_1}}\!\left[
        D_{\score}\big(\eta_{\widetilde\prob}(\predRV_1),\predRV_1\big)
        -D_{\score}\big(
        \eta_{\widetilde\prob}(\predRV_1),\calmap^*(\predRV_1)
        \big)
        \right] \\
        &=
        \E_{\prob_{\predRV_1}}\!\left[
        D_{\score}\big(
        \eta_{\widetilde\prob}(\predRV_1),\operatorname{id}(\predRV_1)
        \big)
        \right] -
        \E_{\prob_{\predRV_1}}\!\left[
        D_{\score}\big(
        \eta_{\widetilde\prob}(\predRV_1),\calmap^*(\predRV_1)
        \big)
        \right] \\
        &=
        D_{\score}(\eta_{\widetilde\prob},\operatorname{id})
        -D_{\score}(\eta_{\widetilde\prob},\calmap^*).
    \end{aligned}
\label{eq:main-risk-difference}
\end{equation}

By \Cref{obs:compatible-geometry}, the compatible posterior set is nonempty, compact, and convex. Hence the projection $\calmap^*$ exists. Applying \Cref{lem:pythagorean} to the comparison point $\eta_{\widetilde\prob}$ gives
\[
D_{\score}(\eta_{\widetilde\prob},\operatorname{id})
\geq D_{\score}(\eta_{\widetilde\prob},\calmap^*)
+D_{\score}(\calmap^*,\operatorname{id}).
\]
Substitution into \Cref{eq:main-risk-difference} proves the uniform lower bound \Cref{eq:uniform-main-guarantee}.

We next prove strictness and necessity.

\begin{itemize}
    \item If the predictors are not mutually calibrated, \Cref{obs:compatible-geometry} implies that $\operatorname{id}\notin\calmapsp_{\mathrm{comp}}(\prob_{\predRV_0,\predRV_1})$. The attained projection $\calmap^*$ therefore differs from the identity $\operatorname{id}$ on a positive-mass report, and strictness yields $\delta=D_{\score}(\calmap^*,\operatorname{id})>0$. Thus $\calmap^*$ strictly improves the primary predictor $\predRV_1$ under every possible ground truth extension.

    \item Conversely, suppose the predictors are mutually calibrated. Then there
    exists an admissible extension $\widetilde\prob^0$ under which both predictors
    are calibrated. In particular,
    \[
    \E_{\widetilde\prob^0}
    [\statervvec\mid\predRV_1]
    =\predRV_1.
    \]
    Therefore, conditional on each primary report, $\predRV_1$ is the true
    conditional distribution of the outcome. For any deterministic transformation
    $\calmap:\predsp_1\to\Delta(\statesp)$, strictness gives
    \[
    \E_{\widetilde\prob^0}
    \!\left[
        \score(\calmap(\predRV_1),\staterv)
        \,\middle|\,\predRV_1
    \right]
    \geq
    \E_{\widetilde\prob^0}
    \!\left[
        \score(\predRV_1,\staterv)
        \,\middle|\,\predRV_1
    \right].
    \]
    Taking expectations over $\predRV_1$ yields
    \[
    \E_{\widetilde\prob^0}
    [\score(\calmap(\predRV_1),\staterv)]
    \geq
    \E_{\widetilde\prob^0}
    [\score(\predRV_1,\staterv)],
    \]
    with equality only if
    $\calmap(\predRV_1)=\predRV_1$ almost surely.

    On the other hand, any transformation that weakly improves the primary
    predictor under every admissible extension must satisfy the reverse inequality
    under $\widetilde\prob^0$. Hence equality is necessary, and strict propriety
    implies that such a transformation must coincide with the identity on the
    support of $\prob_{\predRV_1}$. The identity cannot yield a strict improvement
    under any extension. Consequently, no strictly improving deterministic
    transformation exists when the predictors are mutually calibrated.
\end{itemize}

It remains to establish exact worst-case equality. Since
$\calmap^*\in
\calmapsp_{\mathrm{comp}}(\prob_{\predRV_0,\predRV_1})$,
the definition of the reference-compatible set guarantees the existence of an
admissible extension $\widetilde\prob^*$ satisfying
\[
\E_{\widetilde\prob^*}
[\statervvec\mid\predRV_1]
=
\calmap^*(\predRV_1).
\]
Thus, conditional on $\predRV_1$, the true outcome distribution under
$\widetilde\prob^*$ is $\calmap^*(\predRV_1)$. The pointwise regret definition therefore gives
\[
\E_{\widetilde\prob^*}
\!\left[
    \score(\predRV_1,\staterv)
    -
    \score(\calmap^*(\predRV_1),\staterv)
    \,\middle|\,\predRV_1
\right]
=
D_{\score}
\bigl(\calmap^*(\predRV_1),\predRV_1\bigr).
\]
Taking expectations over $\predRV_1$ and using the fact that every admissible
extension has the fixed marginal $\prob_{\predRV_1}$ yields
\begin{align*}
\E_{\widetilde\prob^*}
[\score(\predRV_1,\staterv)]
-
\E_{\widetilde\prob^*}
[\score(\calmap^*(\predRV_1),\staterv)]=
\E_{\prob_{\predRV_1}}
\!\left[
    D_{\score}
    \bigl(\calmap^*(\predRV_1),\predRV_1\bigr)
\right]
=
\delta,
\end{align*}
where the final equality follows from the definitions of $\calmap^*$ and
$\delta$ in \Cref{eq:bregproj,eq:uniform-main-guarantee}.

We have already shown that the loss improvement is at least $\delta$ under
every admissible extension, while the admissible extension
$\widetilde\prob^*$ constructed above attains equality. Hence $\delta$ is the
exact worst-case loss improvement.
\end{proof}

\section{Performance Guarantees under an Imperfectly Calibrated Reference}
\label{sec:imperfect calibration reference}

We now allow the actual reference to be imperfectly calibrated. To distinguish
the probability laws used below, write
\[
\prob^{\mathrm{obs}}_{\predRV_0,\predRV_1}
:=\prob_{\predRV_0,\predRV_1}
\]
for the observed joint report law. The reference-compatible set in
\Cref{eq:compatible-posterior-set} is constructed from this law and generic
reference-calibrated extensions
$\widetilde\prob_{\predRV_0,\predRV_1,\staterv}$. In contrast,
let $\prob^{\mathrm{act}}_{\predRV_0,\predRV_1,\staterv}$ denote the actual
joint law, whose $(\predRV_0,\predRV_1)$-marginal is
$\prob^{\mathrm{obs}}_{\predRV_0,\predRV_1}$ but under which the reference need
not be calibrated. The projection $\calmap^*$ is computed only from
$\prob^{\mathrm{obs}}_{\predRV_0,\predRV_1}$.

Identify each outcome with its one-hot vector $\statervvec$. For $i\in\{0,1\}$, define the
actual posterior map
\[
\eta_i^{\mathrm{act}}(\pred_i)
:=\E_{\prob^{\mathrm{act}}}
[\statervvec\mid\predRV_i=\pred_i],
\]
and define the actual reference calibration error by
\[
\epsilon_{\mathrm{ref}}
:=\E_{\prob^{\mathrm{obs}}_{\predRV_0}}
\!\left[\left\|
\eta_0^{\mathrm{act}}(\predRV_0)-\predRV_0
\right\|_1\right].
\]

\begin{restatable}{theorem}{mainResultRelaxedGeneral}
\label{thm:main_result_relaxed_general}
Assume finite outcome and report supports and a strictly proper loss $\score$ generated by a differentiable, strictly convex function $\genFunc$ with finite gradients on the prediction domain. Let $\calmap^*$ be the projection in \Cref{eq:bregproj} computed from the observed $\prob^{\mathrm{obs}}_{\predRV_0,\predRV_1}$, and $\prob^{\mathrm{act}}_{\predRV_0,\predRV_1,\staterv}$ be the actual joint distribution involving ground truth $\staterv$ with $(\predRV_0,\predRV_1)$-margin $\prob^{\mathrm{obs}}_{\predRV_0,\predRV_1}$. Define
\begin{align}
A&:=\E_{\prob^{\mathrm{obs}}_{\predRV_1}}
\big[D_{\score}(\calmap^*(\predRV_1),\predRV_1)\big], \\
B&:=\left(
\E_{\prob^{\mathrm{obs}}_{\predRV_1}}\left[
\left\|
\nabla\genFunc(\calmap^*(\predRV_1))
-\nabla\genFunc(\predRV_1)
\right\|_2^2
\right]\right)^{1/2}. \label{eq:imperfect-a-b}
\end{align}

Then the actual loss improvement satisfies
\begin{equation}
\label{eq:imperfect-main-bound}
\begin{aligned}
\Delta L^{\mathrm{act}}
&:=\E_{\prob^{\mathrm{act}}}[\score(\predRV_1,\staterv)]
-\E_{\prob^{\mathrm{act}}}
[\score(\calmap^*(\predRV_1),\staterv)]\\
&\geq A-B\sqrt{\epsilon_{\mathrm{ref}}}.
\end{aligned}
\end{equation}
Consequently, $\calmap^*$ strictly improves expected loss whenever
$A>B\sqrt{\epsilon_{\mathrm{ref}}}$.
\end{restatable}

\begin{proof}
The proof proceeds in three steps: derive an exact decomposition around an arbitrary auxiliary reference-compatible map, construct a nearby map by repairing the outcome, and specialize the decomposition to that repaired posterior map to control the residue.

\paragraph{Step 1: derive the exact auxiliary decomposition.}
Define the gradient-displacement field
\[
\xi^*(\pred_1)
:=\nabla\genFunc(\calmap^*(\pred_1))-\nabla\genFunc(\pred_1).
\]
Following \Cref{eq:conditional-risk-difference}, we can first express the actual loss improvement as
\begin{align}
\Delta L^{\mathrm{act}}
&=\E_{\prob^{\mathrm{obs}}_{\predRV_1}}\!\left[
D_{\score}(\eta_1^{\mathrm{act}}(\predRV_1),\predRV_1)
-D_{\score}(\eta_1^{\mathrm{act}}(\predRV_1),
\calmap^*(\predRV_1))
\right]. \label{eq:imperfect-regret-difference}
\end{align}

Choose any auxiliary map
\[
\calmap_{\mathrm{aux}}
\in\calmapsp_{\mathrm{comp}}
(\prob^{\mathrm{obs}}_{\predRV_0,\predRV_1}).
\]
Applying the Bregman three-point identity and splitting
\[
\eta_1^{\mathrm{act}}-\calmap^*
=
(\calmap_{\mathrm{aux}}-\calmap^*)
+
(\eta_1^{\mathrm{act}}-\calmap_{\mathrm{aux}})
\]
gives the following three-term decomposition:
\begin{align}
\Delta L^{\mathrm{act}}
&=\underbrace{A}_{\text{Geometric Gain (Always }\ge0\text{)}} \notag\\
&\quad+
\underbrace{\E_{\prob^{\mathrm{obs}}_{\predRV_1}}\!\left[
\left\langle \xi^*(\predRV_1),
\calmap_{\mathrm{aux}}(\predRV_1)-\calmap^*(\predRV_1)
\right\rangle\right]}_{\text{Compatibility Bonus (Always }\ge 0\text{)}} \notag\\
&\quad+
\underbrace{\E_{\prob^{\mathrm{obs}}_{\predRV_1}}\!\left[
\left\langle \xi^*(\predRV_1),
\eta_1^{\mathrm{act}}(\predRV_1)
-\calmap_{\mathrm{aux}}(\predRV_1)
\right\rangle\right]}_{\text{Miscalibration Risk (Can be }\le 0\text{)}}.
\label{eq:imperfect-three-term}
\end{align}
Indeed, since $\calmap^*$ minimizes the Bregman displacement over the convex reference-compatible set, its first-order variational inequality in the feasible direction $\calmap_{\mathrm{aux}}-\calmap^*$ is
\[
\E_{\prob^{\mathrm{obs}}_{\predRV_1}}\!\left[
\left\langle \xi^*(\predRV_1),
\calmap_{\mathrm{aux}}(\predRV_1)-\calmap^*(\predRV_1)
\right\rangle\right]\geq0.
\]
This proves the non-negativity of the compatibility bonus in \Cref{eq:imperfect-three-term}.

\paragraph{Step 2: construct a nearby reference-compatible posterior map.}
For each reference report $\pred_0$, choose a maximal coupling
$\pi_{\pred_0}$ whose first marginal is
$\eta_0^{\mathrm{act}}(\pred_0)$ and whose second marginal is $\pred_0$.
Thus
\[
\pi_{\pred_0}
\big(\{(\state,\state^{\mathrm{rep}}):
\state\neq\state^{\mathrm{rep}}\}\big)
=\frac12
\|\eta_0^{\mathrm{act}}(\pred_0)-\pred_0\|_1.
\]
Let $\kernel^{\mathrm{rep}}_{\pred_0}
(\state^{\mathrm{rep}}\mid\state)$ be a transition kernel realizing this
coupling. We make the construction explicit by defining the augmented coupling
law
\begin{align*}
&\prob^{\mathrm{cpl}}_{\predRV_0,\predRV_1,
\staterv,\staterv^{\mathrm{rep}}}
(\pred_0,\pred_1,\state,\state^{\mathrm{rep}})\\
&\qquad:=
\prob^{\mathrm{act}}_{\predRV_0,\predRV_1,\staterv}
(\pred_0,\pred_1,\state)
\kernel^{\mathrm{rep}}_{\pred_0}
(\state^{\mathrm{rep}}\mid\state),
\end{align*}
and its repaired marginal
\begin{align*}
&\prob^{\mathrm{rep}}_{\predRV_0,\predRV_1,
\staterv^{\mathrm{rep}}}
(\pred_0,\pred_1,\state^{\mathrm{rep}})\\
&\qquad:=\sum_{\state\in\statesp}
\prob^{\mathrm{cpl}}_{\predRV_0,\predRV_1,
\staterv,\staterv^{\mathrm{rep}}}
(\pred_0,\pred_1,\state,\state^{\mathrm{rep}}).
\end{align*}
The coupling law has
$\prob^{\mathrm{act}}_{\predRV_0,\predRV_1,\staterv}$ as its
$(\predRV_0,\predRV_1,\staterv)$-marginal, while the repaired law has
$\prob^{\mathrm{obs}}_{\predRV_0,\predRV_1}$ as its report marginal. Moreover,
conditional on $\predRV_0=\pred_0$, the
$(\staterv,\staterv^{\mathrm{rep}})$-marginal under $\prob^{\mathrm{cpl}}$ is
$\pi_{\pred_0}$. The coupling construction therefore gives
\[
\E_{\prob^{\mathrm{rep}}}
[\statervvec^{\mathrm{rep}}\mid\predRV_0]=\predRV_0.
\]
Define the posterior induced by the repaired law on each original primary-report
bin by
\[
\eta_1^{\mathrm{rep}}(\pred_1)
:=\E_{\prob^{\mathrm{rep}}}
[\statervvec^{\mathrm{rep}}\mid\predRV_1=\pred_1].
\]
When $\staterv^{\mathrm{rep}}$ is treated as the outcome coordinate, the law $\prob^{\mathrm{rep}}$ is a reference-calibrated extension of the observed report law and witnesses the compatibility of $\eta_1^{\mathrm{rep}}$. By \Cref{eq:compatible-posterior-set}, we have
\begin{equation}
\label{eq:repaired-compatible}
\eta_1^{\mathrm{rep}}
\in\calmapsp_{\mathrm{comp}}
(\prob^{\mathrm{obs}}_{\predRV_0,\predRV_1}).
\end{equation}

Because $\prob^{\mathrm{cpl}}$ has the actual and repaired laws as the
corresponding marginals, conditional Jensen's inequality gives
\begin{align}
&\E_{\prob^{\mathrm{obs}}_{\predRV_1}}
\|\eta_1^{\mathrm{act}}(\predRV_1)
-\eta_1^{\mathrm{rep}}(\predRV_1)\|_1 \notag\\
&=\E_{\prob^{\mathrm{obs}}_{\predRV_1}}
\left\|
\E_{\prob^{\mathrm{cpl}}}
[\statervvec-\statervvec^{\mathrm{rep}}\mid\predRV_1]
\right\|_1 \notag\\
&\leq\E_{\prob^{\mathrm{cpl}}}
\|\statervvec-\statervvec^{\mathrm{rep}}\|_1 \notag\\
&=2\E_{\prob^{\mathrm{obs}}_{\predRV_0}}
\left[
\Pr_{\prob^{\mathrm{cpl}}}
(\staterv\neq\staterv^{\mathrm{rep}}\mid\predRV_0)
\right] \notag\\
&=\E_{\prob^{\mathrm{obs}}_{\predRV_0}}
\|\eta_0^{\mathrm{act}}(\predRV_0)-\predRV_0\|_1
=\epsilon_{\mathrm{ref}}. \label{eq:repaired-l1}
\end{align}
For probability vectors $\truepred,\pred$,
$\|\truepred-\pred\|_2^2\leq\|\truepred-\pred\|_1$. Therefore,
\begin{equation}
\label{eq:repaired-l2}
\E_{\prob^{\mathrm{obs}}_{\predRV_1}}
\|\eta_1^{\mathrm{act}}(\predRV_1)
-\eta_1^{\mathrm{rep}}(\predRV_1)\|_2^2
\leq\epsilon_{\mathrm{ref}}.
\end{equation}

\paragraph{Step 3: specialize the auxiliary map and control the residual.}
By \Cref{eq:repaired-compatible}, the repaired posterior from Step 2 is a valid
choice $\calmap_{\mathrm{aux}}=\eta_1^{\mathrm{rep}}$ in
\Cref{eq:imperfect-three-term}. Cauchy--Schwarz in the space of transformations
weighted by $\prob^{\mathrm{obs}}_{\predRV_1}$ and \Cref{eq:repaired-l2} give
\begin{align}
&\left|
\E_{\prob^{\mathrm{obs}}_{\predRV_1}}\!\left[
\left\langle \xi^*(\predRV_1),
\eta_1^{\mathrm{act}}(\predRV_1)
-\eta_1^{\mathrm{rep}}(\predRV_1)
\right\rangle\right]
\right| \notag\\
&\quad\leq
\left(
\E_{\prob^{\mathrm{obs}}_{\predRV_1}}
\|\xi^*(\predRV_1)\|_2^2
\right)^{1/2}
\left(
\E_{\prob^{\mathrm{obs}}_{\predRV_1}}
\|\eta_1^{\mathrm{act}}(\predRV_1)
-\eta_1^{\mathrm{rep}}(\predRV_1)\|_2^2
\right)^{1/2} \notag\\
&\quad\leq B\sqrt{\epsilon_{\mathrm{ref}}}.
\label{eq:imperfect-residual-bound}
\end{align}
The compatible bonus term in \Cref{eq:imperfect-three-term} is nonnegative by Step 1, whereas its signed residual (miscalibration risk) is bounded below by the negative of the left-hand side of \Cref{eq:imperfect-residual-bound}. Therefore,
\begin{align*}
\Delta L^{\mathrm{act}}
&\geq A-
\left|
\E_{\prob^{\mathrm{obs}}_{\predRV_1}}\!\left[
\left\langle \xi^*(\predRV_1),
\eta_1^{\mathrm{act}}(\predRV_1)
-\eta_1^{\mathrm{rep}}(\predRV_1)
\right\rangle\right]
\right|\\
&\geq A-B\sqrt{\epsilon_{\mathrm{ref}}},
\end{align*}
which proves \Cref{eq:imperfect-main-bound}.
\end{proof}

\begin{corollary}[Strongly convex and smooth generator]
\label{cor:strong_convex_risk_bound}
In the setting of \Cref{thm:main_result_relaxed_general}, suppose additionally
that $\genFunc$ is $\mu_G$-strongly convex and $L_G$-smooth on the prediction
domain. Then
\[
B\leq L_G\sqrt{\frac{2A}{\mu_G}},
\]
and hence
\begin{equation}
\label{eq:imperfect-smooth-bound}
\Delta L^{\mathrm{act}}
\geq A-L_G\sqrt{\frac{2A\epsilon_{\mathrm{ref}}}{\mu_G}}.
\end{equation}
In particular, the right-hand side is positive whenever
$A>2L_G^2\epsilon_{\mathrm{ref}}/\mu_G$.
\end{corollary}

\begin{proof}
Smoothness and strong convexity imply, pointwise,
\[
\begin{aligned}
\|\nabla\genFunc(\calmap^*(\predRV_1))
-\nabla\genFunc(\predRV_1)\|_2
&\leq L_G\|\calmap^*(\predRV_1)-\predRV_1\|_2,\\
D_{\score}(\calmap^*(\predRV_1),\predRV_1)
&\geq\frac{\mu_G}{2}\|\calmap^*(\predRV_1)-\predRV_1\|_2^2.
\end{aligned}
\]
Square the first inequality, average under
$\prob^{\mathrm{obs}}_{\predRV_1}$, and apply the second to obtain the bound on
$B$. The remaining claims follow from \Cref{eq:imperfect-main-bound}.
\end{proof}

\section{Extension to Elicitable Properties}
\label{app:elicitable}

\subsection{Setup and Definitions}
\label{app:elicitable setup}

Models often report summary statistics, such as an answer and its confidence, rather than full predictive distributions. We therefore extend our framework to elicitable properties \cite{lambert2008eliciting}. We show that in this setting, we can still obtain an optimal worst-case improvement guarantee under the original property loss by combining a convex projection over conditional distributions with a suitably chosen, possibly randomized report-selection rule. The guarantee applies to both differentiable and nondifferentiable Bayes risks.

Let $\statesp$ be a finite nonempty outcome space. For $i\in\{0,1\}$, let $\cA_i$
be the allowable report space and let
$\prop_i:\Delta(\statesp)\to 2^{\cA_i}$ be a nonempty-valued property.
The observed predictor $\repRV_i$ has finite support
$\repsp_i:=\operatorname{supp}(\repRV_i)\subseteq\cA_i$.
We fix the observed joint distribution $\prob_{\repRV_0,\repRV_1}$ throughout
and write $\mu_{\rep}:=\Pr(\repRV_1=\rep)>0$ for each
$\rep\in\repsp_1$.
We equip the primary report space $\cA_1$ with a Hausdorff topology.
Probability simplices and spaces of posterior maps carry their usual
finite-dimensional Euclidean topologies. These topologies allow us to state the compactness and continuity assumptions used to establish the existence of optimal projections and report-selection rules. The level set of a report $\rep$ is
\[
\prop_i^{-1}(\rep)
:=\{\pred\in\Delta(\statesp):\rep\in\prop_i(\pred)\}.
\]
Under a joint distribution $\widetilde\prob$, predictor $\repRV_i$ is
calibrated with respect to $\prop_i$ if
\[
\rep_i\in\prop_i\!\left(
\Pr_{\widetilde\prob}(\staterv=\cdot\mid\repRV_i=\rep_i)
\right)
\qquad\text{for every }\rep_i\in\repsp_i.
\]
All conditional constraints below refer to original positive-probability
report bins.

\begin{definition}[Mutual calibration for elicitable properties]
Reports $\repRV_0,\repRV_1$ are mutually calibrated with respect to
$\prop_0,\prop_1$ if there exists an extension
$\widetilde\prob_{\repRV_0,\repRV_1,\staterv}$ of
$\prob_{\repRV_0,\repRV_1}$ that calibrates both predictors with respect
to their properties. Such an extension is a witness to mutual calibration.
\end{definition}

\paragraph{Loss and property regret.}
Let $\score_i:\cA_i\times\statesp\to\R$ be finite-valued and strictly
proper for $\prop_i$, meaning that
\[
\operatorname*{argmin}_{\rep'\in\cA_i}
\E_{\staterv\sim\pred}[\score_i(\rep',\staterv)]
=\prop_i(\pred)
\qquad\text{for every }\pred\in\Delta(\statesp).
\]
Write
$\mathbf L_{i,\rep}:=(\score_i(\rep,\state))_{\state\in\statesp}$
for the loss vector of a report, and abbreviate
$\mathbf L_{\rep}:=\mathbf L_{1,\rep}$.
The primary loss is continuous in the report if, for each fixed outcome $\state$, the map $\rep\mapsto\score_1(\rep,\state)$ is continuous on $\cA_1$: for every open set $U\subseteq\R$, its preimage $\{\rep\in\cA_1:\score_1(\rep,\state)\in U\}$ is open in the chosen topology on $\cA_1$. Since $\statesp$ is finite, this coordinatewise continuity is equivalent to continuity of the loss-vector map $\rep\mapsto\mathbf L_{\rep}$ into $\R^{|\statesp|}$ equipped with its Euclidean topology.
We impose this continuity condition, together with compactness of
$\cA_1$, in \Cref{thm:elicitable}.
The primary loss has Bayes risk
\[
\underline L_1(\pred)
:=\min_{\rep'\in\cA_1}
\langle\mathbf L_{\rep'},\pred\rangle
\]
and property regret
\begin{equation}
\label{eq:property-regret}
\begin{aligned}
\rho_{\score_1}(\pred,\rep)
&:=\langle\mathbf L_{\rep},\pred\rangle-\underline L_1(\pred)\\
&=\max_{\rep'\in\cA_1}
\langle\mathbf L_{\rep}-\mathbf L_{\rep'},\pred\rangle.
\end{aligned}
\end{equation}
The Bayes risk is a pointwise infimum of linear functions of $\pred$ and is therefore concave. Property regret is convex because it subtracts this concave Bayes risk from a report’s linear expected loss, and nonnegative because the Bayes risk is the minimum expected loss over all reports. Strict properness gives
\[
\rho_{\score_1}(\pred,\rep)=0
\quad\Longleftrightarrow\quad
\rep\in\prop_1(\pred).
\]

\paragraph{Report-selection rules and gain.}
Let $\mathfrak D_{\prop_0}$ be the extensions of the observed joint
distribution under which $\repRV_0$ is calibrated with respect to $\prop_0$.
A report-selection rule $\pi=(\pi_{\rep})_{\rep\in\repsp_1}$ assigns a
finitely supported probability distribution on $\cA_1$ to each primary
report. On input $\repRV_1=\rep$, the rule outputs
$\widehat R\sim\pi_{\rep}$, using randomness independent of
$(\repRV_0,\staterv)$ conditional on $\repRV_1$.
Deterministic rules correspond to point masses.
Let $\Pi$ denote the collection of these rules and define
\begin{equation}
\label{eq:property-report-gain}
\operatorname{Gain}(\pi;\widetilde\prob)
:=\E_{\widetilde\prob,\pi}
\left[\score_1(\repRV_1,\staterv)-\score_1(\widehat R,\staterv)\right].
\end{equation}
This comparison uses the original property loss for both reports.
We call a rule uniformly strictly improving if there exists a constant
$\gamma>0$ such that
$\operatorname{Gain}(\pi;\widetilde\prob)\geq\gamma$
for every $\widetilde\prob\in\mathfrak D_{\prop_0}$.

\subsection{Generalized Reference-Compatible Set and Algorithm}
\label{app:elicitable generalized algorithm}

For $\widetilde\prob\in\mathfrak D_{\prop_0}$, define its primary
posterior map by
\[
\eta_{\widetilde\prob}(\rep)
:=\Pr_{\widetilde\prob}(\staterv=\cdot\mid\repRV_1=\rep).
\]
The reference-compatible set consists of all such maps,
\begin{equation}
\label{eq:property-compatible-set}
\mathcal C_{\prop}
:=\calmapsp_{\mathrm{comp}}^{\prop_0}(\prob_{\repRV_0,\repRV_1})
=\{\eta_{\widetilde\prob}:\widetilde\prob\in\mathfrak D_{\prop_0}\}.
\end{equation}
Thus membership of $\calmap$ requires a single extension that calibrates
the reference and satisfies
$\calmap(\rep)=\Pr_{\widetilde\prob}(\staterv=\cdot\mid\repRV_1=\rep)$
on each original primary-report bin.

\begin{observation}
\label{obs:convex}
Suppose the outcome and observed-report supports are finite, the losses
$\score_i$ are finite-valued and strictly proper for $\prop_i$, and
$\prop_i^{-1}(\rep)\neq\varnothing$ for every
$\rep\in\repsp_i$ and $i\in\{0,1\}$.
Every observed level set $\prop_i^{-1}(\rep_i)$ is nonempty, closed,
and convex. Consequently, $\mathcal C_{\prop}$ is nonempty, compact,
and convex.
\end{observation}

\begin{proof}
Strict properness gives
\[
\prop_i^{-1}(\rep_i)
=\bigcap_{\rep'\in\cA_i}
\left\{\pred\in\Delta(\statesp):
\langle\mathbf L_{i,\rep_i}-\mathbf L_{i,\rep'},\pred\rangle\leq0
\right\}.
\]
For each fixed pair of reports, the function
$\pred\mapsto\langle\mathbf L_{i,\rep_i}-\mathbf L_{i,\rep'},\pred\rangle$
is linear and continuous in the probability vector $\pred$.
The intersection is therefore closed and convex relative to the simplex,
and it is nonempty by assumption.

An extension of the fixed report marginal is specified by a kernel
$\kernel(\state\mid\rep_0,\rep_1)$.
Choose $\pred_{\rep_0}\in\prop_0^{-1}(\rep_0)$ for each reference
report and set $\kernel(\cdot\mid\rep_0,\rep_1)=\pred_{\rep_0}$.
This constructs a reference-calibrated extension and hence a point in
$\mathcal C_{\prop}$.
The reference-calibrated kernels form a closed convex subset of a finite
product of simplices. Mixing kernels preserves the report marginal,
and all conditional distributions are linear in the kernel because the
bin probabilities are fixed and positive. The map from a kernel to its
primary posterior map is therefore linear and continuous.
Its image $\mathcal C_{\prop}$ is compact and convex.
This argument uses finite loss vectors and the linear dependence on
conditional distributions, so it applies without assuming continuity
of the losses in their report arguments.
\end{proof}

We first project onto $\mathcal C_{\prop}$ by minimizing expected
property regret,
\begin{equation}
\label{eq:property-projection}
\begin{aligned}
\Phi(\calmap)
&:=\E_{\prob_{\repRV_1}}
\left[\rho_{\score_1}(\calmap(\repRV_1),\repRV_1)\right],\\
\calmap^*&\in\operatorname*{argmin}_{\calmap\in\mathcal C_{\prop}}
\Phi(\calmap),
\qquad
\delta_{\prop}:=\Phi(\calmap^*).
\end{aligned}
\end{equation}
We then select output reports from $\prop_1(\calmap^*(\rep))$.
The theorem below establishes a selection rule with an exact uniform
improvement guarantee. The rule may depend on the original report bin
and the entire reference-compatible set.

\subsection{Exact Improvement Guarantee}

\begin{theorem}[Optimal improvement for elicitable properties]
\label{thm:elicitable}
Assume the hypotheses of \Cref{obs:convex}. Suppose additionally that
$\cA_1$ is compact in its specified topology and, for every fixed
$\state\in\statesp$, the function
$\rep\mapsto\score_1(\rep,\state)$ is continuous on $\cA_1$.
The solution set in \Cref{eq:property-projection} is nonempty.
For every solution $\calmap^*$, there exists a rule $\pi^*\in\Pi$ such that
\begin{equation}
\label{eq:property-optimal-support}
\operatorname{supp}(\pi^*_{\rep})
\subseteq\prop_1(\calmap^*(\rep))
\qquad\text{for every }\rep\in\repsp_1
\end{equation}
and, for every $\widetilde\prob\in\mathfrak D_{\prop_0}$,
\begin{equation}
\label{eq:property-uniform-gain}
\operatorname{Gain}(\pi^*;\widetilde\prob)\geq\delta_{\prop}.
\end{equation}
The bound is attained by an admissible extension and is optimal among
all rules in $\Pi$,
\begin{equation}
\label{eq:property-worst-case}
\max_{\pi\in\Pi}\min_{\widetilde\prob\in\mathfrak D_{\prop_0}}
\operatorname{Gain}(\pi;\widetilde\prob)
=\min_{\widetilde\prob\in\mathfrak D_{\prop_0}}
\operatorname{Gain}(\pi^*;\widetilde\prob)
=\delta_{\prop}.
\end{equation}
Moreover, $\delta_{\prop}>0$ if and only if $\repRV_0,\repRV_1$ are
not mutually calibrated with respect to $\prop_0,\prop_1$.
Thus a uniformly strictly improving rule in $\Pi$ exists exactly when
mutual calibration fails.
\end{theorem}

\begin{proof}
Define the set of report loss vectors and its convex hull by
\[
\mathcal S:=\{\mathbf L_{\rep'}:\rep'\in\cA_1\},
\qquad
\mathcal K:=\operatorname{conv}(\mathcal S).
\]
Each of the finitely many coordinates of the loss-vector map is
continuous by assumption. Thus $\mathcal S$ is the continuous image of
the compact space $\cA_1$ and is compact in $\R^{|\statesp|}$.
The convex hull of a compact set in a finite-dimensional Euclidean
space is compact, so $\mathcal K$ is compact and convex.
Every vector in $\mathcal K$ is the expected loss vector of a finitely
supported distribution on reports. A rule $\pi\in\Pi$ is therefore
represented, for purposes of expected loss, by
$\mathbf z=(\mathbf z_{\rep})_{\rep\in\repsp_1}\in\mathcal K^{\repsp_1}$,
where $\mathbf z_{\rep}:=\E_{\rep'\sim\pi_{\rep}}[\mathbf L_{\rep'}]$.
Conversely, every such $\mathbf z$ is realized by a rule in $\Pi$.

For a feasible posterior map $\calmap$ and expected loss vectors
$\mathbf z$, define
\begin{equation}
\label{eq:property-game}
F(\calmap,\mathbf z)
:=\sum_{\rep\in\repsp_1}\mu_{\rep}
\langle\mathbf L_{\rep}-\mathbf z_{\rep},\calmap(\rep)\rangle.
\end{equation}
Conditional expectation gives
$\operatorname{Gain}(\pi;\widetilde\prob)
=F(\eta_{\widetilde\prob},\mathbf z)$.
For every fixed $\calmap$,
\begin{equation}
\label{eq:property-game-value}
\begin{aligned}
\max_{\mathbf z\in\mathcal K^{\repsp_1}}F(\calmap,\mathbf z)
&=\sum_{\rep\in\repsp_1}\mu_{\rep}
\left[\langle\mathbf L_{\rep},\calmap(\rep)\rangle
-\min_{\mathbf z_{\rep}\in\mathcal K}
\langle\mathbf z_{\rep},\calmap(\rep)\rangle\right]\\
&=\sum_{\rep\in\repsp_1}\mu_{\rep}
\left[\langle\mathbf L_{\rep},\calmap(\rep)\rangle
-\underline L_1(\calmap(\rep))\right]
=\Phi(\calmap).
\end{aligned}
\end{equation}
To establish existence of the projection, let
$B:=\max_{\mathbf u\in\mathcal S}\|\mathbf u\|_2<\infty$.
For any distributions $\pred,\truepred$,
\[
\begin{aligned}
\bigl|\underline L_1(\pred)-\underline L_1(\truepred)\bigr|
&\leq\max_{\mathbf u\in\mathcal S}
\bigl|\langle\mathbf u,\pred-\truepred\rangle\bigr|\\
&\leq B\|\pred-\truepred\|_2.
\end{aligned}
\]
The Bayes risk is therefore Lipschitz continuous, and $\Phi$ is
continuous. By \Cref{obs:convex}, $\Phi$ attains its minimum on
$\mathcal C_{\prop}$.

The two strategy spaces $\mathcal C_{\prop}$ and
$\mathcal K^{\repsp_1}$ are compact and convex. The payoff $F$ is a
finite sum of inner products, so it is continuous and affine in each
argument. The minimax theorem \citep{sion1958general} yields
\begin{equation}
\label{eq:property-minimax}
\begin{aligned}
\max_{\mathbf z\in\mathcal K^{\repsp_1}}
\min_{\calmap\in\mathcal C_{\prop}}F(\calmap,\mathbf z)
&=\min_{\calmap\in\mathcal C_{\prop}}
\max_{\mathbf z\in\mathcal K^{\repsp_1}}F(\calmap,\mathbf z)\\
&=\min_{\calmap\in\mathcal C_{\prop}}\Phi(\calmap)
=\delta_{\prop}.
\end{aligned}
\end{equation}
The minimum over $\calmap$ is a continuous function of $\mathbf z$
because $F$ is continuous on the compact product of the strategy spaces.
It therefore attains its maximum on $\mathcal K^{\repsp_1}$.
Choose an optimal $\mathbf z^*$ on the left. Then
$F(\calmap,\mathbf z^*)\geq\delta_{\prop}$ for every
$\calmap\in\mathcal C_{\prop}$.

Fix any minimizer $\calmap^*$ of $\Phi$. Since
\[
\delta_{\prop}\leq F(\calmap^*,\mathbf z^*)
\leq\max_{\mathbf z}F(\calmap^*,\mathbf z)
=\Phi(\calmap^*)=\delta_{\prop},
\]
we have
\[
0=\sum_{\rep\in\repsp_1}\mu_{\rep}
\left[\langle\mathbf z^*_{\rep},\calmap^*(\rep)\rangle
-\underline L_1(\calmap^*(\rep))\right].
\]
Each summand is nonnegative, so each bracket is zero.
Represent $\mathbf z^*_{\rep}$ as a finite convex combination of report
loss vectors. Every report with positive weight must attain the Bayes
risk at $\calmap^*(\rep)$, since the average attains that risk and each
individual risk is at least as large. Strict properness therefore places
every such report in $\prop_1(\calmap^*(\rep))$.
These combinations define a rule $\pi^*$ satisfying
\Cref{eq:property-optimal-support,eq:property-uniform-gain}.

An extension $\widetilde\prob^*$ witnessing $\calmap^*$ satisfies
$\eta_{\widetilde\prob^*}=\calmap^*$, hence
$\operatorname{Gain}(\pi^*;\widetilde\prob^*)=\delta_{\prop}$.
For any other rule $\pi$, its conditional expected loss under this
extension is at least the Bayes risk in each bin. Consequently,
\[
\operatorname{Gain}(\pi;\widetilde\prob^*)
\leq\sum_{\rep\in\repsp_1}\mu_{\rep}
\left[\langle\mathbf L_{\rep},\calmap^*(\rep)\rangle
-\underline L_1(\calmap^*(\rep))\right]
=\delta_{\prop},
\]
which proves \Cref{eq:property-worst-case}.

If mutual calibration holds, the primary posterior map of its witness
has zero property regret, so $\delta_{\prop}=0$.
Conversely, if $\delta_{\prop}=0$, nonnegative regret and positive bin
probabilities imply
$\rep\in\prop_1(\calmap^*(\rep))$ for every $\rep\in\repsp_1$.
The extension witnessing $\calmap^*$ then calibrates both original
predictors. This proves the equivalence. When $\delta_{\prop}=0$,
\Cref{eq:property-worst-case} and its attaining witness also exclude a
rule that strictly improves under every feasible extension.
\end{proof}

\paragraph{Selecting reports after projection.}
For a fixed optimal $\calmap^*$, the required selection can be expressed
as a variational condition. Choose
\[
\mathbf z_{\rep}\in
\operatorname{conv}\{\mathbf L_{\rep'}:
\rep'\in\prop_1(\calmap^*(\rep))\}.
\]
Since $F(\calmap^*,\mathbf z)=\delta_{\prop}$, such a selection attains
the theorem's guarantee if and only if
\begin{equation}
\label{eq:property-selection-condition}
\E_{\prob_{\repRV_1}}
\left[\left\langle
\mathbf L_{\repRV_1}-\mathbf z_{\repRV_1},
\calmap(\repRV_1)-\calmap^*(\repRV_1)
\right\rangle\right]\geq0
\quad\text{for every }\calmap\in\mathcal C_{\prop}.
\end{equation}
The proof establishes existence of such a selection.
It can be found by solving the maximization in \Cref{eq:property-minimax}
with each $\mathbf z_{\rep}$ restricted to the displayed convex hull.
The optimal value remains $\delta_{\prop}$.

\subsection{Differentiable Case}

We now show that the optimal worst-case guarantee can be attained by a deterministic report-selection rule when the Bayes risk is differentiable at the projected conditional distributions. Under this condition, all Bayes-optimal reports at each projected distribution have the same loss vector, so any deterministic selection preserves the guarantee of \Cref{thm:elicitable}.

Here, differentiability concerns the Bayes risk as a function of its distribution argument. To make this precise, including at simplex boundaries, extend the Bayes risk to the ambient vector space:
\[
\underline L_1(\mathbf v)
:=\min_{\rep'\in\cA_1}\langle\mathbf L_{\rep'},\mathbf v\rangle,
\qquad\mathbf v\in\R^{|\statesp|}.
\]
Under the theorem's assumptions, this extension is finite and concave. All derivatives below are taken with respect to the ambient coordinates $\mathbf v$ and evaluated at probability vectors. The deterministic guarantee requires differentiability only at the projected distributions $\calmap^*(\rep)$ for $\rep\in\repsp_1$.

At a distribution $\pred$ where this extension is differentiable, define
\begin{equation}
\label{eq:property-induced-loss}
\score_1^\Delta(\pred,\state)
:=\underline L_1(\pred)
+\langle\nabla\underline L_1(\pred),\mathbf e_{\state}-\pred\rangle.
\end{equation}
For any $\rep'\in\prop_1(\pred)$, Bayes optimality gives
$\underline L_1(\pred)=\langle\mathbf L_{\rep'},\pred\rangle$.
Moreover, for every $\mathbf v\in\R^{|\statesp|}$,
\[
\underline L_1(\mathbf v)
\leq \langle\mathbf L_{\rep'},\mathbf v\rangle
=\underline L_1(\pred)
+\langle\mathbf L_{\rep'},\mathbf v-\pred\rangle.
\]
Thus $\mathbf L_{\rep'}$ is a supergradient of the ambient Bayes risk at $\pred$. Differentiability makes this supergradient unique, so
\[
\mathbf L_{\rep'}=\nabla\underline L_1(\pred),
\qquad
\score_1^\Delta(\pred,\state)=\score_1(\rep',\state).
\]
Consequently, all Bayes-optimal reports at $\pred$ have the same loss for every outcome.

Concavity also gives
\[
\E_{\staterv\sim\truepred}
[\score_1^\Delta(\pred,\staterv)]
=
\underline L_1(\pred)
+\langle\nabla\underline L_1(\pred),\truepred-\pred\rangle
\geq\underline L_1(\truepred),
\]
with equality when $\truepred=\pred$.
Thus $\score_1^\Delta$ is a proper loss for distribution-valued reports on its differentiability domain, and
\begin{equation}
\label{eq:property-regret-as-loss-gap}
\rho_{\score_1}(\pred,\rep)
=\E_{\staterv\sim\pred}
[\score_1(\rep,\staterv)-\score_1^\Delta(\pred,\staterv)].
\end{equation}

\begin{corollary}[Deterministic improvement at differentiable projections]
\label{cor:elicitable-smooth}
Under the assumptions of \Cref{thm:elicitable}, let $\calmap^*$ be a minimizer such that the ambient Bayes risk $\underline L_1$ is differentiable at $\calmap^*(\rep)$ for every $\rep\in\repsp_1$.
Any deterministic choice
$f^*(\rep)\in\prop_1(\calmap^*(\rep))$ attains the optimal worst-case gain $\delta_{\prop}$.
Equivalently, defining
\begin{equation}
\label{eq:property-direct-gain}
\operatorname{Gain}_{\Delta}(\calmap^*;\widetilde\prob)
:=\E_{\widetilde\prob}
[\score_1(\repRV_1,\staterv)
-\score_1^\Delta(\calmap^*(\repRV_1),\staterv)],
\end{equation}
we have
\[
\operatorname{Gain}_{\Delta}(\calmap^*;\widetilde\prob)
\geq\delta_{\prop}
\quad\text{for every }\widetilde\prob\in\mathfrak D_{\prop_0},
\qquad
\min_{\widetilde\prob\in\mathfrak D_{\prop_0}}
\operatorname{Gain}_{\Delta}(\calmap^*;\widetilde\prob)
=\delta_{\prop}.
\]
\end{corollary}

\begin{proof}
Every report in $\prop_1(\calmap^*(\rep))$ has loss vector
$\nabla\underline L_1(\calmap^*(\rep))$.
The randomized rule supplied by \Cref{thm:elicitable} is supported on this set. It therefore has the same loss as any deterministic choice $f^*(\rep)$ for every outcome in every report bin.
The claimed guarantee and its exact worst-case value follow from that theorem and \Cref{eq:property-induced-loss}.

In this case, the report-selection condition becomes
\begin{equation}
\label{eq:property-variational}
\E_{\prob_{\repRV_1}}
\left[\left\langle
\mathbf L_{\repRV_1}-\nabla\underline L_1(\calmap^*(\repRV_1)),
\calmap(\repRV_1)-\calmap^*(\repRV_1)
\right\rangle\right]\geq0
\quad\text{for every }\calmap\in\mathcal C_{\prop},
\end{equation}
which is the first-order optimality condition for minimizing the convex projection objective
\[
\Phi(\calmap)
=
\E_{\prob_{\repRV_1}}
\left[
\langle\mathbf L_{\repRV_1},\calmap(\repRV_1)\rangle
-\underline L_1(\calmap(\repRV_1))
\right]
\]
over $\mathcal C_{\prop}$.
\end{proof}
\subsection{Confidence Reports}

We now specialize the general construction to confidence reports, verifying the required regularity conditions and describing the resulting report-selection rule. The rule is deterministic when every projected distribution has a unique mode; otherwise, a suitable mixture over the tied modes attains the same optimal worst-case guarantee.

For the confidence loss in \Cref{eq:proper loss}, the report space is
$\cA_1=\statesp\times[0,1]$, and
\[
\score_1((a,c),\state)
=c^2-2c\,\mathbb I(\state=a),
\qquad
\mathbf L_{(a,c)}=c^2\mathbf 1-2c\,\mathbf e_a.
\]
Equip the finite answer set $\statesp$ with the discrete topology and
$[0,1]$ with its usual topology, and give $\cA_1$ the product topology.
This space is compact and Hausdorff.
For each fixed outcome $\state$ and answer $a$, the loss is a polynomial
in $c$ and hence continuous on $\{a\}\times[0,1]$.
Since these components are open in the product topology, the loss is
continuous on the entire report space.
Thus the compactness and report-continuity assumptions of
\Cref{thm:elicitable} hold.

For a fixed answer $a$, the expected loss is
\[
\E_{\staterv\sim\pred}[\score_1((a,c),\staterv)]
=c^2-2c\,\pred(a)
=(c-\pred(a))^2-\pred(a)^2.
\]
Minimizing over $c$ gives $c=\pred(a)$ and value $-\pred(a)^2$.
Minimizing over $a$ therefore yields
\[
\underline L_1(\pred)
=-\bigl(\max_{\state\in\statesp}\pred(\state)\bigr)^2,
\qquad
\prop_1(\pred)
=\left\{(a,c):
a\in\operatorname*{argmax}_{\state}\pred(\state),\
c=\max_{\state}\pred(\state)
\right\}.
\]
For this property, the level set of a report $(a,c)$ is
\[
\prop_1^{-1}(a,c)
=
\left\{\pred\in\Delta(\statesp):
\pred(a)=c,\ \pred(\state)\leq c
\ \text{for every }\state\in\statesp
\right\}.
\]
It is nonempty precisely when $c\in[1/|\statesp|,1]$.
Thus the standing assumption that observed reports have nonempty
level sets requires their confidence values to lie in this range.

The projection objective takes the explicit form
\[
\Phi(\calmap)
=\sum_{\rep=(a,c)\in\repsp_1}\mu_{\rep}
\left[
c^2-2c\,\calmap(\rep)(a)
+\bigl(\max_{\state}\calmap(\rep)(\state)\bigr)^2
\right].
\]
This objective is convex: the first two terms are affine in
$\calmap(\rep)$, and on the probability simplex the final term equals
$\max_{\state}\calmap(\rep)(\state)^2$, a maximum of convex functions.

Fix an optimal projection and write
\[
\pred_{\rep}^*:=\calmap^*(\rep),\qquad
c_{\rep}^*:=\max_{\state}\pred_{\rep}^*(\state),\qquad
M_{\rep}:=\operatorname*{argmax}_{\state}\pred_{\rep}^*(\state).
\]
Under the standing assumptions of \Cref{thm:elicitable}, that theorem
supplies probability vectors $\alpha_{\rep}$ supported on $M_{\rep}$
such that the output $(A,c_{\rep}^*)$, with $A\sim\alpha_{\rep}$,
attains worst-case expected gain $\delta_{\prop}$ under the original
confidence loss.
Its expected loss vector is
\[
\mathbf z_{\rep}
=(c_{\rep}^*)^2\mathbf 1-2c_{\rep}^*\alpha_{\rep}.
\]
The weights must satisfy \Cref{eq:property-selection-condition}.
Equivalently, they can be chosen by solving the following optimization
with the projection held fixed:
\begin{equation}
\label{eq:confidence-report-selection}
\max_{\substack{
\alpha_{\rep}\in\Delta(\statesp)\ \forall\rep\in\repsp_1\\
\operatorname{supp}(\alpha_{\rep})\subseteq M_{\rep}\ \forall\rep\in\repsp_1}}
\ \min_{\calmap\in\mathcal C_{\prop}}
\sum_{\rep\in\repsp_1}\mu_{\rep}
\left\langle
\mathbf L_{\rep}-(c_{\rep}^*)^2\mathbf 1
+2c_{\rep}^*\alpha_{\rep},
\calmap(\rep)
\right\rangle.
\end{equation}
Its optimal value is $\delta_{\prop}$.

If every $M_{\rep}$ is a singleton, the report-selection rule is deterministic. The support constraint then uniquely determines each
$\alpha_{\rep}$, eliminating the report-selection optimization.
To connect this case to \Cref{cor:elicitable-smooth}, consider a
distribution $\pred$ with unique mode $a^*$ and maximum probability $c$.
Its Bayes-optimal report $(a^*,c)$ is unique.
Since the loss-vector set is compact and the minimizing loss vector
is unique, the ambient Bayes risk defined in the preceding subsection
is differentiable at $\pred$, with gradient
\[
\nabla\underline L_1(\pred)
=c^2\mathbf 1-2c\,\mathbf e_{a^*}.
\]
Consequently, \Cref{cor:elicitable-smooth} applies, and
\[
\score_1^\Delta(\pred,\state)
=c^2-2c\,\mathbb I(\state=a^*)
=\score_1((a^*,c),\state).
\]

At a tied mode, distinct optimal answers have different loss vectors
and hence give distinct supergradients of the ambient Bayes risk,
which is therefore nondifferentiable at that distribution.
The assumptions of \Cref{thm:elicitable} still hold, so
\Cref{eq:confidence-report-selection} yields a suitable mixture over
the tied answers, all with confidence $c_{\rep}^*$.
Uniform averaging or a deterministic tie-breaking rule attains the
same bound if its loss vectors satisfy
\Cref{eq:property-selection-condition}; membership in the set of
Bayes-optimal reports alone is insufficient.
The guarantee includes expectation over the output randomization
and holds for every feasible ground-truth extension.

\section{Mutual Calibration, Common Priors, and No-Trade Theorem}\label{sec:arbitrage}

In this section, we build the connection between our \textbf{mutual calibration} concept and the well-known economic literature on the No-Trade Theorem \cite{milgrom1982information,arieli2021feasible,morris2020no}. We first rephrase the results in \citet{arieli2021feasible,morris2020no} within our mutual-calibration context and then establish a quantitative relationship between the predictors' calibration errors and the arbitrage opportunities studied in that literature.

\subsection{Common Priors and the No-Trade Characterization}

Fix the observed joint report distribution $\prob_{\predRV_0,\predRV_1}$; as in the main text, index $0$ denotes the reference predictor and index $1$ the primary predictor. In \citet{arieli2021feasible}, this joint distribution is called \emph{feasible joint posterior beliefs} if there exists an extension $\widetilde\prob_{\predRV_0,\predRV_1,\staterv}$ with marginal $\prob_{\predRV_0,\predRV_1}$ such that Bayesian predictor $i$'s posterior will coincide with $\predRV_i$, i.e.
\begin{equation}
\E_{\widetilde\prob}[\statervvec\mid\predRV_i]=\predRV_i,
\qquad i\in\{0,1\}.
\label{eq:common-prior-posterior}
\end{equation}
This is exactly mutual calibration (\Cref{def:mutual-calibration}). Indeed, one may take $\predRV_i$ itself as predictor $i$'s signal; then \Cref{eq:common-prior-posterior} says that the reported vector is the predictor's Bayesian posterior under the shared law $\widetilde\prob$. Conversely, posterior rationality under a common prior gives the two calibration identities in \Cref{eq:common-prior-posterior}. Thus checking mutual calibration is the same feasibility problem as checking whether the reported posteriors admit a common-prior rationalization.

The classical No-Trade Theorem shows that a common prior rules out the arbitrage opportunity between predictors \citep{milgrom1982information}. The feasibility results of \citet{arieli2021feasible,morris2020no} formalize the corresponding no-trade or money-pump characterization for a distribution of posterior reports. We next state that characterization in our notation and introduce a normalization that makes the size of an arbitrage opportunity finite.

Let $\mathcal K$ be the set of conditional kernels $\kernel(\cdot\mid\pred_0,\pred_1)\in\Delta(\statesp)$, one for every $(\pred_0,\pred_1)\in\predsp_0\times\predsp_1$. Each $\kernel\in\mathcal K$ induces the candidate extension
\begin{equation}
\widetilde\prob^{\kernel}(\pred_0,\pred_1,\state)
:=\prob_{\predRV_0,\predRV_1}(\pred_0,\pred_1)
\kernel(\state\mid\pred_0,\pred_1).
\label{eq:kernel-extension-arbitrage}
\end{equation}
Only a kernel for which $\widetilde\prob^{\kernel}$ satisfies \Cref{eq:common-prior-posterior} for both predictors witnesses mutual calibration.

For predictor $i\in\{0,1\}$, a state-contingent payment $t_i:\predsp_i\times\statesp\to\R$ is a \emph{zero-value trade} if its value under every reported belief is zero. We use the normalized class
\begin{equation}
\mathcal T_i
:=\left\{t_i:
\sum_{\state\in\statesp}\pred_i(\state)t_i(\pred_i,\state)=0,
\quad
\lVert t_i(\pred_i,\cdot)\rVert_\infty\leq 1
\quad\text{for every }\pred_i\in\predsp_i
\right\}.
\label{eq:normalized-zero-trades}
\end{equation}
Here $t_i(\pred_i,\state)$ is the payment from an external arbitrageur to predictor $i$. Hence each trade is weakly acceptable to the predictor, while the arbitrageur's realized payoff is $-\sum_{i=0}^1t_i(\pred_i,\state)$ when the reported pair is $(\pred_0, \pred_1)$ and the realized outcome is $\state$.

Because the real outcome conditional kernel is unobserved, define the normalized robust arbitrage value achieved by the outsider as
\begin{align}
\operatorname{Arb}(\prob_{\predRV_0,\predRV_1})
&:=\max_{t_0\in\mathcal T_0,\,t_1\in\mathcal T_1}
\min_{\kernel\in\mathcal K}
\E_{\widetilde\prob^{\kernel}}
\left[-\sum_{i=0}^1t_i(\predRV_i,\staterv)\right] \nonumber\\
&=-\min_{t_0\in\mathcal T_0,\,t_1\in\mathcal T_1}
\E_{\prob_{\predRV_0,\predRV_1}}
\left[\max_{\state\in\statesp}\sum_{i=0}^1t_i(\predRV_i,\state)\right].
\label{eq:normalized-arbitrage}
\end{align}
For fixed trades, the inner minimization may choose an outcome separately at each report pair, which gives the second equality. The zero trades ($t_{i}(\pred_i, \state)\equiv 0$) are feasible, so the value is nonnegative. Moreover, \Cref{eq:normalized-arbitrage} is zero exactly when every normalized zero-value scheme satisfies
\begin{equation}
\E_{\prob_{\predRV_0,\predRV_1}}
\left[\max_{\state\in\statesp}\sum_{i=0}^1t_i(\predRV_i,\state)\right]\geq 0.
\label{eq:no-money-pump-condition}
\end{equation}
This is the no-money-pump condition \citep{arieli2021feasible,morris2020no} in mutual-calibration notation with normalization. The following example illustrates how a failure of mutual calibration creates a concrete arbitrage opportunity.

\begin{example}[Arbitrage in Contradictory Predictions]
\label{example:arbitrage intro}
Consider the setting of \Cref{example: intro contradictory}, where the primary predictor outputs $10\%$ (Group~A) and $70\%$ (Group~B), and the reference predictor outputs a calibrated base rate of $50\%$. An arbitrageur offers the following contracts ($0<\epsilon<0.1$):
\begin{itemize}
    \item \textbf{To the reference predictor:} pays $(0.5+\epsilon)$ if the investment doubles; receives $0.5$ if not.
    \item \textbf{To the primary predictor:} when predicting $0.7$, receives $0.3$ if the investment doubles and pays $(0.7+\epsilon)$ otherwise; when predicting $0.1$, receives $0.9$ if the investment doubles and pays $(0.1+\epsilon)$ otherwise.
\end{itemize}
Both predictors accept these contracts because their subjective expected profits are strictly positive. The arbitrageur's payoff is $0.4-\epsilon$ in Group~A and $-0.2-\epsilon$ in Group~B, independently of the realized outcome. Since the groups have equal probability, the expected payoff over the fixed report distribution $\prob_{\predRV_0,\predRV_1}$ is $0.1-\epsilon>0$ and holds uniformly over all possible outcome, indicating an arbitrage opportunity through the contradictory predictions.
\end{example}

\begin{theorem}[No-Trade Theorem in mutual-calibration form \citep{arieli2021feasible,morris2020no}]
\label{thm:no-arbitrage}
Under finite report and outcome supports, the following statements are equivalent:
\begin{enumerate}
    \item $\predRV_0$ and $\predRV_1$ are mutually calibrated;
    \item $\prob_{\predRV_0,\predRV_1}$ admits a common-prior rationalization; and
    \item $\operatorname{Arb}(\prob_{\predRV_0,\predRV_1})=0$.
\end{enumerate}
Consequently, failure of mutual calibration is equivalent to a strictly positive normalized robust
arbitrage value.
\end{theorem}

\begin{proof}
The equivalence of the first two statements follows from \Cref{eq:common-prior-posterior}. By \Cref{eq:normalized-arbitrage}, the zero-trade scheme is feasible and
\[
\operatorname{Arb}(\prob_{\predRV_0,\predRV_1})=0
\]
if and only if the no-money-pump condition in \Cref{eq:no-money-pump-condition} holds for every normalized zero-value trade. Under finite supports, the normalization does not change this qualitative condition: any zero-value trade violating the condition can be multiplied by a positive constant so that it belongs to $\mathcal T_0\times\mathcal T_1$, while preserving the sign of its expected payoff. The no-money-pump characterization of \citet{arieli2021feasible,morris2020no} states that this condition holds if and only if the joint report distribution admits a common-prior rationalization. Combining these equivalences proves the theorem.
\end{proof}

\subsection{A Quantitative Calibration--Arbitrage Relationship}

The no-trade characterization records whether a common-prior extension exists. In this section, we give a quantitative condition, and build the connection with calibration errors.

For each $\kernel\in\mathcal K$, define the posterior map on predictor $i$'s original report bins by
\begin{equation}
\eta_i^{\kernel}(\pred_i)
:=\E_{\widetilde\prob^{\kernel}}
[\statervvec\mid\predRV_i=\pred_i],
\qquad i\in\{0,1\},
\label{eq:kernel-posterior-maps}
\end{equation}
Therefore, its predictor-wise and total calibration gaps can be represented by
\begin{align}
\operatorname{CalGap}_i(\kernel)
&:=\E_{\widetilde\prob^{\kernel}}
\left[\lVert\predRV_i-\eta_i^{\kernel}(\predRV_i)\rVert_1\right],
\qquad i\in\{0,1\}, \nonumber\\
\operatorname{CalGap}(\kernel)
&:=\sum_{i=0}^1\operatorname{CalGap}_i(\kernel).
\label{eq:fixed-kernel-calibration-gap}
\end{align}

\begin{definition}[Minimal full-distribution calibration error (MinECE)]
\label{def:min-ece}
For the observed report distribution $\prob_{\predRV_0,\predRV_1}$, define
\begin{align}
\operatorname{MinECE}(\prob_{\predRV_0,\predRV_1})
&:=\frac12\min_{\kernel\in\mathcal K}\operatorname{CalGap}(\kernel) \nonumber\\
&=\frac12\min_{\kernel\in\mathcal K}
\sum_{i=0}^1\E_{\widetilde\prob^{\kernel}}
\left[
\left\lVert\predRV_i-
\E_{\widetilde\prob^{\kernel}}[\statervvec\mid\predRV_i]
\right\rVert_1
\right].
\label{eq:min-ece}
\end{align}
\end{definition}

The factor $1/2$ converts each $\ell_1$ discrepancy into total variation. To relate this calibration distance to arbitrage, we use the following lemma to bridge $\ell_1$ distance to maximal inner product.

\begin{lemma}[Centered-trade duality]\label{lem:onenorm}
For any reported distribution $\pred\in\Delta(\statesp)$ and any outcome distribution
$\truepred\in\Delta(\statesp)$, let
\[
\Phi(\pred,\truepred)
:=\max_{\substack{\lVert\mathbf v\rVert_\infty\leq 1\\
\langle\mathbf v,\pred\rangle=0}}
\langle\mathbf v,\truepred\rangle.
\]
Then
\begin{equation}
\frac12\lVert\truepred-\pred\rVert_1
\leq \Phi(\pred,\truepred)
\leq \lVert\truepred-\pred\rVert_1.
\label{eq:centered-trade-duality}
\end{equation}
\end{lemma}

\begin{proof}
For every feasible $\mathbf v$, the centering constraint and H\"older's inequality give
\[
\langle\mathbf v,\truepred\rangle
=\langle\mathbf v,\truepred-\pred\rangle
\leq\lvert\langle\mathbf v,\truepred-\pred\rangle\rvert
\leq\lVert\mathbf v\rVert_\infty\lVert\truepred-\pred\rVert_1
\leq\lVert\truepred-\pred\rVert_1,
\]
which proves the upper bound. For the lower bound, set
$\mathbf u(\state)=\operatorname{sign}(\truepred(\state)-\pred(\state))$, with
$\operatorname{sign}(0)=0$, and let $\alpha=\langle\mathbf u,\pred\rangle$. Then
\[
\mathbf v^*:=\frac{\mathbf u-\alpha\mathbf 1}{2}
\]
satisfies $\lVert\mathbf v^*\rVert_\infty\leq1$ and
$\langle\mathbf v^*,\pred\rangle=0$. Since both distributions have unit mass,
\[
\langle\mathbf v^*,\truepred\rangle
=\frac12\langle\mathbf u,\truepred-\pred\rangle
=\frac12\lVert\truepred-\pred\rVert_1,
\]
which proves the lower bound.
\end{proof}

For a fixed report $\pred_i$, apply \Cref{lem:onenorm} with $\pred=\pred_i$, $\truepred=\eta_i^{\kernel}(\pred_i)$, and $\mathbf v(\state)=-t_i(\pred_i,\state)$. Trades may be selected independently across report bins and predictors, so averaging and summing gives the fixed-kernel bridge
\begin{equation}
\frac12\operatorname{CalGap}(\kernel)
\leq\operatorname{TradeVal}(\kernel)
\leq\operatorname{CalGap}(\kernel).
\label{eq:fixed-kernel-bridge}
\end{equation}
where
\begin{equation}
\operatorname{TradeVal}(\kernel)
:=\max_{t_0\in\mathcal T_0,\,t_1\in\mathcal T_1}
\E_{\widetilde\prob^{\kernel}}
\left[-\sum_{i=0}^1t_i(\predRV_i,\staterv)\right].
\label{eq:fixed-kernel-trade-value}
\end{equation}
is the maximal normalized payoff under this specific kernel $\kernel$.

\begin{proposition}[Quantitative calibration--arbitrage relation]
\label{pro:arb}
Under the finite report and outcome supports,
\begin{equation}
\operatorname{MinECE}(\prob_{\predRV_0,\predRV_1})
\leq\operatorname{Arb}(\prob_{\predRV_0,\predRV_1})
\leq2\operatorname{MinECE}(\prob_{\predRV_0,\predRV_1}).
\label{eq:min-ece-arbitrage-bound}
\end{equation}
Hence the arbitrage value and the minimal calibration error vanish simultaneously, precisely under
mutual calibration.
\end{proposition}

\begin{proof}

Under finite supports, $\mathcal T_0\times\mathcal T_1$ and $\mathcal K$ are compact convex sets, and the payoff is bilinear. The finite-dimensional minimax theorem \citep{sion1958general} therefore yields
\begin{equation}
\operatorname{Arb}(\prob_{\predRV_0,\predRV_1})
=\min_{\kernel\in\mathcal K}\operatorname{TradeVal}(\kernel).
\label{eq:arbitrage-min-kernel}
\end{equation}

Minimizing \Cref{eq:fixed-kernel-bridge} over $\kernel\in\mathcal K$ and using
\Cref{eq:arbitrage-min-kernel,eq:min-ece} gives
\[
\frac12\min_{\kernel\in\mathcal K}\operatorname{CalGap}(\kernel)
\leq\min_{\kernel\in\mathcal K}\operatorname{TradeVal}(\kernel)
\leq\min_{\kernel\in\mathcal K}\operatorname{CalGap}(\kernel),
\]
which is exactly \Cref{eq:min-ece-arbitrage-bound}. The final equivalence follows from
\Cref{thm:no-arbitrage}.
\end{proof}

\begin{remark}
MinECE here is a population, full-distribution calibration quantity, not the binned top-label or confidence ECE commonly reported in experiments. It evaluates the complete multiclass probability vector on each original report bin and measures the discrepancy in total variation, matching the stronger \emph{multi-class calibration} discussed by \citet{kull2019beyond}.
\end{remark}

\section{Experiment Details}
\label{app:experiment details}

\subsection{Dataset Details}
\label{sec:dataset details}

We evaluate our method on three established benchmarks. All three are formulated as multiple-choice tasks, though they differ in the number of answer options:
\begin{itemize}
    \item \textbf{MMLU~\citep{hendrycks2021measuring}} (MIT License). This dataset is a \textbf{4-option} multiple-choice benchmark containing questions from various subjects. We evaluate our method on the \emph{test} split, which contains 14,042 samples. To estimate the joint distribution $\prob_{\repRV_0,\repRV_1}$ or tune parameters, we use the \emph{validation} split, which contains 1,531 samples.

    \item \textbf{MMLU-Redux~\citep{gema2025we}} (CC BY 4.0).  This dataset is a \textbf{4-option} multiple-choice benchmark containing a manually curated subset of \textbf{MMLU}, designed to remove annotation errors and ambiguous questions. We utilize the \emph{test} split for evaluation, which contains 5,700 samples. Since MMLU-Redux lacks a distinct validation set, we employ the \emph{validation split of the original MMLU dataset} as the validation set to estimate the joint distribution $\prob_{\repRV_0, \repRV_1}$ for our method or tune parameters for supervised baselines, which contains 1,531 samples.

    \item \textbf{CommonSenseQA (CSQA)~\citep{talmor2019commonsenseqa}} (MIT License). This dataset features \textbf{5-option} multiple-choice questions requiring various types of commonsense knowledge. We evaluate our method on the \emph{validation} split, which contains 1,221 samples, as the labels for the test split are unavailable. To estimate the joint distribution $\prob_{\repRV_0,\repRV_1}$ or tune parameters, we use the \emph{training} split, which contains 9,741 samples.
\end{itemize}

\paragraph{Asset Sources and Licenses.}
\label{sec:asset-versions}
Table~\ref{tab:existing-assets} lists the datasets and model checkpoints used in our experiments, together with their repository identifiers, original sources, and licenses.
\begin{table}[t]
\centering
\caption{Repository identifiers, sources, and licenses of existing experimental assets. Each model row identifies the checkpoint used in our experiments.}
\label{tab:existing-assets}
\begin{tabular}{@{}p{0.23\linewidth}p{0.43\linewidth}p{0.24\linewidth}@{}}
\toprule
Asset & Repository identifier and source & License \\
\midrule
MMLU~\citep{hendrycks2021measuring}
& \href{https://huggingface.co/datasets/cais/mmlu}{\nolinkurl{cais/mmlu}}
& MIT \\
MMLU-Redux 2.0~\citep{gema2025we}
& \href{https://huggingface.co/datasets/edinburgh-dawg/mmlu-redux-2.0}{\nolinkurl{edinburgh-dawg/mmlu-redux-2.0}}
& CC BY 4.0 \\
CommonsenseQA~\citep{talmor2019commonsenseqa}
& \href{https://huggingface.co/datasets/tau/commonsense_qa}{\nolinkurl{tau/commonsense_qa}}
& MIT \\
Qwen3 reference~\citep{yang2025qwen3}
& \href{https://huggingface.co/Qwen/Qwen3-8B-Base}{\nolinkurl{Qwen/Qwen3-8B-Base}}
& Apache 2.0 \\
Qwen3 primary~\citep{yang2025qwen3}
& \href{https://huggingface.co/Qwen/Qwen3-8B}{\nolinkurl{Qwen/Qwen3-8B}}
& Apache 2.0 \\
Llama reference~\citep{grattafiori2024llama}
& \href{https://huggingface.co/meta-llama/Llama-3.1-8B}{\nolinkurl{meta-llama/Llama-3.1-8B}}
& Llama 3.1 Community License \\
Llama primary~\citep{grattafiori2024llama}
& \href{https://huggingface.co/meta-llama/Llama-3.1-8B-Instruct}{\nolinkurl{meta-llama/Llama-3.1-8B-Instruct}}
& Llama 3.1 Community License \\
Ministral reference~\citep{liu2026ministral}
& \href{https://huggingface.co/mistralai/Ministral-3-8B-Base-2512}{\nolinkurl{mistralai/Ministral-3-8B-Base-2512}}
& Apache 2.0 \\
Ministral primary~\citep{liu2026ministral}
& \href{https://huggingface.co/mistralai/Ministral-3-8B-Instruct-2512-BF16}
       {\nolinkurl{mistralai/Ministral-3-8B-Instruct-2512-BF16}}
& Apache 2.0 \\
\bottomrule
\end{tabular}
\end{table}
The \texttt{LLaMA-3.1-8B} checkpoints are also subject to the accompanying Acceptable Use Policy, as specified in their linked model cards.

The implementation uses \href{https://github.com/pytorch/pytorch}{PyTorch 2.11.0} (BSD 3-Clause) and \href{https://github.com/vllm-project/vllm}{vLLM 0.26.0} (Apache 2.0). Convex optimization uses \href{https://github.com/cvxpy/cvxpy}{CVXPY 1.9.3} and \href{https://github.com/oxfordcontrol/Clarabel.rs}{Clarabel 0.11.1}, both under Apache 2.0. Other principal dependencies are \href{https://github.com/numpy/numpy}{NumPy 2.3.5}, \href{https://github.com/scipy/scipy}{SciPy 1.17.1}, and \href{https://github.com/scikit-learn/scikit-learn}{scikit-learn 1.9.1} (BSD 3-Clause), together with \href{https://github.com/huggingface/transformers}{Transformers 5.13.0}, \href{https://github.com/huggingface/datasets}{Datasets 5.0.1}, and \href{https://github.com/mistralai/mistral-common}{mistral-common 1.11.7} (Apache 2.0). The adapted SmoothECE implementation is derived from \href{https://github.com/apple/ml-calibration}{Apple's \texttt{ml-calibration}} and retains its upstream Apple license. The corresponding license notice is included in our public repository.

\subsection{Implementation Details}
\label{sec:implementation details}

\paragraph{Code and Data Availability.}
\label{sec:code-availability}
The implementation and reproduction instructions are available at \url{https://github.com/smyft/calibration-without-ground-truth}. The datasets and model checkpoints are obtained from their original sources, subject to their respective licenses and terms of use.

\paragraph{Reproduction Configuration.}
\label{sec:reproduction-configuration}
The experiments use Python 3.11.15, PyTorch \texttt{2.11.0+cu130} (the CUDA 13.0 build), and vLLM 0.26.0, together with the dependency versions listed above. The dataset and model repository identifiers, sources, and licenses are listed in Table~\ref{tab:existing-assets}. The exact configuration and execution commands are provided in \url{https://github.com/smyft/calibration-without-ground-truth}.

\paragraph{Compute Resources.}
\label{sec:compute-resources}
Experiments were run on \href{https://www.autodl.com/}{AutoDL}. LLM inference used 1 RTX 6000D with 84 GB of device memory. The host system provided 18 vCPU Intel(R) Xeon(R) Platinum 8470Q and 110 GB of RAM.

As shown in Appendix~\ref{sec:prompts}, we use six LLM checkpoints, comprising three Base/Instruct pairs. Base inference uses one stage, whereas Instruct inference uses a reasoning stage followed by an answer-extraction stage. The inference stages used to produce the final predictions took approximately 6 GPU hours. After inference, fitting and deploying the calibration map required only a few seconds on the above CPU configuration.

\subsubsection{Prompting Strategy and Confidence Extraction}
\label{sec:prompts}

We describe the prompting strategies and the subsequent procedure for extracting confidence scores.

\begin{figure}[htbp]
    \centering
    \begin{tcolorbox}[colback=gray!10, title={Standard zero-shot prompts for \emph{base} model},fontupper=\ttfamily]
        The following is a multiple-choice question about \{subject\}.  \\
        \\
        Question: \{task\_query\} \\
        A. \{choiceA\} \\
        B. \{choiceB\} \\
        C. \{choiceC\} \\
        D. \{choiceD\} \\
        E. \{choiceE\} \% only for CommonSenseQA \\
        \\
        Answer:
    \end{tcolorbox}
    \caption{Standard zero-shot prompt template utilized for \textit{Base} models to directly elicit the first token probability following the question. The \texttt{\{subject\}} placeholder is instantiated with the specific domain for MMLU/MMLU-Redux or \texttt{common} \texttt{sense} for CommonsenseQA, and \texttt{\{task\_query\}},\texttt{\{choice\}} placeholders are instantiated with the question and choices of the specific sample.}
    \label{fig:prompts base models}
\end{figure}

\definecolor{keywordcolor}{RGB}{0, 0, 150}
\begin{figure}[htbp]
    \centering
    \begin{tcolorbox}[colback=gray!10, title={Two-step zero-shot CoT prompts for \emph{instruct} model},fontupper=\ttfamily]

        \textbf{\textcolor{keywordcolor}{Step 1: Reasoning Generation}}
        \vspace{0.2cm}

        \textbf{System Prompts:} \\
        The following is a multiple-choice question about \{subject\}.

        \vspace{0.2cm}
        \hrule
        \vspace{0.2cm}

        \textbf{User Prompt:} \\
        Question: \{task\_query\} \\
        A. \{choiceA\} \\
        B. \{choiceB\} \\
        C. \{choiceC\} \\
        D. \{choiceD\} \\
        E. \{choiceE\} \% only for CommonSenseQA \\

        \textbf{Model Response:} \textcolor{blue}{\{reasoning\_process\}} \\

        \vspace{0.2cm}
        \textbf{\textcolor{keywordcolor}{Step 2: Answer Extraction}}
        \vspace{0.2cm}

        \textbf{System Prompts:} \\
        The following is a multiple-choice question about \{subject\}.

        \vspace{0.2cm}
        \hrule
        \vspace{0.2cm}

        \textbf{User Prompt:} \\
        Question: \{task\_query\} \\
        A. \{choiceA\} \\
        B. \{choiceB\} \\
        C. \{choiceC\} \\
        D. \{choiceD\} \\
        E. \{choiceE\} \% only for CommonSenseQA \\

        \vspace{0.2cm}
        \hrule
        \vspace{0.2cm}

        \textbf{Assistant Prompt:} \\
        \textcolor{blue}{\{filtered\_reasoning\}} \\
        So the final answer is (answer with an option letter):

    \end{tcolorbox}

    \caption{Two-step zero-shot CoT prompt templates  designed for \textit{Instruct} models. We first prompt the model to generate a reasoning chain (Step 1), and then filter the reasoning process and subsequently append an answer-triggering phrase to extract the next-token probability (Step 2). The \texttt{\{subject\}} placeholder is instantiated with the specific domain for MMLU/MMLU-Redux or \texttt{common} \texttt{sense} for CommonsenseQA, and \texttt{\{task\_query\}},\texttt{\{choice\}} placeholders are instantiated with the question and choices of the specific sample.}
    \label{fig:prompts instruct models}
\end{figure}

\paragraph{Prompting Strategy.}
Following previous studies~\citep{hendrycks2021measuring,wang-etal-2024-answer-c}, we use standard zero-shot prompts to directly elicit the answer from the first token log probabilities. This works well for \textbf{Base} models, but \textbf{Instruct} models typically generate reasoning contents before the final answer, making direct next-token probability extraction unreliable. To employ their CoT capabilities~\citep{wei2022cot}, we adopt a two-step approach widely used in instructing LLMs to verbalize confidence scores~\citep{lin2022teaching,tian-etal-2023-just}. We first prompt the model to generate a reasoning process using the standard zero-shot prompts. The generated reasoning process is filtered before being used in the second stage. Subsequently, we append this reasoning process along with an answer-triggering phrase to extract the next-token probability. Specific prompt templates are detailed in Figure~\ref{fig:prompts base models} and Figure~\ref{fig:prompts instruct models}.

\paragraph{Confidence Extraction.}
After extracting the log-probabilities from LLMs, we construct the (answer, confidence) pair through the following renormalization process. Let $\statesp$ denote the set of valid answer options. We construct an option log-score $\mathcal P_{\mathcal M}(\state\mid\mathbf x)$ from the returned top-20 next-token log-probabilities and normalize these scores over $\statesp$:
\begin{equation*}
    \hat{\mathcal{P}}_{\mathcal{M}}(\state|\mathbf{x}) = \frac{\exp(\mathcal{P}_{\mathcal{M}}(\state|\mathbf{x}))}{\sum_{\state' \in \statesp} \exp(\mathcal{P}_{\mathcal{M}}(\state'|\mathbf{x}))},
\end{equation*}
The answer token with the highest re-normalized probability is selected as the predicted answer, with $\hat{\mathcal{P}}_{\mathcal{M}}(\state|\mathbf{x})$ serving as the associated confidence score.

\subsubsection{Joint Distribution Estimation}

To compute the recalibration map, we approximate the joint distribution $\prob_{\repRV_0, \repRV_1}$ using the empirical distribution of paired predictor reports from the validation dataset. Let the report from a predictor be denoted as a pair $\rep=(\hat{\state}, \hat{c})$, consisting of an answer $\hat{\state} \in \statesp$ and a continuous confidence score $\hat{c} \in [0, 1]$.

Given the continuous nature of confidence scores, direct estimation of the joint distribution on finite samples is theoretically sparse. Therefore, we utilize a discretization process with a step size $s\in(0,1]$ with $1/s\in\mathbb{N}$. We define a quantization function $\psi_s(c)=s\left\lfloor\frac{c}{s}+\frac12\right\rfloor$ that maps a continuous confidence score to the nearest value on the discrete grid $\mathcal{C}_{s} = \{0, s, 2s, \dots, 1\}$. Specifically, for an observed report $\rep_{i} = (\hat{\state}_i, \hat{c}_i)$, the discretized report is given by $\tilde{\rep}_{i} = (\hat{\state}_i, \psi_s(\hat{c}_i))$. This discretization applies only to the input reports, but does not constrain the coordinates of the projected outcome distributions to grid values.

Let $\{(\tilde{\rep}_{0i}, \tilde{\rep}_{1i})\}_{i=1}^N$ denote the set of discretized report pairs generated by the \emph{Base} and \emph{Instruct} models for $N$ input samples in the validation dataset. The estimated joint distribution is then formulated as the normalized frequency of these discretized pairs:
\begin{equation*}
    \hat{\prob}_{\repRV_0, \repRV_1} = \frac{1}{N} \sum_{i=1}^N \delta_{(\tilde{\rep}_{0i}, \tilde{\rep}_{1i})},
\end{equation*}
where $\delta_{(\cdot)}$ is the Dirac delta mass. In our main experiments, we set $s = 0.01$, ensuring a fine-grained yet computationally tractable grid for probability estimation. We also report results under discretization step $s\in\{0.1, 0.05, 0.02, 0.01, 0.005, 0.002\}$ in Appendix~\ref{sec:discretization sensitivity} to check the sensitivity of our method to the confidence discretization step $s$.

\subsubsection{Algorithm Details}
\label{sec:algorithm details}

In practice, we first solve the projection in \Cref{eq:property-projection} using the empirical distribution of discretized reports, restricting $\calmap$ to the observed primary-report support. For Confidence Loss, the finite-support optimization yields a second-order cone formulation with linear reference-compatibility constraints.
The resulting finite-dimensional convex program is solved using Clarabel 0.11.1 through CVXPY 1.9.3 in double precision. We use the default solver settings, including a maximum of 200 iterations and target tolerances of $10^{-8}$ for feasibility and absolute and relative duality gaps. The default reduced-accuracy termination criteria are also retained. Across the nine default-grid settings, four runs were reported as \texttt{optimal} and five as \texttt{optimal\_inaccurate}. All returned primal solutions, with no solver errors or infeasibility reports. In the 54-setting discretization study, the corresponding counts were 23 and 31. At the default discretization step $s=0.01$, the CVXPY solve calls took 0.10--5.51 seconds per (dataset, model) setting on the CPU.

Then, we convert each computed projection vector into an (answer, confidence) report by selecting the first maximizing class and its probability. If every exact projected distribution has a unique mode, the report-selection problem in \Cref{eq:confidence-report-selection} reduces to this deterministic choice, requiring no additional mixture optimization or sampling. Our implementation always uses this first-maximizer rule, without optimizing mixtures over tied modes. Across the normalized projection vectors from the nine default discretization step $s=0.01$ settings and 54 discretization sensitivity analysis settings, no two largest entries were exactly equal in floating-point arithmetic. The smallest top-two gaps before output-confidence rounding were $3.17\times10^{-8}$ and $1.55\times10^{-15}$, respectively.

At deployment, the learned map is applied to the primary reports without another reference-model query. For an unseen discretized report, we set the transformation to the identity map. Specifically, let $r=(a,c)$ be a primary report and $\tilde r=(a,\psi_s(c))$ its discretized version. Let $\mathcal S$ denote the observed support of discretized primary reports, and let $F_s$ denote the learned report-valued transformation on this support. At deployment, we use
\[
\overline F_s(\tilde r)=
\begin{cases}
F_s(\tilde r), & \tilde r\in\mathcal S,\\
\tilde r, & \tilde r\notin\mathcal S.
\end{cases}
\]
Support membership requires an exact match of both the predicted answer and the discretized confidence. For an unseen discretized report, both components are retained unchanged. For mapped reports, answer selection precedes output rounding. Only the confidence component of the final report is rounded to four decimal places, independently of $s$. At our default $s=0.01$, such out-of-support predictions account for at most 4--6\% of test instances on MMLU-Redux and under 1\% on CommonSenseQA, confirming that the learned recalibration map covers the vast majority of test-time inputs.

\subsection{Baseline Details}
\label{sec:baselines details}

We compare our method against three categories of baselines that differ in their information requirements and deployment costs: (i) a \textbf{dummy reference} that provides no useful information, (ii) \textbf{supervised} calibration methods that require a labeled validation set, and (iii) \textbf{label-free} baselines that, like our method, require no labels but instead leverage a reference model. Within the label-free category, we further distinguish between \emph{test-time} heuristics that query the reference model on every test input, and \emph{validation-time} methods that, like our own, use the reference model only to learn a recalibration mapping on a held-out calibration split and require no access to the reference model at test time. All \textbf{supervised} and \textbf{label-free} baselines preserve the original predicted answer $\hat{\state}$ and only modify the associated confidence $\hat{c}$, producing a calibrated score $\tilde{c}$. In contrast, our method re-estimates the whole probability vector, providing opportunities to update the final answer.

\paragraph{Dummy Reference.}
As a sanity check, we create a \textbf{Dummy Reference} by randomizing the option order and replacing the reference model with a uniform distribution over answer positions. This creates a calibrated reference while throwing away all informative signals from the reference. Thus it isolates whether observed gains stem from useful reference information rather than from generic conservative shrinkage of confidences.

\paragraph{Supervised Baselines.}
Although these baselines are originally developed for full probability distributions over the outcome space $\statesp$, we adapt them to our (answer, confidence) setting by formulating calibration as a binary classification task. Specifically, we define the ground truth for the $i$-th sample as $z_i = \mathbb{I}(\hat{\state}_i = \state_i) \in \{0, 1\}$, indicating whether the predicted answer $\hat{\state}_i$ matches the true label $\state_i$. Consequently, the calibration dataset is constructed as pairs of initial confidence and correctness indicators: $\{(\hat{c}_i, z_i)\}_{i=1}^{N_{\text{val}}}$.
\begin{itemize}
    \item \textbf{Temperature Scaling (TS)}~\citep{guo2017calibration}.
    Originally designed for multi-class logits, TS is adapted to binary calibration by scaling the log-odds of the confidence, which can also be viewed as a special case of \textbf{Platt Scaling}~\citep{platt1999probabilistic} with no bias. It optimizes a scalar temperature $T > 0$ to minimize the Negative Log Likelihood (NLL) on the calibration set:
    \begin{equation*}
        \min_{T} - \sum_{i=1}^{N_{\text{val}}} \Big[ z_i \log(\sigma(\phi_i/T)) + (1-z_i) \log(1 - \sigma(\phi_i/T)) \Big],
    \end{equation*}
    where $\phi_i = \text{logit}(\hat{c}_i) = \log(\frac{\hat{c}_i}{1-\hat{c}_i})$ denotes the uncalibrated log-odds, and $\sigma(\cdot)$ is the sigmoid function. The resulting calibrated confidence score is given by:
    \begin{equation*}
        \tilde{c} = \sigma\left(\frac{1}{T} \log\left(\frac{\hat{c}}{1-\hat{c}}\right)\right).
    \end{equation*}

    \item \textbf{Histogram Binning (HB)}~\citep{zadrozny2001obtaining}.
    This non-parametric method partitions the confidence interval $[0,1]$ into $B$ disjoint bins defined by boundaries $0 = a_1 < \dots < a_{B+1} = 1$. For any sample falling into bin $B_m = (a_m, a_{m+1}]$, the calibrated score $\theta_m$ is estimated as the average empirical accuracy (i.e., the mean of $z_i$) within that bin. The calibration map can be viewed as minimizing the bin-wise squared loss:
    \begin{equation*}
        \min_{\theta_1, \dots, \theta_B} \sum_{m=1}^B \sum_{i=1}^{N_{\text{val}}} \mathbb{I}(\hat{c}_i \in B_m) (\theta_m - z_i)^2,
    \end{equation*}
    assigning $\tilde{c} = \theta_m$ wherever $\hat{c} \in B_m$. An empty bin uses the midpoint $(a_m+a_{m+1})/2$ and the bin index is clipped to the valid range at deployment. We use a uniform partition with $B=10$ bins, and assign confidence zero to the first bin.

    \item \textbf{Isotonic Regression (IR)}~\citep{zadrozny2002transforming}.
    IR minimizes squared error subject to non-decreasing fitted values, as specified below. Our implementation uses isotonic least-squares fitting, linearly interpolates between the fitted thresholds at prediction time, and clips predictions outside the fitted confidence range to the corresponding endpoint value. Specifically, we can write the optimization problem as:
    \begin{equation*}
    \begin{aligned}
        \min_{f} & \sum_{i=1}^{N_{\text{val}}} (f(\hat{c}_i) - z_i)^2 \\
        \text{s.t.} & \quad \forall i, j: \hat{c}_i \le \hat{c}_j \implies f(\hat{c}_i) \le f(\hat{c}_j).
    \end{aligned}
    \end{equation*}
    The calibrated confidence is obtained via $\tilde{c} = f(\hat{c})$.

    \item \textbf{Gaussian Process Calibration (GPcalib)}~\citep{wenger2020nonparametric}.
    We use a binary Gaussian-process calibration baseline that maps the scalar confidence $c$ to the correctness indicator $z=\mathbb{I}(\hat{\state}=\state)$. Its predictive probability $\Pr(z=1\mid c)$ becomes the calibrated confidence, while the original predicted answer is retained. This binary implementation adapts the GP-based calibration approach to the (answer, confidence) setting. On CommonsenseQA, GPcalib uses a subset of 1,500 examples to balance performance and time complexity.

    \item \textbf{PAC-Bayes Recalibration (PBR)}~\citep{fujisawa2024pac}.
    We apply the recalibration objective in Eq.~(8) of \citet{fujisawa2024pac} to the binary correctness-calibration task defined above. The posterior-predictive mean provides the calibrated confidence, while the original predicted answer is retained. We select the regularization strength by five-fold cross-validation on the calibration set using held-out Brier score, and then refit on the full calibration set.
\end{itemize}

\paragraph{Label-Free Baselines.} These baselines, like our method, require no validation labels but instead exploit a reference model. They differ, however, in \emph{when} the reference model is queried. The first two heuristics (BCT and ConfAvg) query the reference model on \emph{every test input}, incurring the cost of an additional forward pass at inference time. The third baseline (DACA), like our own method, uses the reference model only on the validation split to fit a recalibration mapping and requires no access to the reference model at test time.
\begin{itemize}
    \item \textbf{Best-Confidence Transfer (BCT).} BCT retains the primary model's predicted answer and replaces its confidence with the probability assigned to that answer by the reference model. This baseline requires access to the probability assigned by the reference model to the primary model's chosen answer and a reference-model query for each evaluation input.
    \item \textbf{Confidence Averaging (ConfAvg).} Keep the primary model's predicted answer but use the average value of the primary and reference confidences. Requires a reference query at test time.
    \item \textbf{Disagreement-Aware Confidence Alignment (DACA)}~\citep{luo2026your}. DACA performs unsupervised temperature scaling by minimizing the KL divergence between the reference and primary models on the validation split, restricted to inputs where their final answers agree. We similarly adapt it to our (answer, confidence) setting by formulating calibration as a binary classification task. Specifically, DACA fits the temperature $\tau$ through:
    \begin{equation*}
        \min_{\tau>0} \sum_{i=1}^{N_{\text{val}}} \mathbb{I}(\hat{\state}_{ref,i}= \hat{\state}_{pri,i}) \cdot \Big[ \hat{c}_{ref,i}\log\dfrac{\hat{c}_{ref,i}}{\sigma(\phi_{pri,i}/\tau)}+(1-\hat{c}_{ref,i})\log\dfrac{1-\hat{c}_{ref,i}}{1-\sigma(\phi_{pri,i}/\tau)}  \Big],
    \end{equation*}
    where $\phi_{pri,i} = \text{logit}(\hat{c}_{pri,i}) = \log(\frac{\hat{c}_{pri,i}}{1-\hat{c}_{pri,i}})$ denotes the uncalibrated log-odds, and $\sigma(\cdot)$ is the sigmoid function. The resulting calibrated confidence score is given by:
    \begin{equation*}
        \tilde{c} = \sigma\left(\frac{1}{\tau} \log\left(\frac{\hat{c}}{1-\hat{c}}\right)\right).
    \end{equation*}
\end{itemize}

\subsection{Evaluation Metrics}
\label{sec:evaluation metrics formula}

Let $\mathcal{D} = \{(\bm{x}_i, \state_i)\}_{i=1}^{N}$ denote the evaluation set. For each sample $\bm{x}_i$, the model predicts an answer $\hat{\state}_i$ with an associated confidence $\hat{c}_i \in [0,1]$. We evaluate the model's performance using the following three metrics in our main experiments:
\begin{itemize}
    \item \textbf{Prediction Accuracy (Acc).} This metric measures the correctness of the LLM's predictions:
    \begin{align*}
        \text{Acc} = \frac{1}{N}\sum_{i=1}^N \mathbb{I}(\hat{\state}_i = \state_i),
    \end{align*}
    where $\mathbb{I}(\cdot)$ is the indicator function.

    \item \textbf{Calibration Metrics.} We assess calibration performance using \textbf{Expected Calibration Error (ECE)}~\citep{guo2017calibration} and the \textbf{Brier Score (BS)}~\citep{brier1950verification}. Brier Score measures the mean squared error between the confidence and the prediction correctness:
    \begin{align*}
        \text{BS} = \frac{1}{N} \sum_{i=1}^N (\hat{c}_i - \mathbb{I}(\hat{\state}_i = \state_i))^2.
    \end{align*}

    For ECE, we partition $[0,1]$ into $M=10$ uniform-width bins. Writing $z_i=\mathbb{I}(\hat{\state}_i=\state_i)$, the first bin contains $0\leq \hat c_i\leq 1/M$, and bin $m\geq 2$ contains $(m-1)/M<\hat c_i\leq m/M$. Let $B_m$ denote the set of sample indices in bin $m$. We compute
    \[
    \mathrm{ECE}
    =
    \sum_{m:|B_m|>0}\frac{|B_m|}{N}
    \left|
    \frac{1}{|B_m|}\sum_{i\in B_m}z_i
    -
    \frac{1}{|B_m|}\sum_{i\in B_m}\hat c_i
    \right|.
    \]
    Empty bins contribute zero.
    \item \textbf{Optimization Objective:} We also report the \textbf{Confidence Loss (CL)} as defined in Equation~\ref{eq:proper loss}, which serves as our primary optimization target:
    \begin{equation*}
        \text{CL} = \frac{1}{N} \sum_{i=1}^N \left( -\mathbb{I}(\hat{\state}_i = \state_i) + (\hat{c}_i - \mathbb{I}(\hat{\state}_i = \state_i))^2 \right).
    \end{equation*}
\end{itemize}

Beyond the standard ECE with 10 uniform-width bins used in our main experiments, we evaluate alternative binned and smooth calibration-error estimators.

For binned ECE, we consider uniform-width binning (UWB) and uniform-mass binning (UMB). UWB uses equal-width intervals with the boundary convention defined above. UMB sorts examples by confidence and partitions the ordered examples into groups whose sizes differ by at most one. We evaluate $M=10$, $M=15$, and $M^*=\max\{1,\lfloor n^{1/3}\rfloor\}$, where $n$ is the number of evaluated examples, following the sample-size scaling discussed by \citet{futami2024information}.

For the non-binned check, we compute automatic-bandwidth SmoothECE~\citep{blasiok2024smooth} following the official repo\footnote{\url{https://github.com/apple/ml-calibration}} with a reflected Gaussian kernel and no logit scaling. Within each dataset--model setting, all methods use the common sample-ID intersection of the full comparison set, and correctness is recomputed from each method's final answer.

Unless otherwise stated, the $\pm$ entries in the metric tables are normal-approximation $95\%$ interval half-widths, $1.96\,\widehat{\mathrm{SE}}$. For accuracy, we use $\widehat{\mathrm{SE}}=\sqrt{\widehat a(1-\widehat a)/n}$. For Brier Score and Confidence Loss, we use the sample standard deviation of the per-example losses divided by $\sqrt n$. The Confidence Loss calculation uses the complete per-example quantity $(\hat c_i-z_i)^2-z_i$.

For binned ECE, let $n_m$ be the number of examples in bin $m$ and let $s_m^2$ be the sample variance of $z_i-\hat c_i$ within that bin. We use the within-bin approximation
\[
\widehat{\mathrm{SE}}_{\mathrm{ECE}}
=
\frac{1}{n}
\sqrt{\sum_{m:n_m\geq 2}n_m s_m^2}.
\]
Bins with fewer than two examples contribute zero to this variance estimate. This approximation conditions on the empirical bin memberships and counts, including the empirical partition for UMB.

For SmoothECE, we report $95\%$ percentile intervals from $1000$ paired bootstrap replicates. For each (dataset,model) setting, every replicate draws $n$ examples with replacement from the matched evaluation set, where $n$ is the matched sample count reported in the table. All methods use the same sampled example indices. Within each replicate, the SmoothECE bandwidth is selected anew for each method, while the fitted calibration maps remain fixed. The interval endpoints are the $2.5$th and $97.5$th percentiles of the bootstrap estimates. The reported point estimate is computed using all matched examples; percentile intervals need not contain this estimate.

\section{Additional Results}
\label{app:additional results}

\subsection{Few-Label Reference Validation}
\label{sec:few-label validation}

Our framework assumes access to an (approximately) calibrated reference predictor, which is typically identified via prior knowledge of the model family. When such prior knowledge is unavailable, the better-calibrated model can instead be selected empirically using a small labeled diagnostic set. To assess the practicality of this approach, we quantify the minimum number of samples $N^*$ required to empirically verify that one model is better calibrated than the other. This is a diagnostic analysis on fixed evaluation splits: MMLU test, MMLU-Redux test, and CommonsenseQA validation. These labels are not used to fit our recalibration map.

For each tested sample size
\[
N\in\{10,20,\ldots,200,250,300,400,500,750,1000\},
\]
we independently draw $1000$ subsets of size $N$ without replacement from the corresponding evaluation split. On each subset $S_r$, we recompute 10-bin uniform-width ECE for both models and form
\[
\Delta_r(N)
=
\mathrm{ECE}_{S_r}(\mathrm{Instruct})
-
\mathrm{ECE}_{S_r}(\mathrm{Base}).
\]
We define $N^*$ as the smallest tested $N$ for which the empirical 2.5th percentile of $\{\Delta_r(N)\}_{r=1}^{1000}$ is positive. We repeat this procedure using $1000$ Monte Carlo seeds and report the mean and standard deviation of $N^*$. Results are shown in Table~\ref{tab:few_label_validation}. The mean $N^*$ is below $200$ in eight of the nine settings. For \texttt{Ministral-3-8B} on CommonsenseQA, all $1000$ seeds yield $N^*=750$ on the tested grid. This is only about $7.7\%$ of the full calibration set of $9,741$ examples.

\begin{table}[ht]
\centering
\caption{\textbf{Minimum number of labeled samples $N^*$ needed to verify that the Base model has lower ECE than their Instruct counterpart.} $N^*$ is the smallest tested sample size for which the empirical 2.5th percentile of the Instruct-minus-Base ECE gap over $1000$ independently sampled subsets is positive. Entries report mean $\pm$ standard deviation across $1000$ Monte Carlo seeds on fixed evaluation splits. These labels are used only for diagnostic evaluation, not for fitting the recalibration map.}
\label{tab:few_label_validation}
\begin{tabular}{lccc}
\toprule
$N^*(\downarrow)$ & MMLU & MMLU-Redux & CommonSenseQA \\
\midrule
Qwen3-8B & 110.1$\pm$6.3 & 159.4$\pm$8.3 & 196.9$\pm$20.4 \\
LLaMA-3.1-8B & 48.1$\pm$3.9 &  49.2$\pm$2.8 & 77.9$\pm$4.4 \\
Ministral-3-8B & 74.9$\pm$5.2 & 125.9$\pm$7.3 & 750.0$\pm$0.0 \\
\bottomrule
\end{tabular}
\end{table}

\subsection{Comparison with Additional Supervised Baselines}
\label{sec:additional baselines}

Tables~\ref{tab:all baselines mmlu-redux} and~\ref{tab:all baselines commonseqa} report the full results on \textbf{MMLU-Redux} and \textbf{CommonSenseQA} respectively, including comparisons against an extended set of baselines (See Appendix~\ref{sec:baselines details}). The trends observed in Section~\ref{sec:results} are consistent across this broader set of comparisons.

Specifically, among label-free methods that do not require Base model access at test time (except Base itself), our approach attains the best ECE for \texttt{Qwen3-8B} on MMLU-Redux and \texttt{Ministral-3-8B} on both datasets, while simultaneously preserving or marginally improving the predictive accuracy of the Instruct model. Within this group, our method achieves the best Confidence Loss (CL) for \texttt{Ministral-3-8B}, confirming that the gains of our method arise from optimization of the proper loss. The \emph{Dummy Reference} yields much smaller reductions in ECE, supporting the usefulness of information from the Base reference. Consistent with the failure case discussed in the main paper, \texttt{LLaMA-3.1-8B} remains the most challenging setting for our method.

\begin{table}[ht]
\centering
\caption{\textbf{Main results on MMLU-Redux with more baselines and error bars.} \textbf{Bold} and \underline{underline} denote the \textbf{best} and \underline{second-best} results among \emph{label-free} methods that \emph{do not} require access to Base Model during test time, while \dag indicates the \emph{global} best across all methods. Accuracy for some Baselines is marked with \textit{--} as these post-hoc methods do not alter final predictions.}
\label{tab:all baselines mmlu-redux}
\setlength{\tabcolsep}{2pt}
\renewcommand{\arraystretch}{1.1}
\resizebox{\textwidth}{!}{%
\begin{tabular}{lcccccc}
\toprule
\textbf{Methods} & \textbf{Label?} & \makecell{\textbf{Base}\\\textbf{Access?}} & \textbf{ACC} ($\uparrow$) & \textbf{BS} ($\downarrow$) & \textbf{ECE} ($\downarrow$) & \textbf{CL} ($\downarrow$) \\
\midrule
\multicolumn{7}{c}{\textbf{Qwen3-8B}} \\
\midrule
Base Model ($Q_0$) & \ding{55} & \ding{55} & 77.19\pms{1.09}\% & \underline{0.1227\pms{0.0043}} & \textbf{0.0305\pms{0.0091}} & $-$0.6493\pms{0.0140} \\
Instruct Model ($Q_1$) & \ding{55} & \ding{55} & \underline{82.53\pms{0.99}\%} & 0.1362\pms{0.0079} & 0.1269\pms{0.0089} & $-$0.6891\pms{0.0174} \\
\midrule
\rowcolor{gray!10} Temperature Scaling & \ding{51} & \ding{55} & -- & 0.1137\pms{0.0051} & 0.0426\pms{0.0086} & $-$0.7116\pms{0.0145} \\
\rowcolor{gray!10} Histogram Binning & \ding{51} & \ding{55} & -- & 0.1198\pms{0.0059} & 0.0208\pms{0.0090} & $-$0.7054\pms{0.0150} \\
\rowcolor{gray!10} Isotonic Regression & \ding{51} & \ding{55} & -- & 0.1108\pms{0.0054}\dag & 0.0189\pms{0.0086}\dag & $-$0.7145\pms{0.0144}\dag \\
\rowcolor{gray!10} GPcalib & \ding{51} & \ding{55} & -- & 0.1132\pms{0.0056} & 0.0470\pms{0.0086} & $-$0.7120\pms{0.0147} \\
\rowcolor{gray!10} PBR & \ding{51} & \ding{55} & -- & 0.1146\pms{0.0061} & 0.0565\pms{0.0087} & $-$0.7107\pms{0.0152} \\
\midrule
\rowcolor{gray!10} BCT & \ding{55} & \ding{51} & -- & 0.1407\pms{0.0055} & 0.1227\pms{0.0088} & $-$0.6845\pms{0.0122} \\
\rowcolor{gray!10} ConfAvg & \ding{55} & \ding{51} & -- & 0.1118\pms{0.0053} & 0.0346\pms{0.0086} & $-$0.7135\pms{0.0148} \\
\midrule
DACA& \ding{55} & \ding{55} & -- & \textbf{0.1225\pms{0.0041}} & 0.0985\pms{0.0086} & \textbf{$-$0.7028\pms{0.0135}} \\
Dummy Reference & \ding{55} & \ding{55} & \underline{82.53\pms{0.99}\%} & 0.1336\pms{0.0077} & 0.1145\pms{0.0089} & $-$0.6916\pms{0.0172} \\
\textbf{Ours} & \ding{55} & \ding{55} & \textbf{82.56\pms{0.99}\%}\dag & 0.1236\pms{0.0062} & \underline{0.0348\pms{0.0089}} & \underline{$-$0.7020\pms{0.0158}} \\
\midrule
\multicolumn{7}{c}{\textbf{LLaMA-3.1-8B}} \\
\midrule
Base Model ($Q_0$) & \ding{55} & \ding{55} & 65.35\pms{1.24}\% & \textbf{0.1574\pms{0.0045}}\dag & \textbf{0.0208\pms{0.0103}} & $-$0.4961\pms{0.0147} \\
Instruct Model ($Q_1$) & \ding{55} & \ding{55} & \underline{70.91\pms{1.18}\%} & 0.2550\pms{0.0107} & 0.2526\pms{0.0113} & $-$0.4542\pms{0.0222} \\
\midrule
\rowcolor{gray!10} Temperature Scaling & \ding{51} & \ding{55} & -- & 0.1873\pms{0.0053} & 0.0364\pms{0.0111} & $-$0.5218\pms{0.0166} \\
\rowcolor{gray!10} Histogram Binning & \ding{51} & \ding{55} & -- & 0.1927\pms{0.0056} & 0.0118\pms{0.0114}\dag & $-$0.5165\pms{0.0167} \\
\rowcolor{gray!10} Isotonic Regression & \ding{51} & \ding{55} & -- & 0.1836\pms{0.0056} & 0.0181\pms{0.0111} & $-$0.5255\pms{0.0164} \\
\rowcolor{gray!10} GPcalib & \ding{51} & \ding{55} & -- & 0.1878\pms{0.0055} & 0.0126\pms{0.0112} & $-$0.5214\pms{0.0166} \\
\rowcolor{gray!10} PBR & \ding{51} & \ding{55} & -- & 0.2142\pms{0.0087} & 0.1657\pms{0.0112} & $-$0.4950\pms{0.0192} \\
\midrule
\rowcolor{gray!10} BCT & \ding{55} & \ding{51} & -- & 0.1783\pms{0.0057} & 0.1061\pms{0.0104} & $-$0.5308\pms{0.0133} \\
\rowcolor{gray!10} ConfAvg & \ding{55} & \ding{51} & -- & 0.1758\pms{0.0062} & 0.1011\pms{0.0104} & $-$0.5333\pms{0.0176}\dag \\
\midrule
DACA & \ding{55} & \ding{55} & -- & \underline{0.1878\pms{0.0048}} & \underline{0.0453\pms{0.0111}} & \textbf{$-$0.5213\pms{0.0162}} \\
Dummy Reference & \ding{55} & \ding{55} & \textbf{70.92\pms{1.18}\%}\dag & 0.2516\pms{0.0105} & 0.2454\pms{0.0113} & $-$0.4577\pms{0.0220} \\
\textbf{Ours} & \ding{55} & \ding{55} & \underline{70.91\pms{1.18}\%} & 0.2115\pms{0.0080} & 0.1340\pms{0.0114} & \underline{$-$0.4976\pms{0.0196}} \\
\midrule
\multicolumn{7}{c}{\textbf{Ministral-3-8B}} \\
\midrule
Base Model ($Q_0$) & \ding{55} & \ding{55} & 75.72\pms{1.11}\% & \textbf{0.1281\pms{0.0041}}\dag & \underline{0.0534\pms{0.0091}} & $-$0.6291\pms{0.0135} \\
Instruct Model ($Q_1$) & \ding{55} & \ding{55} & \underline{80.30\pms{1.03}\%} & 0.1726\pms{0.0093} & 0.1656\pms{0.0099} & $-$0.6304\pms{0.0194} \\
\midrule
\rowcolor{gray!10} Temperature Scaling & \ding{51} & \ding{55} & -- & 0.1476\pms{0.0054} & 0.0445\pms{0.0098} & $-$0.6553\pms{0.0152} \\
\rowcolor{gray!10} Histogram Binning & \ding{51} & \ding{55} & -- & 0.1494\pms{0.0053} & 0.0537\pms{0.0099} & $-$0.6536\pms{0.0152} \\
\rowcolor{gray!10} Isotonic Regression & \ding{51} & \ding{55} & -- & 0.1459\pms{0.0055} & 0.0460\pms{0.0098} & $-$0.6571\pms{0.0150} \\
\rowcolor{gray!10} GPcalib & \ding{51} & \ding{55} & -- & 0.1452\pms{0.0054} & 0.0472\pms{0.0098} & $-$0.6578\pms{0.0151} \\
\rowcolor{gray!10} PBR & \ding{51} & \ding{55} & -- & 0.1590\pms{0.0075} & 0.1118\pms{0.0099} & $-$0.6439\pms{0.0168} \\
\midrule
\rowcolor{gray!10} BCT & \ding{55} & \ding{51} & -- & 0.1494\pms{0.0052} & 0.1326\pms{0.0091} & $-$0.6536\pms{0.0123} \\
\rowcolor{gray!10} ConfAvg & \ding{55} & \ding{51} & -- & 0.1285\pms{0.0056} & 0.0440\pms{0.0092} & $-$0.6745\pms{0.0156}\dag \\
\midrule
DACA & \ding{55} & \ding{55} & -- & 0.1534\pms{0.0045} & 0.0914\pms{0.0098} & \underline{$-$0.6495\pms{0.0143}} \\
Dummy Reference & \ding{55} & \ding{55} & \textbf{80.48\pms{1.03}\%}\dag & 0.1681\pms{0.0090} & 0.1498\pms{0.0099} & $-$0.6368\pms{0.0191} \\
\textbf{Ours} & \ding{55} & \ding{55} & \textbf{80.48\pms{1.03}\%}\dag & \underline{0.1475\pms{0.0065}} & \textbf{0.0218\pms{0.0098}}\dag & \textbf{$-$0.6573\pms{0.0167}} \\
\bottomrule
\end{tabular}%
}
\end{table}

\begin{table}[ht]
\centering
\caption{\textbf{Main results on CommonSenseQA with more baselines and error bars.} \textbf{Bold} and \underline{underline} denote the \textbf{best} and \underline{second-best} results among \emph{label-free} methods that \emph{do not} require access to Base Model during test time, while \dag indicates the \emph{global} best across all methods. Accuracy for some Baselines is marked with \textit{--} as these post-hoc methods do not alter final predictions.}
\label{tab:all baselines commonseqa}
\setlength{\tabcolsep}{2pt}
\renewcommand{\arraystretch}{1.1}
\resizebox{\textwidth}{!}{%
\begin{tabular}{lcccccc}
\toprule
\textbf{Methods} & \textbf{Label?} & \makecell{\textbf{Base}\\\textbf{Access?}} & \textbf{ACC} ($\uparrow$) & \textbf{BS} ($\downarrow$) & \textbf{ECE} ($\downarrow$) & \textbf{CL} ($\downarrow$) \\
\midrule
\multicolumn{7}{c}{\textbf{Qwen3-8B}} \\
\midrule
Base Model ($Q_0$) & \ding{55} & \ding{55} & \textbf{83.21\pms{2.10}\%}\dag & \textbf{0.1167\pms{0.0095}} & \textbf{0.0611\pms{0.0188}} & \textbf{$-$0.7154\pms{0.0286}}\dag \\
Instruct Model ($Q_1$) & \ding{55} & \ding{55} & \underline{82.80\pms{2.12}\%} & 0.1447\pms{0.0181} & 0.1355\pms{0.0199} & $-$0.6833\pms{0.0387} \\
\midrule
\rowcolor{gray!10} Temperature Scaling & \ding{51} & \ding{55} & -- & 0.1178\pms{0.0113} & 0.0266\pms{0.0192} & $-$0.7102\pms{0.0317} \\
\rowcolor{gray!10} Histogram Binning & \ding{51} & \ding{55} & -- & 0.1278\pms{0.0126} & 0.0270\pms{0.0200} & $-$0.7002\pms{0.0325} \\
\rowcolor{gray!10} Isotonic Regression & \ding{51} & \ding{55} & -- & 0.1171\pms{0.0112} & 0.0234\pms{0.0191} & $-$0.7109\pms{0.0306} \\
\rowcolor{gray!10} GPcalib & \ding{51} & \ding{55} & -- & 0.1209\pms{0.0123} & 0.0228\pms{0.0195}\dag & $-$0.7071\pms{0.0315} \\
\rowcolor{gray!10} PBR & \ding{51} & \ding{55} & -- & 0.1228\pms{0.0135} & 0.0596\pms{0.0194} & $-$0.7052\pms{0.0331} \\
\midrule
\rowcolor{gray!10} BCT & \ding{55} & \ding{51} & -- & 0.1165\pms{0.0104}\dag & 0.0934\pms{0.0181} & $-$0.7115\pms{0.0270} \\
\rowcolor{gray!10} ConfAvg & \ding{55} & \ding{51} & -- & 0.1199\pms{0.0128} & 0.0510\pms{0.0191} & $-$0.7082\pms{0.0335} \\
\midrule
DACA & \ding{55} & \ding{55} & -- & \underline{0.1226\pms{0.0095}} & \underline{0.0631\pms{0.0192}} & \underline{$-$0.7054\pms{0.0299}} \\
Dummy Reference & \ding{55} & \ding{55} & \underline{82.80\pms{2.12}\%} & 0.1447\pms{0.0181} & 0.1354\pms{0.0199} & $-$0.6833\pms{0.0387} \\
\textbf{Ours} & \ding{55} & \ding{55} & \underline{82.80\pms{2.12}\%} & 0.1309\pms{0.0153} & 0.0638\pms{0.0199} & $-$0.6971\pms{0.0360} \\
\midrule
\multicolumn{7}{c}{\textbf{LLaMA-3.1-8B}} \\
\midrule
Base Model ($Q_0$) & \ding{55} & \ding{55} & \underline{71.33\pms{2.54}\%} & \underline{0.1708\pms{0.0108}} & \textbf{0.0363\pms{0.0231}}\dag & $-$0.5425\pms{0.0327} \\
Instruct Model ($Q_1$) & \ding{55} & \ding{55} & \textbf{78.30\pms{2.31}\%}\dag & 0.2099\pms{0.0224} & 0.2097\pms{0.0229} & $-$0.5731\pms{0.0454} \\
\midrule
\rowcolor{gray!10} Temperature Scaling & \ding{51} & \ding{55} & -- & 0.1671\pms{0.0111} & 0.0471\pms{0.0227} & $-$0.6158\pms{0.0337} \\
\rowcolor{gray!10} Histogram Binning & \ding{51} & \ding{55} & -- & 0.1687\pms{0.0114} & 0.0430\pms{0.0229} & $-$0.6143\pms{0.0342} \\
\rowcolor{gray!10} Isotonic Regression & \ding{51} & \ding{55} & -- & 0.1651\pms{0.0116} & 0.0400\pms{0.0227} & $-$0.6179\pms{0.0338} \\
\rowcolor{gray!10} GPcalib & \ding{51} & \ding{55} & -- & 0.1666\pms{0.0111} & 0.0487\pms{0.0228} & $-$0.6164\pms{0.0338} \\
\rowcolor{gray!10} PBR & \ding{51} & \ding{55} & -- & 0.1889\pms{0.0184} & 0.1456\pms{0.0228} & $-$0.5941\pms{0.0391} \\
\midrule
\rowcolor{gray!10} BCT & \ding{55} & \ding{51} & -- & 0.1933\pms{0.0134} & 0.1509\pms{0.0218} & $-$0.5896\pms{0.0286} \\
\rowcolor{gray!10} ConfAvg & \ding{55} & \ding{51} & -- & 0.1554\pms{0.0140}\dag & 0.0605\pms{0.0218} & $-$0.6276\pms{0.0367}\dag \\
\midrule
DACA & \ding{55} & \ding{55} & -- & \textbf{0.1703\pms{0.0099}} & \underline{0.0740\pms{0.0227}} & \textbf{$-$0.6126\pms{0.0324}} \\
Dummy Reference & \ding{55} & \ding{55} & \textbf{78.30\pms{2.31}\%}\dag & 0.2028\pms{0.0216} & 0.1906\pms{0.0229} & $-$0.5802\pms{0.0446} \\
\textbf{Ours} & \ding{55} & \ding{55} & \textbf{78.30\pms{2.31}\%}\dag & 0.1776\pms{0.0177} & 0.1054\pms{0.0228} & \underline{$-$0.6053\pms{0.0407}} \\
\midrule
\multicolumn{7}{c}{\textbf{Ministral-3-8B}} \\
\midrule
Base Model ($Q_0$) & \ding{55} & \ding{55} & 69.94\pms{2.57}\% & \underline{0.1948\pms{0.0085}} & \underline{0.1349\pms{0.0235}} & $-$0.5046\pms{0.0279} \\
Instruct Model ($Q_1$) & \ding{55} & \ding{55} & \underline{75.35\pms{2.42}\%} & 0.2050\pms{0.0200} & 0.1691\pms{0.0236} & $-$0.5485\pms{0.0433} \\
\midrule
\rowcolor{gray!10} Temperature Scaling & \ding{51} & \ding{55} & -- & 0.1930\pms{0.0096} & 0.1111\pms{0.0237} & $-$0.5605\pms{0.0321} \\
\rowcolor{gray!10} Histogram Binning & \ding{51} & \ding{55} & -- & 0.1768\pms{0.0112} & 0.0349\pms{0.0235} & $-$0.5766\pms{0.0342} \\
\rowcolor{gray!10} Isotonic Regression & \ding{51} & \ding{55} & -- & 0.1765\pms{0.0112} & 0.0344\pms{0.0235} & $-$0.5770\pms{0.0342} \\
\rowcolor{gray!10} GPcalib & \ding{51} & \ding{55} & -- & 0.1808\pms{0.0113} & 0.0560\pms{0.0235} & $-$0.5727\pms{0.0334} \\
\rowcolor{gray!10} PBR & \ding{51} & \ding{55} & -- & 0.1944\pms{0.0158} & 0.1200\pms{0.0237} & $-$0.5591\pms{0.0382} \\
\midrule
\rowcolor{gray!10} BCT & \ding{55} & \ding{51} & -- & 0.2226\pms{0.0120} & 0.2458\pms{0.0221} & $-$0.5308\pms{0.0248} \\
\rowcolor{gray!10} ConfAvg & \ding{55} & \ding{51} & -- & 0.1638\pms{0.0108}\dag & 0.0225\pms{0.0227}\dag & $-$0.5897\pms{0.0339}\dag \\
\midrule
DACA & \ding{55} & \ding{55} & -- & 0.2098\pms{0.0053} & 0.1674\pms{0.0238} & $-$0.5437\pms{0.0283} \\
Dummy Reference & \ding{55} & \ding{55} & \textbf{75.41\pms{2.42}\%}\dag & 0.2051\pms{0.0200} & 0.1684\pms{0.0236} & \underline{$-$0.5490\pms{0.0433}} \\
\textbf{Ours} & \ding{55} & \ding{55} & \textbf{75.41\pms{2.42}\%}\dag & \textbf{0.1779\pms{0.0148}} & \textbf{0.0571\pms{0.0235}} & \textbf{$-$0.5762\pms{0.0380}} \\
\bottomrule
\end{tabular}%
}
\end{table}

\subsection{Main Results on MMLU}
\label{sec:additional results for MMLU}

Table~\ref{tab:all baselines mmlu} presents the full experimental results on the \textbf{MMLU} benchmark. We observe trends similar to the findings on MMLU-Redux and CommonSenseQA reported in Section~\ref{sec:results}.

\begin{table}[ht]
\centering
\caption{\textbf{Main results on MMLU with more baselines and error bars.} \textbf{Bold} and \underline{underline} denote the \textbf{best} and \underline{second-best} results among \emph{label-free} methods that \emph{do not} require access to Base Model during test time, while \dag indicates the \emph{global} best across all methods. Accuracy for some Baselines is marked with \textit{--} as these post-hoc methods do not alter final predictions.}
\label{tab:all baselines mmlu}
\setlength{\tabcolsep}{2pt}
\renewcommand{\arraystretch}{1.1}
\resizebox{\textwidth}{!}{%
\begin{tabular}{lcccccc}
\toprule
\textbf{Methods} & \textbf{Label?} & \makecell{\textbf{Base}\\\textbf{Access?}} & \textbf{ACC} ($\uparrow$) & \textbf{BS} ($\downarrow$) & \textbf{ECE} ($\downarrow$) & \textbf{CL} ($\downarrow$) \\
\midrule
\multicolumn{7}{c}{\textbf{Qwen3-8B}} \\
\midrule
Base Model ($Q_0$) & \ding{55} & \ding{55} & 74.55\pms{0.72}\% & \underline{0.1367\pms{0.0029}} & \textbf{0.0266\pms{0.0061}} & $-$0.6089\pms{0.0093} \\
Instruct Model ($Q_1$) & \ding{55} & \ding{55} & 79.72\pms{0.66}\% & 0.1579\pms{0.0054} & 0.1490\pms{0.0060} & $-$0.6392\pms{0.0117} \\
\midrule
\rowcolor{gray!10} Temperature Scaling & \ding{51} & \ding{55} & -- & 0.1275\pms{0.0034} & 0.0483\pms{0.0058} & $-$0.6696\pms{0.0097} \\
\rowcolor{gray!10} Histogram Binning & \ding{51} & \ding{55} & -- & 0.1362\pms{0.0040} & 0.0179\pms{0.0061}\dag & $-$0.6610\pms{0.0101} \\
\rowcolor{gray!10} Isotonic Regression & \ding{51} & \ding{55} & -- & 0.1244\pms{0.0036}\dag & 0.0269\pms{0.0058} & $-$0.6728\pms{0.0096}\dag \\
\rowcolor{gray!10} GPcalib & \ding{51} & \ding{55} & -- & 0.1274\pms{0.0037} & 0.0533\pms{0.0058} & $-$0.6698\pms{0.0098} \\
\rowcolor{gray!10} PBR & \ding{51} & \ding{55} & -- & 0.1288\pms{0.0040} & 0.0704\pms{0.0058} & $-$0.6683\pms{0.0101} \\
\midrule
\rowcolor{gray!10} BCT & \ding{55} & \ding{51} & -- & 0.1550\pms{0.0037} & 0.1143\pms{0.0059} & $-$0.6422\pms{0.0083} \\
\rowcolor{gray!10} ConfAvg & \ding{55} & \ding{51} & -- & 0.1288\pms{0.0036} & 0.0471\pms{0.0058} & $-$0.6684\pms{0.0100} \\
\midrule
DACA & \ding{55} & \ding{55} & -- & \textbf{0.1342\pms{0.0027}} & 0.0902\pms{0.0058} & \textbf{$-$0.6630\pms{0.0091}} \\
Dummy Reference & \ding{55} & \ding{55} & \underline{79.73\pms{0.66}\%} & 0.1551\pms{0.0052} & 0.1369\pms{0.0060} & $-$0.6422\pms{0.0116} \\
\textbf{Ours} & \ding{55} & \ding{55} & \textbf{79.75\pms{0.66}\%}\dag & 0.1412\pms{0.0042} & \underline{0.0499\pms{0.0060}} & \underline{$-$0.6563\pms{0.0106}} \\
\midrule
\multicolumn{7}{c}{\textbf{LLaMA-3.1-8B}} \\
\midrule
Base Model ($Q_0$) & \ding{55} & \ding{55} & 64.39\pms{0.79}\% & \textbf{0.1665\pms{0.0029}}\dag & \textbf{0.0090\pms{0.0067}}\dag & $-$0.4775\pms{0.0094} \\
Instruct Model ($Q_1$) & \ding{55} & \ding{55} & 70.35\pms{0.76}\% & 0.2639\pms{0.0069} & 0.2613\pms{0.0073} & $-$0.4396\pms{0.0143} \\
\midrule
\rowcolor{gray!10} Temperature Scaling & \ding{51} & \ding{55} & -- & 0.1918\pms{0.0034} & 0.0388\pms{0.0072} & $-$0.5116\pms{0.0106} \\
\rowcolor{gray!10} Histogram Binning & \ding{51} & \ding{55} & -- & 0.1966\pms{0.0036} & 0.0189\pms{0.0073} & $-$0.5069\pms{0.0108} \\
\rowcolor{gray!10} Isotonic Regression & \ding{51} & \ding{55} & -- & 0.1888\pms{0.0035} & 0.0181\pms{0.0072} & $-$0.5146\pms{0.0105} \\
\rowcolor{gray!10} GPcalib & \ding{51} & \ding{55} & -- & 0.1921\pms{0.0035} & 0.0149\pms{0.0072} & $-$0.5114\pms{0.0107} \\
\rowcolor{gray!10} PBR                       & \ding{51} & \ding{55} & -- & 0.2216\pms{0.0056} & 0.1697\pms{0.0072} & $-$0.4818\pms{0.0123} \\
\midrule
\rowcolor{gray!10} BCT & \ding{55} & \ding{51} & -- & 0.1907\pms{0.0037} & 0.1132\pms{0.0068} & $-$0.5128\pms{0.0085} \\
\rowcolor{gray!10} ConfAvg & \ding{55} & \ding{51} & -- & 0.1830\pms{0.0040} & 0.1017\pms{0.0068} & $-$0.5204\pms{0.0113}\dag \\
\midrule
DACA & \ding{55} & \ding{55} & -- & \underline{0.1920\pms{0.0031}} & \underline{0.0356\pms{0.0072}} & \textbf{$-$0.5114\pms{0.0104}} \\
Dummy Reference & \ding{55} & \ding{55} & \textbf{70.37\pms{0.76}\%}\dag & 0.2603\pms{0.0068} & 0.2540\pms{0.0073} & $-$0.4433\pms{0.0142} \\
\textbf{Ours} & \ding{55} & \ding{55} & \underline{70.35\pms{0.76}\%} & 0.2181\pms{0.0052} & 0.1426\pms{0.0073} & \underline{$-$0.4854\pms{0.0126}} \\
\midrule
\multicolumn{7}{c}{\textbf{Ministral-3-8B}} \\
\midrule
Base Model ($Q_0$) & \ding{55} & \ding{55} & 72.57\pms{0.74}\% & \textbf{0.1376\pms{0.0026}}\dag & \underline{0.0558\pms{0.0061}} & $-$0.5880\pms{0.0089} \\
Instruct Model ($Q_1$) & \ding{55} & \ding{55} & 76.37\pms{0.70}\% & 0.2143\pms{0.0065} & 0.2079\pms{0.0068} & $-$0.5494\pms{0.0134} \\
\midrule
\rowcolor{gray!10} Temperature Scaling & \ding{51} & \ding{55} & -- & 0.1694\pms{0.0037} & 0.0329\pms{0.0068} & $-$0.5943\pms{0.0104} \\
\rowcolor{gray!10} Histogram Binning & \ding{51} & \ding{55} & -- & 0.1729\pms{0.0037} & 0.0113\pms{0.0069}\dag & $-$0.5909\pms{0.0105} \\
\rowcolor{gray!10} Isotonic Regression & \ding{51} & \ding{55} & -- & 0.1683\pms{0.0038} & 0.0358\pms{0.0068} & $-$0.5954\pms{0.0103} \\
\rowcolor{gray!10} GPcalib & \ding{51} & \ding{55} & -- & 0.1686\pms{0.0037} & 0.0228\pms{0.0068} & $-$0.5951\pms{0.0104} \\
\rowcolor{gray!10} PBR & \ding{51} & \ding{55} & -- & 0.1880\pms{0.0052} & 0.1252\pms{0.0068} & $-$0.5757\pms{0.0115} \\
\midrule
\rowcolor{gray!10} BCT & \ding{55} & \ding{51} & -- & 0.1546\pms{0.0032} & 0.1069\pms{0.0061} & $-$0.6091\pms{0.0083} \\
\rowcolor{gray!10} ConfAvg & \ding{55} & \ding{51} & -- & 0.1507\pms{0.0038} & 0.0758\pms{0.0062} & $-$0.6130\pms{0.0107}\dag \\
\midrule
DACA & \ding{55} & \ding{55} & -- & \underline{0.1716\pms{0.0031}} & \textbf{0.0535\pms{0.0068}} & \textbf{$-$0.5921\pms{0.0098}} \\
Dummy Reference & \ding{55} & \ding{55} & \textbf{76.50\pms{0.70}\%}\dag & 0.2086\pms{0.0063} & 0.1922\pms{0.0069} & $-$0.5564\pms{0.0132} \\
\textbf{Ours} & \ding{55} & \ding{55} & \underline{76.48\pms{0.70}\%} & 0.1760\pms{0.0045} & 0.0634\pms{0.0068} & \underline{$-$0.5889\pms{0.0115}} \\
\bottomrule
\end{tabular}%
}
\end{table}

\subsection{Discretization Sensitivity}
\label{sec:discretization sensitivity}

When estimating the joint distribution $\prob_{\repRV_0, \repRV_1}$ (Appendix~\ref{sec:implementation details}), we use a discretization step $s$ to map continuous input confidence scores to the nearest value on a discrete grid. Table~\ref{tab:discretization_sensitivity} reports the sensitivity of our method to $s$. Our method is robust across choices of $s$, consistently reducing ECE relative to the Instruct model (primary model). \texttt{Qwen3-8B} and \texttt{Ministral-3-8B} exhibit a U-shaped dependence on $s$ on MMLU and MMLU-Redux, whereas all three families attain their lowest tested ECE at $s=0.002$ on CommonsenseQA. CommonsenseQA has a larger calibration set of $9,741$ examples, compared with $1,531$ for the other two benchmarks, providing more observations for estimating the joint distribution on finer grids. We retain $s=0.01$ as the common default rather than selecting a separate grid for each evaluation setting.

\begin{table}[t]
\centering
\caption{\textbf{ECE sensitivity to discretization step $s$ across benchmarks.} We default $s=0.01$ in our main experiments. \textbf{Bold} denotes the \textbf{best} results among different $s$.}
\label{tab:discretization_sensitivity}
\begin{tabular}{l|ccc}
\toprule
$s$ & \textbf{Qwen3-8B} & \textbf{LLaMA-3.1-8B} & \textbf{Ministral-3-8B} \\
\midrule
\multicolumn{4}{c}{\textbf{MMLU-Redux}} \\
\midrule
Base & 0.0305\pms{0.0091} & 0.0208\pms{0.0103} & 0.0534\pms{0.0091} \\
Instruct & 0.1269\pms{0.0089} & 0.2526\pms{0.0113} & 0.1656\pms{0.0099} \\
\midrule
0.1 & 0.0795\pms{0.0088} & 0.2020\pms{0.0113} & 0.0641\pms{0.0099} \\
0.05 & 0.0752\pms{0.0088} & 0.1855\pms{0.0113} & 0.0548\pms{0.0099} \\
0.02 & 0.0501\pms{0.0088} & 0.1622\pms{0.0113} & 0.0437\pms{0.0098} \\
\textbf{0.01 (Main)} & \textbf{0.0348\pms{0.0089}} & 0.1340\pms{0.0114} & \textbf{0.0218\pms{0.0098}} \\
0.005 & 0.0683\pms{0.0089} & 0.0927\pms{0.0113} & 0.0396\pms{0.0099} \\
0.002 & 0.0991\pms{0.0090} & \textbf{0.0514\pms{0.0114}} & 0.0852\pms{0.0100} \\
\midrule
\multicolumn{4}{c}{\textbf{CommonSenseQA}} \\
\midrule
Base & 0.0611\pms{0.0188} & 0.0363\pms{0.0231} & 0.1349\pms{0.0235} \\
Instruct & 0.1355\pms{0.0199} & 0.2097\pms{0.0229} & 0.1691\pms{0.0236} \\
\midrule
0.1 & 0.0795\pms{0.0199} & 0.1147\pms{0.0228} & 0.0665\pms{0.0235} \\
0.05 & 0.0759\pms{0.0200} & 0.1151\pms{0.0228} & 0.0666\pms{0.0234} \\
0.02 & 0.0711\pms{0.0199} & 0.1128\pms{0.0228} & 0.0613\pms{0.0235} \\
\textbf{0.01 (Main)} & 0.0638\pms{0.0199} & 0.1054\pms{0.0228} & 0.0571\pms{0.0235} \\
0.005 & 0.0511\pms{0.0199} & 0.0920\pms{0.0229} & 0.0451\pms{0.0235} \\
0.002 & \textbf{0.0285\pms{0.0201}} & \textbf{0.0567\pms{0.0229}} & \textbf{0.0138\pms{0.0236}} \\
\midrule
\multicolumn{4}{c}{\textbf{MMLU}} \\
\midrule
Base & 0.0266\pms{0.0061} & 0.0090\pms{0.0067} & 0.0558\pms{0.0061} \\
Instruct & 0.1490\pms{0.0060} & 0.2613\pms{0.0073} & 0.2079\pms{0.0068} \\
\midrule
0.1 & 0.1023\pms{0.0060} & 0.2106\pms{0.0073} & 0.1055\pms{0.0068} \\
0.05 &  0.0976\pms{0.0059} & 0.1936\pms{0.0073} & 0.0962\pms{0.0068} \\
0.02 & 0.0734\pms{0.0060} & 0.1707\pms{0.0073} & 0.0848\pms{0.0068} \\
\textbf{0.01 (Main)} & \textbf{0.0499\pms{0.0060}} & 0.1426\pms{0.0073} & 0.0634\pms{0.0068} \\
0.005 & 0.0633\pms{0.0060} & 0.1000\pms{0.0073} & \textbf{0.0419\pms{0.0068}} \\
0.002 & 0.0991\pms{0.0060} & \textbf{0.0568\pms{0.0073}} & 0.0477\pms{0.0069} \\
\bottomrule
\end{tabular}
\end{table}

\subsection{Robustness to Alternative Calibration-Error Estimators}
\label{sec:alternative ece}

Tables~\ref{tab:alternative_ece} and~\ref{tab:smooth_ece} evaluate robustness across alternative binned ECE measures and automatic-bandwidth SmoothECE (Appendix~\ref{sec:evaluation metrics formula}). Our method improves over the Instruct model under every reported binning scheme and achieves lower SmoothECE in all nine (dataset,model) settings, indicating that the improvement is robust to the choice of calibration-error estimator.

\begin{table}[t]
\setlength{\extrarowheight}{2pt}
\centering
\caption{\textbf{ECE under alternative binning schemes.} UWB/UMB denote uniform-width and uniform-mass binning. OPT uses the sample-size-aware bin count $M^*=\max\{1,\lfloor n^{1/3}\rfloor\}$, where $n$ is the number of evaluated examples. \textbf{Bold} denotes the \textbf{best} results.}
\label{tab:alternative_ece}
\setlength{\tabcolsep}{2pt}
\resizebox{\textwidth}{!}{%
\begin{tabular}{ll|cccccc}
\toprule
Model & Method & UWB-10 & UWB-15 & UWB-OPT & UMB-10 & UMB-15 & UMB-OPT \\
\midrule
\multicolumn{8}{c}{\textbf{MMLU-Redux}} \\
\midrule
\multirow{2}{*}{Qwen3-8B} & Instruct & 0.1269\pms{0.0089} & 0.1269\pms{0.0088} & 0.1269\pms{0.0088} & 0.1269\pms{0.0086} & 0.1269\pms{0.0086} & 0.1269\pms{0.0086} \\
& DACA & 0.0985\pms{0.0086} & 0.0985\pms{0.0086} & 0.0988\pms{0.0086} & 0.0987\pms{0.0086} & 0.0987\pms{0.0086} & 0.0979\pms{0.0086} \\
& \textbf{Ours} & \textbf{0.0348\pms{0.0089}} & \textbf{0.0643\pms{0.0088}} & \textbf{0.0570\pms{0.0089}} & \textbf{0.0693\pms{0.0087}} & \textbf{0.0708\pms{0.0087}} & \textbf{0.0711\pms{0.0088}} \\
\midrule
\multirow{2}{*}{LLaMA-3.1-8B} & Instruct & 0.2526\pms{0.0113} & 0.2526\pms{0.0113} & 0.2526\pms{0.0113} & 0.2526\pms{0.0111} & 0.2526\pms{0.0111} & 0.2526\pms{0.0111} \\
& DACA & \textbf{0.0453\pms{0.0111}} & \textbf{0.0453\pms{0.0111}} & \textbf{0.0471\pms{0.0111}} & \textbf{0.0550\pms{0.0111}} & \textbf{0.0588\pms{0.0111}} & \textbf{0.0569\pms{0.0111}} \\
& \textbf{Ours} & 0.1340\pms{0.0114} & 0.1339\pms{0.0113} & 0.1339\pms{0.0113} & 0.1338\pms{0.0112} & 0.1338\pms{0.0112} & 0.1338\pms{0.0113} \\
\midrule
\multirow{2}{*}{Ministral-3-8B} & Instruct & 0.1656\pms{0.0099} & 0.1656\pms{0.0099} & 0.1656\pms{0.0099} & 0.1656\pms{0.0098} & 0.1656\pms{0.0097} & 0.1656\pms{0.0097} \\
& DACA & 0.0914\pms{0.0098} & 0.0953\pms{0.0098} & 0.0968\pms{0.0098} & 0.0922\pms{0.0098} & 0.0943\pms{0.0098} & 0.0962\pms{0.0098} \\
& \textbf{Ours} & \textbf{0.0218\pms{0.0098}} & \textbf{0.0221\pms{0.0098}} & \textbf{0.0321\pms{0.0098}} & \textbf{0.0474\pms{0.0098}} & \textbf{0.0453\pms{0.0098}} & \textbf{0.0447\pms{0.0098}} \\
\midrule
\multicolumn{8}{c}{\textbf{CommonSenseQA}} \\
\midrule
\multirow{2}{*}{Qwen3-8B} & Instruct & 0.1355\pms{0.0199} & 0.1357\pms{0.0198} & 0.1355\pms{0.0199} & 0.1347\pms{0.0192} & 0.1347\pms{0.0191} & 0.1347\pms{0.0192} \\
& DACA & \textbf{0.0631\pms{0.0192}} & 0.0676\pms{0.0192} & \textbf{0.0631\pms{0.0192}} & 0.0716\pms{0.0192} & \textbf{0.0695\pms{0.0192}} & 0.0716\pms{0.0192} \\
& \textbf{Ours} & 0.0638\pms{0.0199} & \textbf{0.0641\pms{0.0199}} & 0.0638\pms{0.0199} & \textbf{0.0708\pms{0.0197}} & 0.0707\pms{0.0198} & \textbf{0.0708\pms{0.0197}} \\
\midrule
\multirow{2}{*}{LLaMA-3.1-8B} & Instruct & 0.2097\pms{0.0229} & 0.2097\pms{0.0228} & 0.2097\pms{0.0229} & 0.2075\pms{0.0227} & 0.2075\pms{0.0227} & 0.2075\pms{0.0227} \\
& DACA & \textbf{0.0740\pms{0.0227}} & \textbf{0.0740\pms{0.0227}} & \textbf{0.0740\pms{0.0227}} & \textbf{0.0776\pms{0.0226}} & \textbf{0.0750\pms{0.0227}} & \textbf{0.0776\pms{0.0226}} \\
& \textbf{Ours} & 0.1054\pms{0.0228} & 0.1072\pms{0.0227} & 0.1054\pms{0.0228} & 0.1039\pms{0.0230} & 0.1039\pms{0.0229} & 0.1039\pms{0.0230} \\
\midrule
\multirow{2}{*}{Ministral-3-8B} & Instruct & 0.1691\pms{0.0236} & 0.1710\pms{0.0236} & 0.1691\pms{0.0236} & 0.1665\pms{0.0235} & 0.1665\pms{0.0235} & 0.1665\pms{0.0235} \\
& DACA & 0.1674\pms{0.0238} & 0.1660\pms{0.0237} & 0.1674\pms{0.0238} & 0.1660\pms{0.0236} & 0.1677\pms{0.0235} & 0.1660\pms{0.0236} \\
& \textbf{Ours} & \textbf{0.0571\pms{0.0235}} & \textbf{0.0573\pms{0.0234}} & \textbf{0.0571\pms{0.0235}} & \textbf{0.0524\pms{0.0234}} & \textbf{0.0655\pms{0.0234}} & \textbf{0.0524\pms{0.0234}} \\
\midrule
\multicolumn{8}{c}{\textbf{MMLU}} \\
\midrule
\multirow{2}{*}{Qwen3-8B} & Instruct & 0.1490\pms{0.0060} & 0.1490\pms{0.0060} & 0.1490\pms{0.0059} & 0.1490\pms{0.0058} & 0.1490\pms{0.0058} & 0.1490\pms{0.0058} \\
& DACA & 0.0902\pms{0.0058} & 0.0927\pms{0.0058} & 0.0950\pms{0.0058} & 0.0944\pms{0.0058} & 0.0951\pms{0.0058} & 0.0945\pms{0.0058} \\
& \textbf{Ours} & \textbf{0.0499\pms{0.0060}} & \textbf{0.0641\pms{0.0059}} & \textbf{0.0613\pms{0.0060}} & \textbf{0.0691\pms{0.0059}} & \textbf{0.0818\pms{0.0059}} & \textbf{0.0820\pms{0.0059}} \\
\midrule
\multirow{2}{*}{LLaMA-3.1-8B} & Instruct & 0.2613\pms{0.0073} & 0.2613\pms{0.0073} & 0.2614\pms{0.0073} & 0.2609\pms{0.0072} & 0.2609\pms{0.0072} & 0.2609\pms{0.0072} \\
& DACA & \textbf{0.0356\pms{0.0072}} & \textbf{0.0356\pms{0.0072}} & \textbf{0.0378\pms{0.0072}} & \textbf{0.0518\pms{0.0072}} & \textbf{0.0545\pms{0.0072}} & \textbf{0.0557\pms{0.0072}} \\
& \textbf{Ours} & 0.1426\pms{0.0073} & 0.1427\pms{0.0073} & 0.1428\pms{0.0073} & 0.1417\pms{0.0072} & 0.1417\pms{0.0072} & 0.1417\pms{0.0072} \\
\midrule
\multirow{2}{*}{Ministral-3-8B} & Instruct & 0.2079\pms{0.0068} & 0.2079\pms{0.0068} & 0.2079\pms{0.0068} & 0.2079\pms{0.0067} & 0.2079\pms{0.0067} & 0.2079\pms{0.0066} \\
& DACA & \textbf{0.0535\pms{0.0068}} & \textbf{0.0573\pms{0.0068}} & \textbf{0.0574\pms{0.0068}} & 0.0794\pms{0.0067} & 0.0858\pms{0.0067} & 0.0881\pms{0.0067} \\
& \textbf{Ours} & 0.0634\pms{0.0068} & 0.0633\pms{0.0068} & 0.0633\pms{0.0067} & \textbf{0.0631\pms{0.0067}} & \textbf{0.0683\pms{0.0067}} & \textbf{0.0673\pms{0.0067}} \\
\bottomrule
\end{tabular}%
}
\end{table}

\begin{table*}[t]
    \centering
    \caption{\textbf{SmoothECE on all datasets (lower is better).} Each entry reports the SmoothECE estimate followed by a $95\%$ bootstrap percentile interval in brackets. Within each (dataset,model) setting, all methods are evaluated on the same examples. \emph{Matched $n$} counts examples whose sample IDs appear in the valid predictions of all 12 compared methods, including methods not displayed here. Intervals are computed from $1000$ paired bootstrap resamples, using the same sampled examples for every method. See Appendix~\ref{sec:evaluation metrics formula} for details.}
    \label{tab:smooth_ece}
    \setlength{\tabcolsep}{3pt}
    \resizebox{\textwidth}{!}{
    \begin{tabular}{llccc}
        \toprule
        \textbf{Model} & \textbf{Method} & \textbf{MMLU} & \textbf{MMLU-Redux} & \textbf{CommonSenseQA} \\
        \midrule
        \multirow{6}{*}{\texttt{Qwen3-8B}}
          & \emph{Matched $n$} & 14,040 & 5,700 & 1,221 \\
          & Base & 0.0262 [0.0221, 0.0322] & 0.0291 [0.0254, 0.0377] & 0.0552 [0.0407, 0.0736] \\
          & Instruct & 0.1751 [0.1683, 0.1825] & 0.1510 [0.1402, 0.1619] & 0.1634 [0.1397, 0.1855] \\
          & TS \emph{(Supervised)} & 0.0486 [0.0453, 0.0518] & 0.0445 [0.0393, 0.0501] & 0.0295 [0.0246, 0.0459] \\
          & DACA \emph{(Label-Free)} & 0.0806 [0.0773, 0.0837] & 0.0870 [0.0820, 0.0925] & 0.0660 [0.0531, 0.0800] \\
          & \textbf{Ours} \emph{(Label-Free)} & 0.0475 [0.0432, 0.0530] & 0.0405 [0.0363, 0.0449] & 0.0625 [0.0438, 0.0814] \\
        \midrule
        \multirow{6}{*}{\texttt{LLaMA-3.1-8B}}
          & \emph{Matched $n$} & 14,038 & 5,699 & 1,221 \\
          & Base & 0.0116 [0.0119, 0.0179] & 0.0211 [0.0178, 0.0312] & 0.0318 [0.0270, 0.0497] \\
          & Instruct & 0.2860 [0.2781, 0.2938] & 0.2793 [0.2662, 0.2914] & 0.2233 [0.1975, 0.2491] \\
          & TS \emph{(Supervised)} & 0.0373 [0.0332, 0.0417] & 0.0374 [0.0304, 0.0436] & 0.0459 [0.0347, 0.0663] \\
          & DACA \emph{(Label-Free)} & 0.0340 [0.0299, 0.0397] & 0.0409 [0.0336, 0.0486] & 0.0716 [0.0519, 0.0939] \\
          & \textbf{Ours} \emph{(Label-Free)} & 0.1416 [0.1346, 0.1486] & 0.1337 [0.1215, 0.1449] & 0.1042 [0.0828, 0.1281] \\
        \midrule
        \multirow{6}{*}{\texttt{Ministral-3-8B}}
          & \emph{Matched $n$} & 14,019 & 5,687 & 1,220 \\
          & Base & 0.0493 [0.0449, 0.0536] & 0.0529 [0.0473, 0.0616] & 0.1352 [0.1097, 0.1579] \\
          & Instruct & 0.2275 [0.2202, 0.2349] & 0.1827 [0.1728, 0.1936] & 0.1630 [0.1408, 0.1903] \\
          & TS \emph{(Supervised)} & 0.0344 [0.0290, 0.0396] & 0.0441 [0.0367, 0.0517] & 0.0970 [0.0804, 0.1139] \\
          & DACA \emph{(Label-Free)} & 0.0513 [0.0461, 0.0574] & 0.0858 [0.0771, 0.0945] & 0.1665 [0.1416, 0.1885] \\
          & \textbf{Ours} \emph{(Label-Free)} & 0.0630 [0.0566, 0.0694] & 0.0259 [0.0231, 0.0311] & 0.0556 [0.0383, 0.0786] \\
        \bottomrule
    \end{tabular}
    }
\end{table*}

\end{document}